\documentclass{IEEEtran}
\usepackage{cite}
\usepackage{amsmath,amssymb,amsfonts}
\usepackage{graphicx}
\usepackage{textcomp}

\usepackage{url}
\usepackage{subfigure}
\usepackage{amsthm}
\newtheorem{theorem}{$\mathbf{Theorem}$}[section]

\newtheorem{definition}{$\mathbf{Definition}$}[section]

\usepackage{multirow}
\usepackage{mdwtab}
\usepackage{algorithm,algpseudocode,float}
\usepackage{lipsum}

\makeatletter
\newenvironment{breakablealgorithm}
  {
     \refstepcounter{algorithm}
     \hrule height.8pt depth0pt \kern2pt
     \renewcommand{\caption}[2][\relax]{
       {\raggedright\textbf{Model~\thealgorithm} ##2\par}%
       \ifx\relax##1\relax 
         \addcontentsline{loa}{algorithm}{\protect\numberline{\thealgorithm}##2}%
       \else 
         \addcontentsline{loa}{algorithm}{\protect\numberline{\thealgorithm}##1}%
       \fi
       \kern2pt\hrule\kern2pt
     }
  }{
     \kern2pt\hrule\relax
  }
\makeatother

\renewcommand{\thesubfigure}{\roman{subfigure}} \makeatletter
\renewcommand{\@thesubfigure}{(\thesubfigure)\space}
\renewcommand{\p@subfigure}{\thefigure} \makeatother

\def\BibTeX{{\rm B\kern-.05em{\sc i\kern-.025em b}\kern-.08em
    T\kern-.1667em\lower.7ex\hbox{E}\kern-.125emX}}
\begin{document}
\title{A general model for plane-based clustering with loss function}
\author{Zhen Wang,~Yuan-Hai~Shao,~Lan~Bai,~Chun-Na~Li,~and~Li-Ming~Liu
\thanks{Submitted in \today. This work is supported in part by National Natural Science Foundation of
China (Nos. 11501310, 61866010, 11871183, and 61703370), in part by Natural Science Foundation of Hainan Province (No. 118QN181), and in part by Scientific Research Foundation of Hainan University (No. kyqd(sk)1804).}
\thanks{Zhen Wang is with
School of Mathematical Sciences, Inner Mongolia University, Hohhot,
010021, P.R.China e-mail: wangzhen@imu.edu.cn.}
\thanks{Yuan-Hai Shao (*Corresponding author) is with School of Economics and Management, Hainan University, Haikou,
570228, P.R.China e-mail: shaoyuanhai21@163.com.}
\thanks{Lan Bai is with
School of Mathematical Sciences, Inner Mongolia University, Hohhot,
010021, P.R.China e-mail: imubailan@163.com.}
\thanks{Chun-Na Li is with Zhijiang College, Zhejiang University of Technology, Hangzhou, 310024, P.R.China e-mail:
na1013na@163.com.}
\thanks{Li-Ming Liu is with School of Statistics, Capital University of Economics and Business, Beijing, 100070, P.R.China. e-mail: llm5609@163.com.}
}

\maketitle

\begin{abstract}
In this paper, we propose a general model for plane-based clustering. The general model contains many existing plane-based clustering methods, e.g., k-plane clustering (kPC), proximal plane clustering (PPC), twin support vector clustering (TWSVC) and its extensions. Under this general model, one may obtain an appropriate clustering method for specific purpose.
The general model is a procedure corresponding to an optimization problem, where the optimization problem minimizes the total loss of the samples. Thereinto, the loss of a sample derives from
both within-cluster and between-cluster. In theory, the termination conditions are discussed, and we prove that the general model terminates in a finite number of steps at a local or weak local optimal point.
Furthermore, based on this general model, we propose a plane-based clustering
method by introducing a new loss function
to capture the data distribution precisely. Experimental
results on artificial and public available datasets verify the effectiveness of
the proposed method.
\end{abstract}


\section{Introduction}
\IEEEPARstart{C}{lustering}, discovering the similarity among the data samples, is one of the most important unsupervised learning topics \cite{ClusterBook3,ClusterBook2,ClusterTNNA2}. Many approaches assign the samples into the clusters via certain cluster centers \cite{nKmeans,DP,Kmeans2,Kmedian,Kplane,Kflat}. The plane-based clustering treats the cluster center as a plane, and thus it is able to find the plane-based shape clusters. Moreover, the plane-based clustering can be extended to nonlinear manifold modeling easily to cope with complex data structures. The plane-based clustering has attracted much attention \cite{Kplane,kPPC,PPC,TWSVC,TWSVC_A1,TWSVC_A2,TWSVC_A3,RTWSVC}.

The first plane-based clustering, k-plane clustering (kPC) \cite{Kplane}, was proposed by O.L. Mangasarian et al., where the discriminative information from within-cluster was considered. Subsequently, the discriminative information from between-cluster has been introduced in plane-based clustering. For instance, proximal plane clustering (PPC) \cite{PPC} and twin support vector clustering (TWSVC) \cite{TWSVC} considered that the cluster center plane should be not only as close as possible to the current cluster samples but also far away from the other clusters. Still later, robust twin support vector clustering (RTWSVC) and fast robust twin support vector clustering (FRTWSVC) were also appeared \cite{RTWSVC}. Until recently, ramp-based twin support vector clustering (RampTWSVC) \cite{RampTWSVC} was proposed to deal with noise or outliers. So it is interesting to find a cluster center plane by considering the discriminative information both from within-cluster and between-cluster.

Let us notice the close relationship between the cluster problem and the classification problem. In fact, there are the following corresponding relationships between them: PPC corresponds to the generalized eigenvalue proximal support vector machines (GEPSVM) \cite{GEPSVM}, TWSVC to the twin support vector machines (TWSVM) \cite{TWSVM,TBSVM}, RTWSVC to the $L_1$-TWSVM \cite{L1TWSVM}, FRTWSVC to the $L_1$ least square TWSVM \cite{L1LSTWSVM}, and RampTWSVC to the best fitting hyperplanes for classification (BFHC) \cite{BFHC}. So it seems like a great way to relate the plane-based clustering to the supervised learning.

Briefly speaking, the supervised learning is essentially based on two concepts ``loss function'' and ``regularization'' \cite{SLT,L1SVM,PBSVM,LSSVM1,LPSVM,L1TWSVM,L1LSTWSVM,LPNSVM,ClusterTNNA3}. We find that the plane-based clustering can also be established in a similar way. This yields our general model. It is concerned with the new defined loss function from discriminative information \cite{ClusterTNNA1} and the regularization. The general model iteratively implements two parts: cluster update and cluster assignment. In the cluster update, the new cluster center planes would be obtained by minimizing the loss derived from the current cluster assignment. Besides, in the cluster assignment, each sample would be assigned to the cluster with the least loss.
For the general model, it is allowed to select various loss functions and regularization terms, and most of the existing plane-based clustering methods can be regarded as the particular cases with different selections. Furthermore, a new plane-based clustering is derived from the general model. More precisely, following the model, we propose a robust fitting distribution planes clustering (RFDPC) by hiring a new loss function, the ramp loss \cite{BFHC} combined with certain statistics, which owns clear geometric meaning and captures the data distribution.

The main contributions of this paper include:

\noindent
(i) A general model for plane-based clustering is proposed, in which different loss functions and regularization terms can be chosen, particularly yielding the existing kPC, PPC, TWSVC, RTWSVC, FRTWSVC, RampTWSVC, and etc.

\noindent
(ii) The cluster update and cluster assignment in the general model is consistent on minimizing the loss of samples, resulting in its finite termination at a local or weak local optimal point.

\noindent
(iii) A new loss function is introduced in the general model with named RFDPC, to cope with outliers, noise, and capture the data distribution more precisely.

\noindent
(iv) Experiments show the amazing performance of RFDPC compared with the existing plane-based clustering methods.

The rest of this paper is organized as follows. The general model is elaborated in section II. Some plane-based clustering methods are summarised under the general model in section III. A novel plane-based clustering method (RFDPC) is described in section IV.
Experiments and conclusions are presented in sections V and VI,
respectively.

\section{The general model for plane-based clustering}

\subsection{Formulation}
Remind the clustering problem with $m$
data samples $\{x_1,\,x_2,\ldots,x_m\}$ in the $n$-dimensional real vector space
$R^{n}$, which is represented by $X\in R^{n\times m}$. Assume that these $m$ samples belong to $k$ clusters
with their corresponding labels $y\in\{1,\,2,\ldots,k\}$.
Our task is to assign the $m$ samples into $k$ clusters, or to give their cluster labels
\begin{equation}\label{Y}
\begin{array}{l}
\mathbf{y}=(y(x_1),y(x_2),\ldots,y(x_m))^\top.
\end{array}
\end{equation}
For partition-based clustering \cite{Kmeans,Kmedian,PartClustering,PartClustering2}, the usual way is to find the cluster labels $\mathbf{y}$ as well as the $k$ cluster centers. The plane-based
clustering treats each cluster center as a plane. The $k$ cluster center planes are described as
\begin{equation}\label{Center}
\begin{array}{l}
w_j^\top x +b_j=0,~j=1,\ldots,k,
\end{array}
\end{equation}
where $w_j\in R^n$ is the weight vector and $b_j\in R$ is the bias term.
Consider the deviation of a sample $x$ from the $j$-th cluster center plane $w_j^\top x+b_j=0$ ($j=1,\ldots,k$). For instance, the deviation can be measured by the signed distance of $x$ to the plane as
\begin{equation}\label{D1}
\begin{array}{l}
f(x;w_j,b_j)=\frac{w_j^\top x+b_j}{||w_j||},
\end{array}
\end{equation}
where $||\cdot||$ denotes $L_2$ norm.
Another simpler way to reduce computation is to hire
\begin{equation}\label{D2}
\begin{array}{l}
f(x;w_j,b_j)=w_j^\top x+b_j.
\end{array}
\end{equation}

Combining these $k$ deviation functions, either \eqref{D1} or \eqref{D2}, yields the vector function
\begin{equation}\label{F}
\begin{array}{l}
F(x;W,\mathbf{b})=(f(x;w_1,b_1),f(x;w_2,b_2),\ldots,f(x;w_k,b_k))^\top,
\end{array}
\end{equation}
where $W=(w_1,\ldots,w_k)$ and $\mathbf{b}=(b_1,\ldots,b_k)^\top$.
Thus, the $k$ cluster center planes can be represented as
\begin{equation}\label{C3}
\begin{array}{l}
F(x;W,\mathbf{b})=0.
\end{array}
\end{equation}

One of the popular approaches \cite{Kplane,PPC,TWSVC} is to find the cluster labels $\mathbf{y}$ and the $k$ cluster center planes \eqref{C3} iteratively. Start with an initial assignment $\mathbf{y}$. Next, for the given $\mathbf{y}$, find the corresponding $F(x;W,\mathbf{b})$ by establishing and solving an optimization problem. Then, update $\mathbf{y}$ and the vector function $F(x;W,~\mathbf{b})$ alternately until certain termination conditions are satisfied.

The key point of this paper is to introduce the loss function into a general optimization problem. For the $i$th sample $x_i$ with assigned label $y(x_i)$, the ideal case is to find  $W^*=(w_1^*,\ldots,w_k^*)$ and $\mathbf{b}^*=(b_1^*,\ldots,b_k^*)^\top$ such that, on one hand the sample $x_i$ lies exactly on the center plane $f(x_i;w_{y(x_i)}^*,b_{y(x_i)}^*)=0$, and on the other hand, the sample $x_i$ is far away from other center planes (in extremity, $f(x_i;w_j^*,b_j^*)=\pm\infty$, $j\neq y(x_i)$). For the actual situation, the loss of sample $x_i$ should be a measure of the deviation from the ideal case. Therefore, it should consist of two parts: (i) for its own center plane, the loss should depend on the deviation $f(x_i;w_{y(x_i)},b_{y(x_i)})$ and can be measured by a within-cluster function $J^w(f(x_i;w_{y(x_i)},b_{y(x_i)}))$, where $J^w(\rho)$ is a function from $R$ to $R$ with the condition $J^w(0)=0$; (ii) for other center planes, the loss can be measured by a between-cluster function $J^b(f(x_i;w_j,b_j))$ with $j\neq y(x_i)$, where $J^b(\rho)$ is a function from $R$ to $R$. Thus, for the sample $x_i$ ($i=1,\ldots,m$), the loss is described by
\begin{equation}\label{LossDef}
\begin{array}{l}
L(y(x_i);F(x_i;W,\mathbf{b}))=c_wJ^w(f(x_i;w_{y(x_i)},b_{y(x_i)}))+\\c_b\sum\limits_{\begin{subarray}{c}
j=1\\
    j\neq y(x_i)
  \end{subarray}}^kJ^b(f(x_i;w_j,b_j)),
\end{array}
\end{equation}
where $c_w$ and $c_b$ are positive parameters. Furthermore,
for the dataset $X=\{x_1,\ldots,x_m\}$, the total loss is
\begin{equation}\label{Tloss}
\begin{array}{l}
L(\mathbf{y};F(x_1;W,\mathbf{b}),\ldots,F(x_m;W,\mathbf{b}))=\sum\limits_{i=1}^mL(y(x_i);F(x_i;W,\mathbf{b})).
\end{array}
\end{equation}
This leads to the optimization problem for both $\mathbf{y}$ and ($W$, $\mathbf{b}$) as
\begin{equation}\label{MainProblem}
\begin{array}{l}
\underset{\mathbf{y},W,\mathbf{b}}{\min}~~G(\mathbf{y};W,\mathbf{b})=L(\mathbf{y};F(x_1;W,\mathbf{b}),\ldots,F(x_m;W,\mathbf{b}))+||F||_\mathcal{F},
\end{array}
\end{equation}
where $||\cdot||_\mathcal{F}$ denotes the regularization term in the functional space $\mathcal{F}$.

Problem \eqref{MainProblem} is very similar to the optimization problem in supervised learning, i.e., it consists of the loss and regularization, but their concerns are not the same. For supervised learning, it aims at predicting the unknown samples by minimizing the loss of the training samples. However, for clustering, we focus on minimizing the loss of the given $m$ samples, and the regularization makes this efficient.
Based on problem \eqref{MainProblem}, the general model for plane-based clustering is constructed in Model 1.

~

\begin{breakablealgorithm}
	\caption{The general model}
\noindent \textbf{Input:} Dataset $X$, the within-cluster function $J^w(\rho)$, the between-cluster function $J^b(\rho)$, and the parameters $c_w$, $c_b$.\\
\textbf{Output:} $\mathbf{y}^*$ and $(W^*,~\mathbf{b}^*)$.\\
1. Initialize the sample labels $\mathbf{y}^{(0)}=(y^{(0)}(x_1),\ldots,y^{(0)}(x_m))^\top$.

\noindent 2. \textbf{For} $t=0,1,2,\ldots$, compute $W^{(t)},~\mathbf{b}^{(t)}$ and $\mathbf{y}^{(t+1)}$ by the following steps:

\textbf{(a) Cluster update:} For the current $\mathbf{y}^{(t)}$, $(W^{(t)},~\mathbf{b}^{(t)})$ is set to be the solution to the optimization problem
\begin{equation}\label{M1}
\begin{array}{l}
\underset{W,\mathbf{b}}{\min}~~G(\mathbf{y}^{(t)};W,\mathbf{b}),
\end{array}
\end{equation}
where $G(\cdot)$ is given by \eqref{MainProblem},
or equivalently the solutions to $k$ subproblems with $j=1,\ldots,k$ as follow:
\begin{equation}\label{Sub1}
\begin{array}{l}
\underset{w_j,b_j}{\min}~~c_w\sum\limits_{\begin{subarray}{c}
i=1\\
    y^{(t)}(x_i)=j
  \end{subarray}}^m J^w(f(x_i;w_j,b_j))+\\c_b\sum\limits_{\begin{subarray}{c}
i=1\\
    y^{(t)}(x_i)\neq j
  \end{subarray}}^m J^b(f(x_i;w_j,b_j))+||f(x_i;w_j,b_j)||_\mathcal{F}.
\end{array}
\end{equation}

\textbf{(b) Cluster assignment:} For the current $(W^{(t)},~\mathbf{b}^{(t)})$, the labels $\mathbf{y}^{(t+1)}$ are set to be the solution of the following optimization problem
\begin{equation}\label{M2}
\begin{array}{l}
\underset{\mathbf{y}}{\min}~~G(\mathbf{y},W^{(t)},\mathbf{b}^{(t)}),
\end{array}
\end{equation}
or equivalently, the labels $\mathbf{y}^{(t+1)}$ are given by
\begin{equation}\label{Sub2}
\begin{array}{l}
y(x_i)=\underset{j}{\arg}\min ~~L(j;F(x_i;W^{(t)},\mathbf{b}^{(t)}))
\end{array}
\end{equation}
with $i=1,\ldots,m$.
If there is a tie, the cluster with the smallest label number is selected.

\textbf{(c) Repetitiveness check:} If $(W^{(t)},~\mathbf{b}^{(t)})$ is a solution to problem \eqref{M1} where $\mathbf{y}^{(t)}$ is replaced by $\mathbf{y}^{(t+1)}$, break the loop and go to step 3.

\textbf{(d)} If the termination condition is satisfied, go to step 3; otherwise, set $t=t+1$ and back to step 2(a).

3. Set $\mathbf{y}^*=\mathbf{y}^{(t)}$, $W^*=W^{(t)}$, $\mathbf{b}^*=b^{(t)}$.
\end{breakablealgorithm}

~

\noindent$\mathbf{Remark.}$ In step 1 of Model 1, a common way is to assign the samples into $k$ clusters randomly, resulting in unstable clustering performance. It is preferable to choose some stable initialization techniques, e.g., nearest neighbor graph (NNG) \cite{TWSVC}, which has been successfully applied to several plane-based clustering methods \cite{TWSVC,RTWSVC,RampTWSVC}. In step 2(a), if there are many global solutions can be obtained, the same ones are selected for the same $\mathbf{y}$. However, we may only get a local solution. Note that there has been $(W^{(t-1)},~\mathbf{b}^{(t-1)})$ (for $t\geq1$) before solving problem \eqref{M1}. The local solution in step 2(a) must be not worse than previous solution if a local solution is inevitable. Thus, it is a good choice to hire the $(t-1)$-th local solution as the initial point of the $t$-th problem in step 2(a), if the assumption is false in step 2(c). In other words, the inequality $G(\mathbf{y}^{(t)},W^{(t-1)},\mathbf{b}^{(t-1)})\geq G(\mathbf{y}^{(t)},W^{(t)},\mathbf{b}^{(t)})$ always holds in iteration.

Besides, the functions $J^w(\rho)$ and $J^b(\rho)$ should also be pre-defined. Obviously, it is reasonable to select them with the following properties.

\noindent$\mathbf{Properties:}$

\noindent
(i) $J^w(\rho)=J^w(-\rho)$ and $J^b(\rho)=J^b(-\rho)$.

\noindent
(ii) $J^w(\rho)$ is monotonically non-decreasing in $[0,\infty)$.

\noindent
(iii) $J^b(\rho)$ is monotonically non-increasing in $[0,\infty)$.\newline
In this case, we have following theorem.
\begin{theorem}
If the within-cluster function
$J^w(\rho)$ and the between-cluster function $J^b(\rho)$ satisfy the above three properties (i)-(iii), then the sample assignment \eqref{Sub2} can be simplified as
\begin{equation}\label{Predict}
\begin{array}{l}
y(x_i)=\underset{j}{\arg}\min |f(x_i;w_j,b_j)|,
\end{array}
\end{equation}
where $|\cdot|$ denotes the absolute value.
\end{theorem}
\begin{proof}
Suppose that, for an arbitrary sample $x$, $l^*$ is the label of $x$ obtained by \eqref{Predict}, i.e., $|f(x;w_{l^*},b_{l^*})|$ is the smallest one in $\{|f(x;w_1,b_1)|,\ldots,|f(x;w_k,b_k)|\}$, and suppose $l$ is an arbitrary label of $x$. The objective values of \eqref{Sub2} at $l^*$ and $l$ are
\begin{equation}\label{Sub2-1}
\begin{array}{l}
L(l^*;F)=c_wJ^w(f(x;w_{l^*},b_{l^*}))+c_bJ^b(f(x;w_l,b_l))\\+c_b\sum\limits_{\begin{subarray}{c}
j=1\\
    j\neq l^*,j\neq l
  \end{subarray}}^kJ^b(f(x;w_j,b_j)),
\end{array}
\end{equation}
and
\begin{equation}\label{Sub2-2}
\begin{array}{l}
L(l;F)=c_wJ^w(f(x;w_l,b_l))+c_bJ^b(f(x;w_{l^*},b_{l^*}))\\+c_b\sum\limits_{\begin{subarray}{c}
j=1\\
    j\neq l^*,j\neq l
  \end{subarray}}^kJ^b(f(x;w_j,b_j)),
\end{array}
\end{equation}
respectively.

From the properties (i) and (ii), we have $J^w(f(x;w_{l^*},b_{l^*}))\leq J^w(f(x;w_l,b_l))$. Similarly, we have $J^b(f(x;w_l,b_l))\leq J^b(f(x;w_{l^*},b_{l^*}))$ from the properties (i) and (iii). Thus, $L(l^*;F)\leq L(l;F)$ because of the positive $c_w$ and $c_b$. This implies that $l^*$ corresponds to the smallest objective value of \eqref{Sub2}.
\end{proof}

Now, we extend the general model to the nonlinear case via a kernel trick \cite{nKmeans,Kernel2,TWSVC,SGTSVM}. For the nonlinear manifold clustering, the $k$ cluster centers are defined as
\begin{equation}\label{nCenter}
\begin{array}{l}
w_j^\top \phi(x) +b_j=0,~j=1,\ldots,k,
\end{array}
\end{equation}
where $\phi(\cdot)$ is a pre-defined nonlinear mapping. Thus, the deviation of a sample from a cluster center depends on the nonlinear mapping $\phi(\cdot)$ strictly.
Generally, it is not necessary to give the explicit nonlinear mapping $\phi(\cdot)$. Note that the deviation in general model is just considered. There are many kernel tricks to estimate the deviation. For instance, the deviation $f(\phi(x);w_j,b_j)$ can be estimated by $f(K(x,X);w_j,b_j)$, where $K(\cdot,\cdot)$ is a predetermined kernel function \cite{Kernel2} and $K(x_1,x_2)=<\phi(x_1),\phi(x_2)>$ ($<\cdot,\cdot>$ denotes the inner product). By selecting an appropriate kernel function, the nonlinear general model can be obtained without any difficulty, so the details are omitted.

\subsection{Analysis}
In this subsection, the termination conditions of the above general model are analysed. More exactly, it is concerned with the following three termination conditions.

\noindent$\mathbf{Termination~Conditions:}$

\noindent
(i) It happens that there is a repeated overall assignment of samples to clusters, i.e., $\mathbf{y}^{(p)}=\mathbf{y}^{(q)}$ where $p\neq q$ \cite{Kplane,PPC}.

\noindent
(ii) It happens that there is a non-decrease in the objective function $G(\cdot)$ \cite{Kplane,PPC}.

\noindent
(iii) Both the cases (i) and (ii) happen.\\
Corresponding to the different meaning of the solution in step 2(a), we have the following two theorems.
\begin{theorem}
Under either termination condition (i) or (ii), the general model terminates in a finite number of steps if the solution in step 2(a) means global solution.
\end{theorem}
\begin{proof}
The iterations in the general model can be summarized as
\begin{equation}\label{Sequence}
\begin{array}{l}
\mathbf{y}^{(0)}\rightarrow (W^{(0)},\mathbf{b}^{(0)})\rightarrow \mathbf{y}^{(1)}\rightarrow (W^{(1)},\mathbf{b}^{(1)})\rightarrow\cdots\\ \rightarrow \mathbf{y}^{(t)}\rightarrow (W^{(t)},\mathbf{b}^{(t)})\rightarrow \cdots.
\end{array}
\end{equation}
Since there are a finite number of ways that the $m$ samples can be assigned to $k$ clusters, there are two integers $t,p>0$ such that $\mathbf{y}^{(t)}=\mathbf{y}^{(t+p)}$. Therefore, the general model terminates in a finite number of steps under termination condition (i).

Moreover, the corresponding $(W^{(t)},\mathbf{b}^{(t)})$ and $(W^{(t+p)},\mathbf{b}^{(t+p)})$ are the global solutions to the same optimization problem \eqref{M1}. Thus, we have $G(\mathbf{y}^{(t)};W^{(t)},\mathbf{b}^{(t)})=G(\mathbf{y}^{(t+p)};W^{(t+p)},\mathbf{b}^{(t+p)})$. Note that the global solution in step 2(a) guarantees that the objective $G(\mathbf{y};W,\mathbf{b})$ is non-increasing in iteration. Then we have
\begin{equation}\label{Equals}
\begin{array}{l}
G(\mathbf{y}^{(t)};W^{(t)},\mathbf{b}^{(t)})=G(\mathbf{y}^{(t+1)};W^{(t)},\mathbf{b}^{(t)})=
G(\mathbf{y}^{(t+1)};\\W^{(t+1)}, \mathbf{b}^{(t+1)})=\cdots=G(\mathbf{y}^{(t+p)};W^{(t+p)},\mathbf{b}^{(t+p)}).
\end{array}
\end{equation}
Therefore, the general model terminates in a finite number of steps under termination condition (ii).
\end{proof}

\begin{theorem}
Suppose the number of the local solutions or the local optimal values to the problem $\underset{W,\mathbf{b}}{\min}~G(\mathbf{y};W,\mathbf{b})$ is finite. Under termination condition (iii), the general model terminates in a finite number of steps if the solution in step 2(a) means local solution.
\end{theorem}
\begin{proof}
Consider sequence \eqref{Sequence}. Since there are a finite number of ways that the $m$ samples can be assigned to $k$ clusters, we can find a subsequence of $\mathbf{y}$ from \eqref{Sequence} in which the elements are the same. Based on the assumptions, there are two integers $t,p>0$ such that $G(\mathbf{y}^{(t)};W^{(t)},\mathbf{b}^{(t)})=G(\mathbf{y}^{(t)};W^{(t+p)},\mathbf{b}^{(t+p)})$, where $\mathbf{y}^{(t)}$ belongs to the above subsequence of $\mathbf{y}$. Since the objective $G$ is non-increasing in the iteration, it is invariable from the step $t$ to $t+p$. Therefore, the general model terminates before or at step $t+p$ under termination condition (iii).
\end{proof}

Generally speaking, the general model may also terminate in a finite number of steps with other termination conditions. However, the termination point obtained by the general model would be very different under different termination conditions. In fact, when the general model terminates, there should not be any other available points which make the objective function $G(\cdot)$ decrease.
To study the convergence of the general model further, we introduce two definitions.

\begin{definition}
($\mathbf{Local~optimal~point}$ by O.L. Mangasarian in \cite{Kplane}) Point $(\mathbf{y}^*;W^*,\mathbf{b}^*)$ is defined as the local optimal point to the function $G(\mathbf{y};W,\mathbf{b})$ if $\mathbf{y}^*$ is the global solution to the problem $\min~G(\mathbf{y};W^*,~\mathbf{b}^*)$, and meanwhile $(W^*,~\mathbf{b}^*)$ is the global solution to the problem $\min~G(\mathbf{y}^*;W,~\mathbf{b})$.
\end{definition}

\begin{definition}
($\mathbf{Weak~local~optimal~point}$) Point $(\mathbf{y}^{**};W^{**},\mathbf{b}^{**})$ is defined as the weak local optimal point to the function $G(\mathbf{y};W,\mathbf{b})$ if $\mathbf{y}^{**}$ is the global solution to the problem $\min~G(\mathbf{y};W^{**},~\mathbf{b}^{**})$, and meanwhile $(W^{**},~\mathbf{b}^{**})$ is a local solution to the problem $\min~G(\mathbf{y}^{**};W,~\mathbf{b})$.
\end{definition}

Now, we have the following two theorems.

\begin{theorem}
The general model with termination condition (i) or (ii) terminates at a local optimal point if the solution in step 2(a) means global solution.
\end{theorem}
\begin{proof}
From the proof of Theorem 2.2, there is a finite number $t>0$ such that equations \eqref{Equals} hold. Thus, the point $(\mathbf{y}^{(t)};W^{(t)},\mathbf{b}^{(t)})$ is a local optimal point and $(\mathbf{y}^{(t)};W^{(t)},\mathbf{b}^{(t)})=(\mathbf{y}^{(t+1)};W^{(t+1)},\mathbf{b}^{(t+1)})$. Then, the general model terminates at step $t+1$ under termination condition (i) or (ii).
\end{proof}

\begin{theorem}
Suppose the number of the local solutions or the local optimal values to the problem $\underset{W,\mathbf{b}}{\min}~G(\mathbf{y};W,\mathbf{b})$ is finite.
The general model with termination condition (iii) terminates at a weak local optimal point if the solution in step 2(a) means local solution.
\end{theorem}
\begin{proof}
From the proof of Theorem 2.3, there are two finite integers $t,p>0$ such that $G(\mathbf{y}^{(t)};W^{(t)},\mathbf{b}^{(t)})=G(\mathbf{y}^{(t)};W^{(t+p)},\mathbf{b}^{(t+p)})$. Due to the non-increase of the objective $G$ in the iteration, equations \eqref{Equals} hold and $\mathbf{y}^{(t)}=\mathbf{y}^{(t+p)}$. Note that $\mathbf{y}^{(t+1)}$ is the global solution to the problem $\underset{\mathbf{y}}{\min}~G(\mathbf{y};W^{(t)},\mathbf{b}^{(t)})$, and $G(\mathbf{y}^{(t)};W^{(t)},\mathbf{b}^{(t)})$ also attains the same optimal value. It shows that $\mathbf{y}^{(t)}$ is also the global solution to the above problem. Thus, $\mathbf{y}^{(t)}=\mathbf{y}^{(t+1)}$ holds because of the uniqueness of the assigned labels guaranteed in step 2(b), and then $(\mathbf{y}^{(t)};W^{(t)},\mathbf{b}^{(t)})$ is a weak local optimal point. Therefore, the conclusion holds by Theorem 2.3.
\end{proof}

\section{Reorganization of the plane-based clustering methods}
In this section, we show that the general model yields current plane-based clustering methods by selecting different deviation formations and loss functions.
\subsection{kPC}
kPC \cite{Kplane} is the first plane-based clustering method. It starts with a random assignment of the samples. Then, for the $j$th cluster ($j=1,\ldots,k$), its cluster center \eqref{Center} requires the samples be along with it by solving the following problem
\begin{equation}\label{kPC1}
\begin{array}{l}
\underset{w_j,b_j}{\min}~~\sum\limits_{\begin{subarray}{c}
i=1\\
    y(x_i)=j
  \end{subarray}}^m(w_j^\top x_i+b_j)^2.\\
  \text{s.t.}~~~~||w_j||=1.
\end{array}
\end{equation}
When the $k$ cluster centers are obtained, the samples are reassigned to $k$ clusters by
\begin{equation}\label{Predict-base}
\begin{array}{l}
y(x_i)=\underset{j}{\arg}\min \frac{|w_j^\top x_i+b_j|}{||w_j||}.
\end{array}
\end{equation}
The cluster centers and the samples' labels are updated alternately until termination condition (i) or (ii) is satisfied.

\begin{figure*}[htbp]
\centering
\subfigure[kPC]{\includegraphics[width=0.23\textheight]{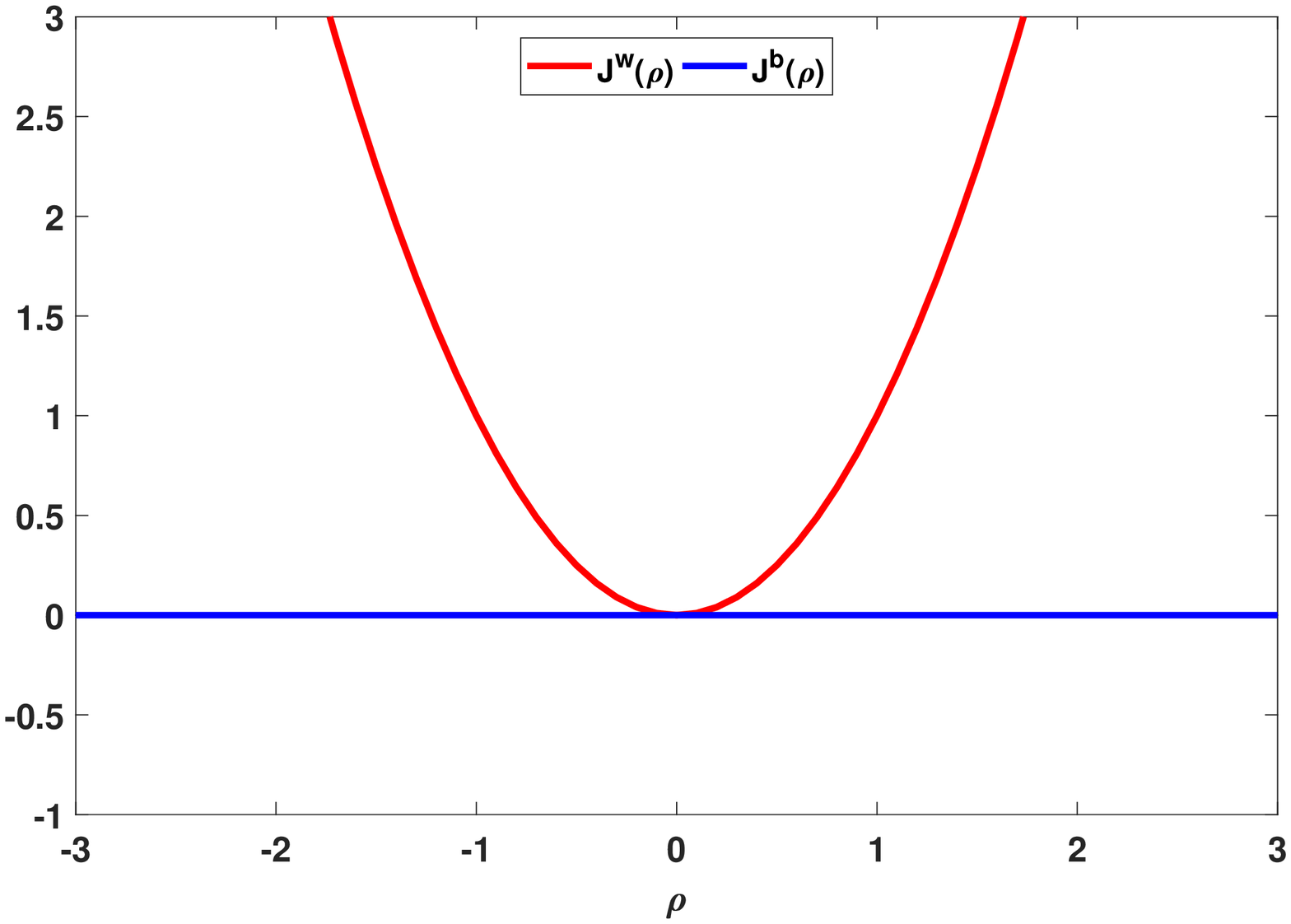}}
    \subfigure[PPC]{\includegraphics[width=0.23\textheight]{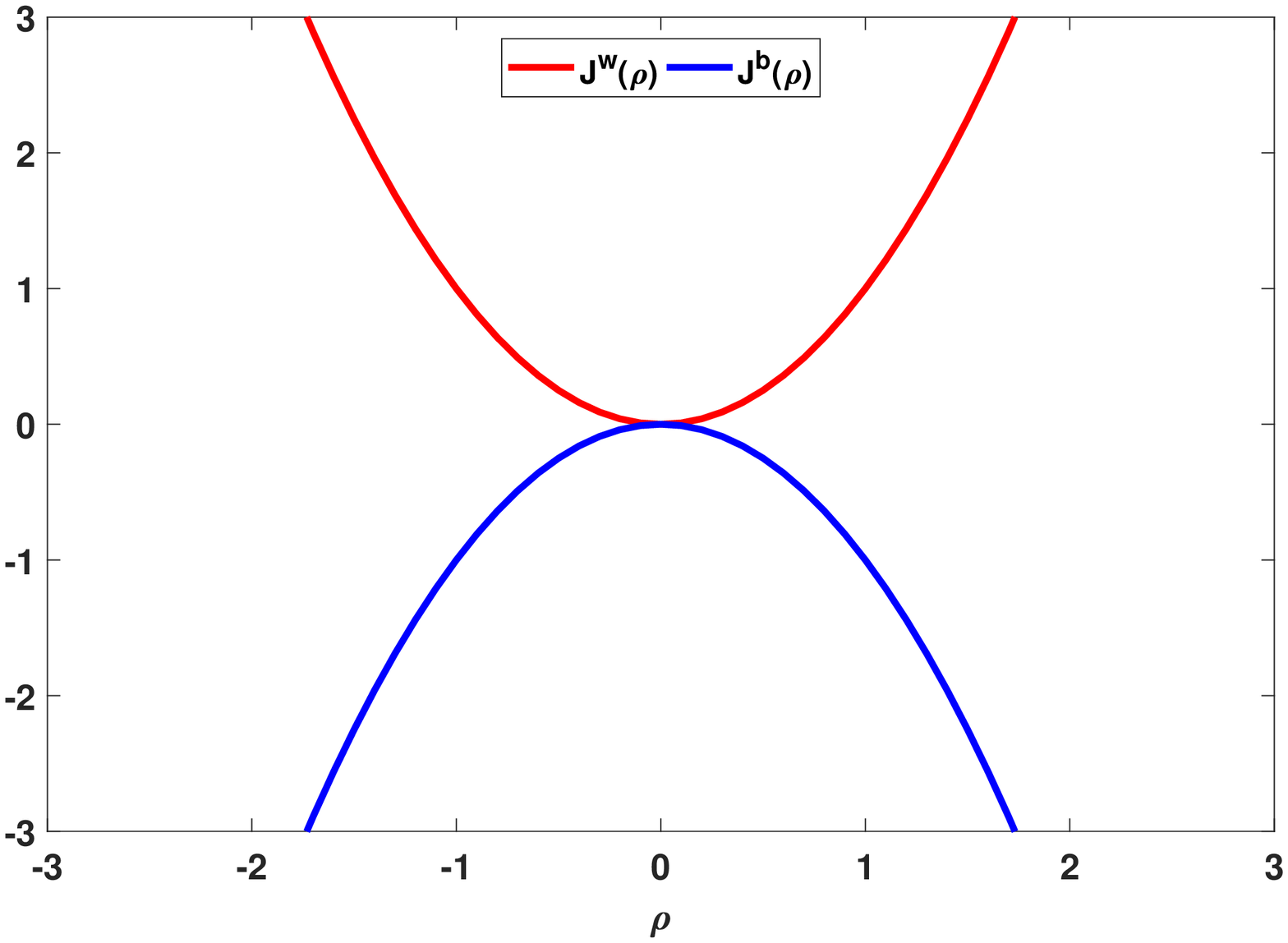}}
    \subfigure[TWSVC]{\includegraphics[width=0.23\textheight]{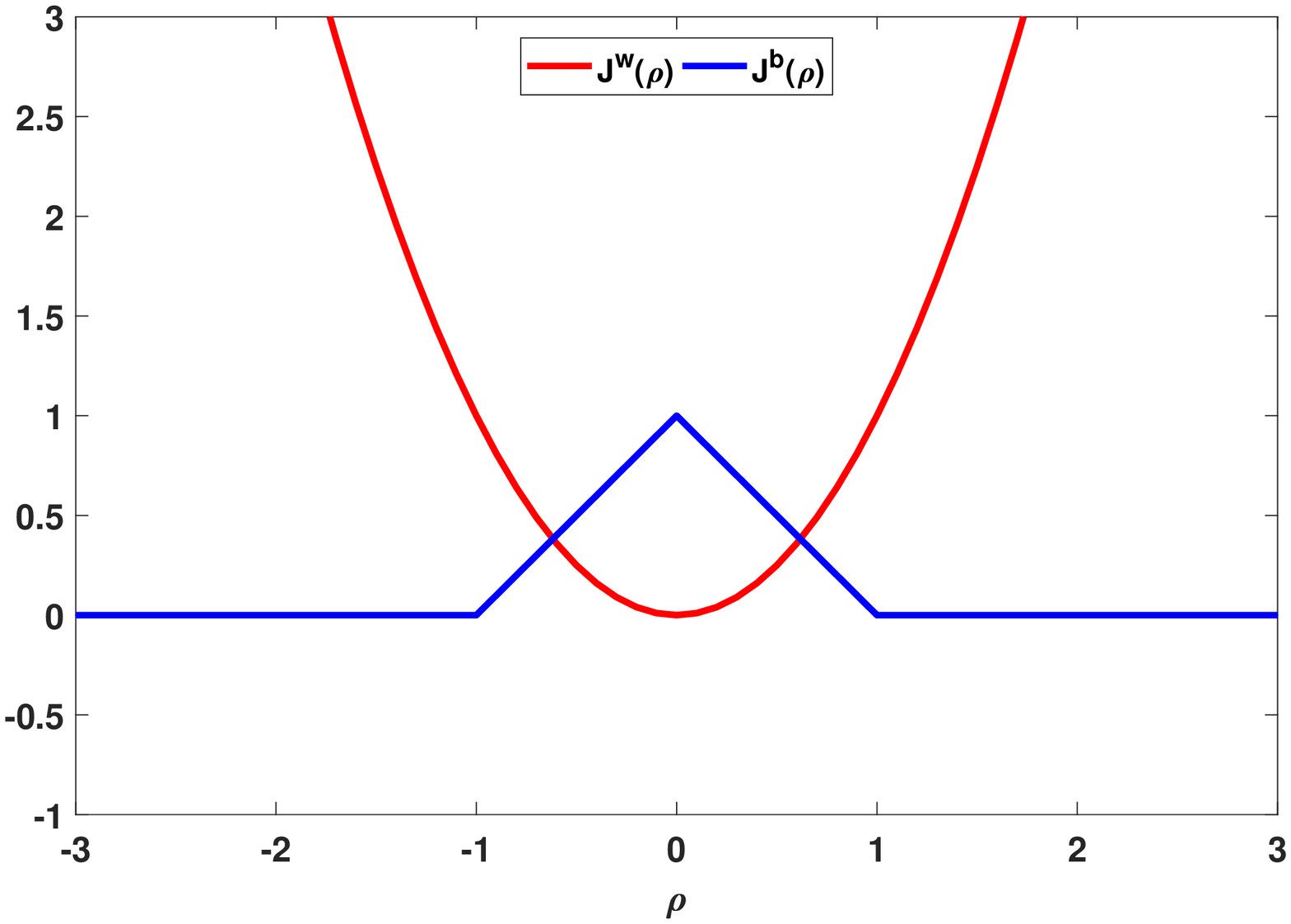}}
    \subfigure[RTWSVC]{\includegraphics[width=0.23\textheight]{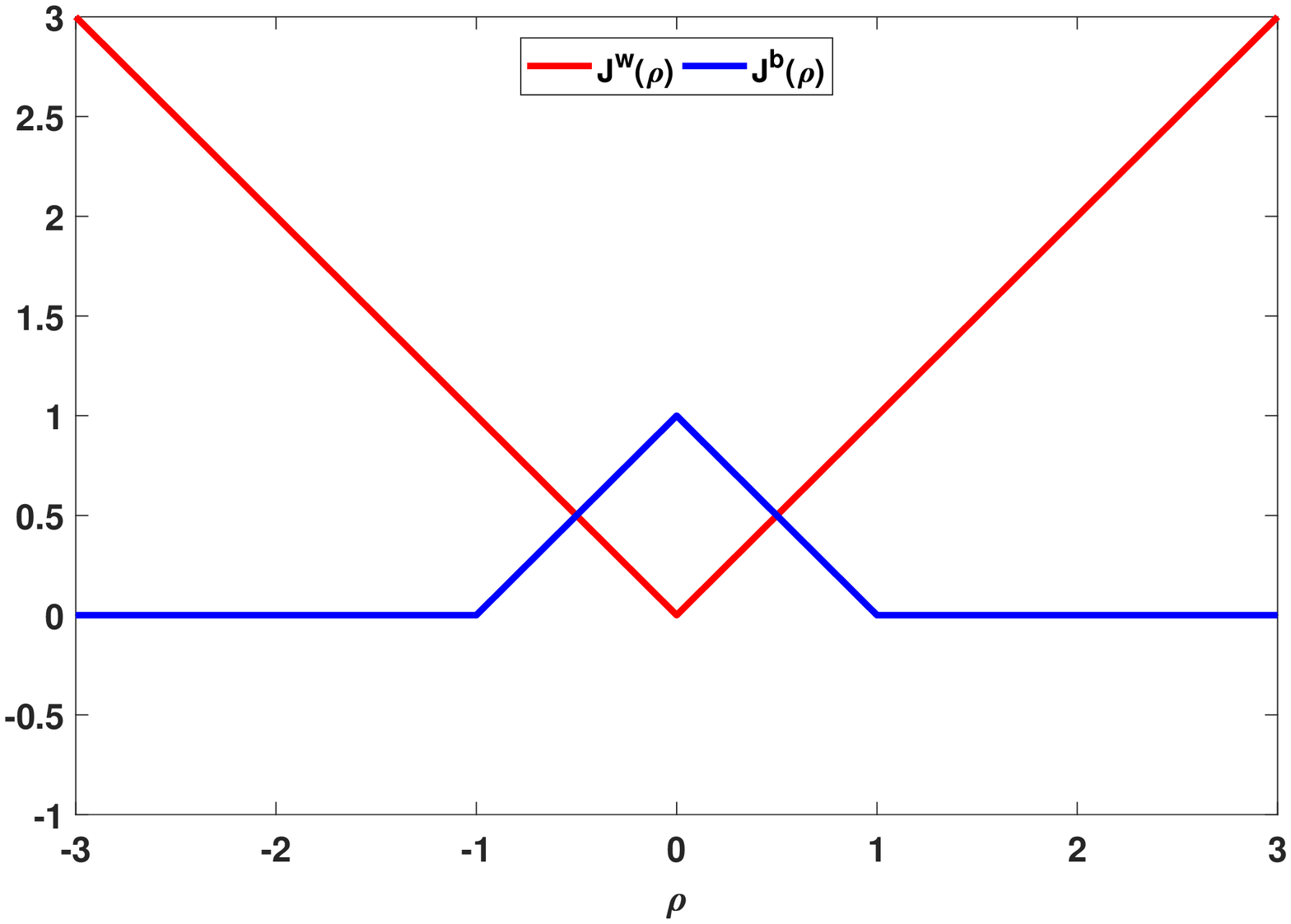}}
    \subfigure[FRTWSVC]{\includegraphics[width=0.23\textheight]{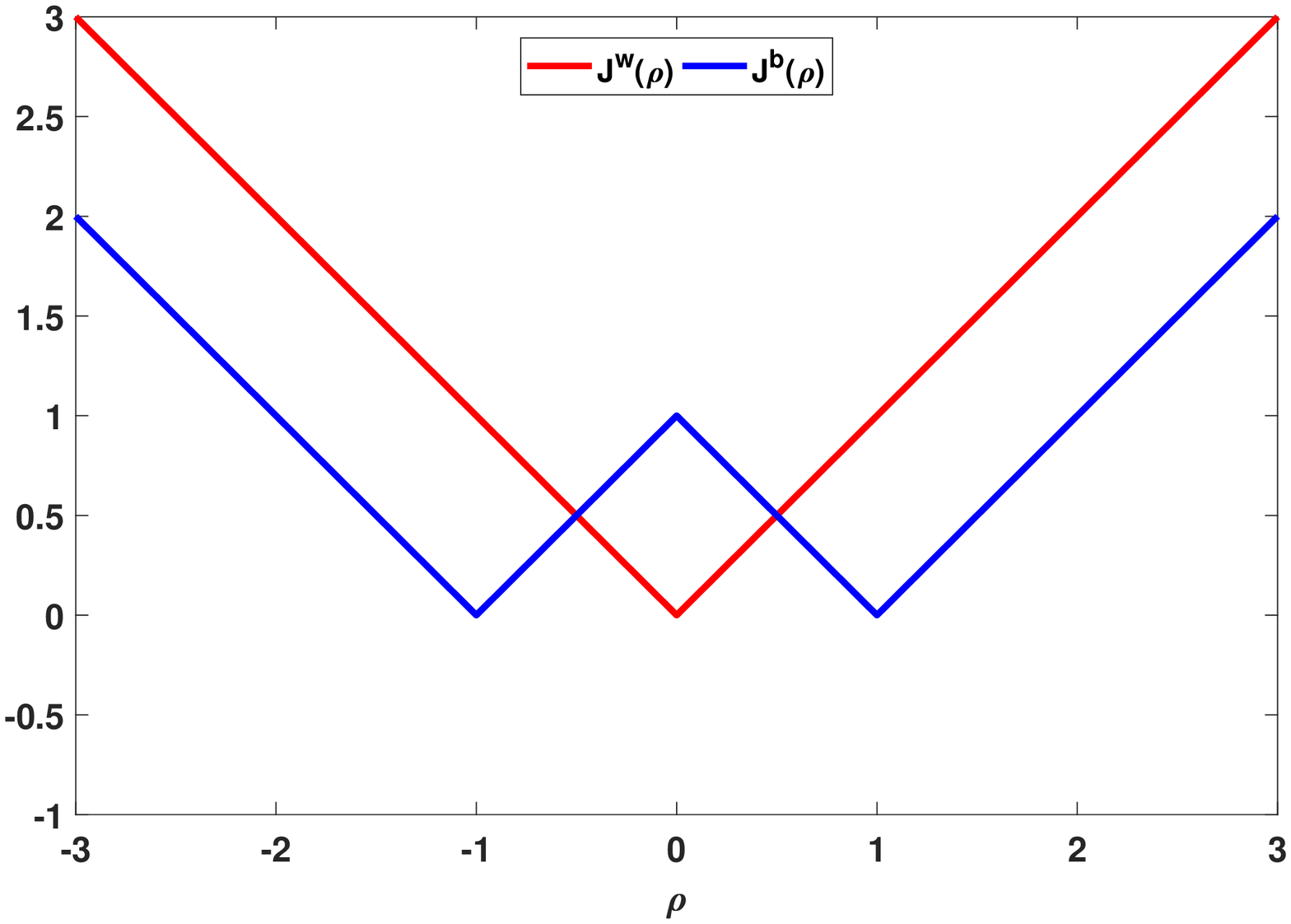}}
    \subfigure[RampTWSVC]{\includegraphics[width=0.23\textheight]{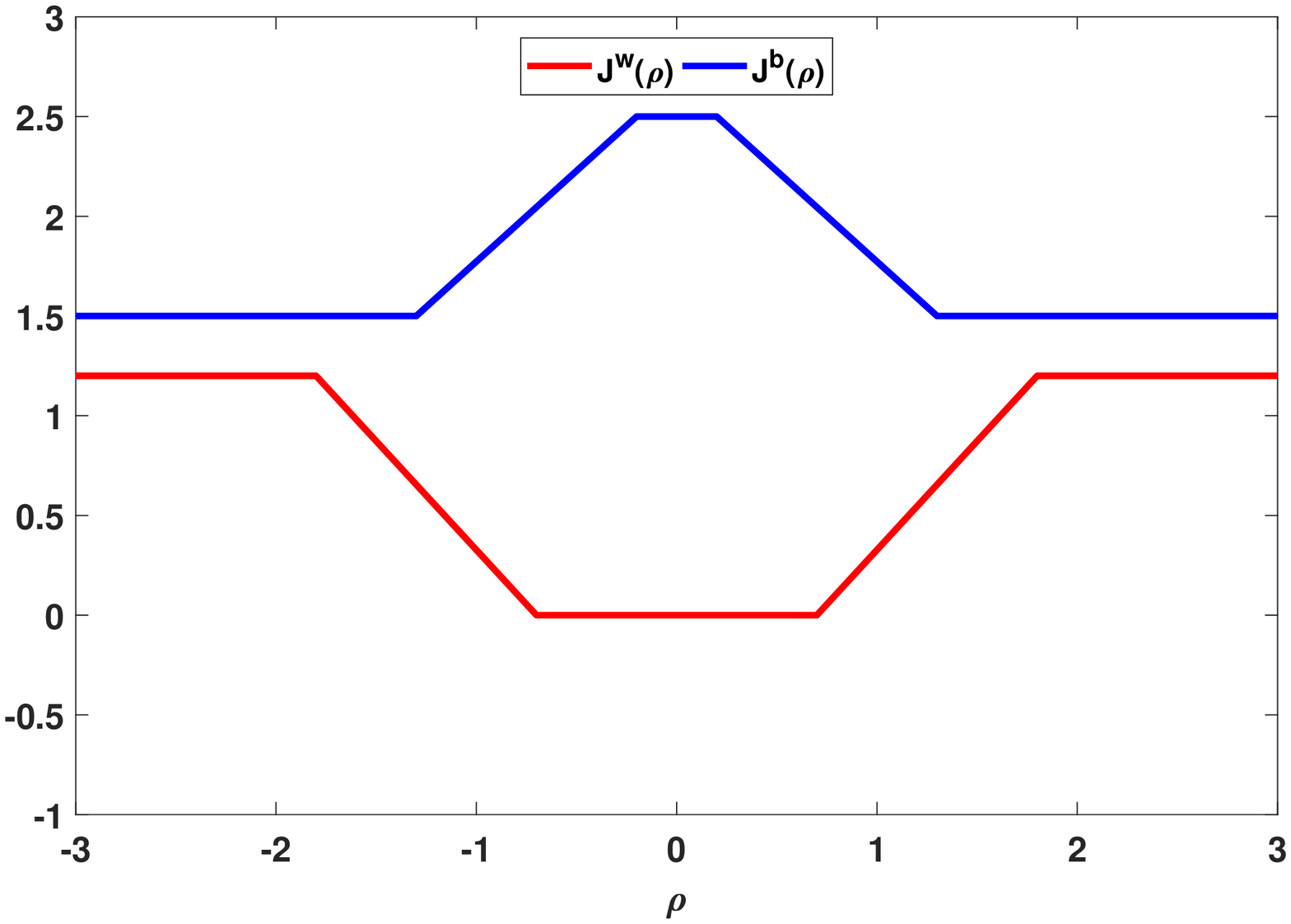}}
\caption{Within-cluster function $J^w(\rho)$ and between-cluster function $J^b(\rho)$ selected in different plane-based clustering methods, where $c_w=c_b=1,~\Delta=0.3$, and $s=-0.2$.}\label{LossFig}
\end{figure*}

To organize kPC by the general model, we select $f(x;w_j,b_j)=\frac{w_j^\top x+b_j}{||w_j||}$ and hire the within-cluster function $J_{1}^w(\rho)=\rho^2$ and between-cluster function $J_{1}^b(\rho)=0$ with $c_w=1$ (see Fig. \ref{LossFig}(i)). Thus, the loss of the sample $x_i$ ($i=1,\ldots,m$) is
\begin{equation}\label{LosskPC}
\begin{array}{l}
L_1(y(x_i);F(x_i;W,\mathbf{b}))=(w_{y(x_i)}^\top x_i+b_{y(x_i)})^2.
\end{array}
\end{equation}
Without any difficulty, we can use the general model to generate a plane-based clustering method by using the loss function \eqref{LosskPC}. By setting $\tilde{w}_j=w_j/||w_j||$ and $\tilde{b}_j=b_j/||w_j||$, problem \eqref{Sub1} solved in the general model is equivalent to problem \eqref{kPC1} in kPC. Since $J_{1}^w(\rho)$ and $J_{1}^b(\rho)$ satisfy the conditions of Theorem 2.1, it is easy to conclude that kPC is consistent with the general model by the loss function \eqref{LosskPC}.
The global solution to problem \eqref{kPC1} can be obtained by solving an eigenvalue problem, and we immediately conclude that kPC finitely terminates at a local optimal point by Theorem 2.4 (this finite termination has been proven by Mangasarian, see Theorem 7 in \cite{Kplane}).

It is worth to notice that kPC only considers the discriminative information from within-cluster. The following PPC was proposed by introducing the discriminative information from between-cluster.

\subsection{PPC}
The procedure of PPC \cite{PPC,kPPC} is similar to kPC, where the only difference is the stage of reconstructing the cluster centers. PPC considers the samples from the current cluster should close to its cluster center, and  meanwhile the samples from different clusters should be far away from it. The $j$th ($j=1,\ldots,k$) cluster center plane is obtained by solving following problem
\begin{equation}\label{PPC1}
\begin{array}{l}
\underset{w_j,b_j}{\min}~~\sum\limits_{\begin{subarray}{c}
i=1\\
    y(x_i)=j
  \end{subarray}}^m(w_j^\top x_i+b_j)^2-c\sum\limits_{\begin{subarray}{c}
i=1\\
    y(x_i)\neq j
  \end{subarray}}^m(w_j^\top x_i+b_j)^2\\
\text{s.t.}~~~~||w_j||=1,
\end{array}
\end{equation}
where $c$ is a positive parameter.

Similarly, to organize PPC by the general model, we select $f(x;w_j,b_j)=\frac{w_j^\top x+b_j}{||w_j||}$ and hire the functions $J_{2}^w(\rho)=\rho^2$ and $J_{2}^b(\rho)=-\rho^2$ with $c_w=1,~c_b=c$ (see Fig. \ref{LossFig}(ii)). Thus, the loss of the sample $x_i$ ($i=1,\ldots,m$) is
\begin{equation}\label{LossPPC}
\begin{array}{l}
L_2(y(x_i);F(x_i;W,\mathbf{b}))=(w_{y(x_i)}^\top x_i+b_{y(x_i)})^2-\\c\sum\limits_{\begin{subarray}{c}
j=1\\
    j\neq y(x_i)
  \end{subarray}}^k(w_j^\top x_i+b_j)^2.
\end{array}
\end{equation}
Obviously, $J_{2}^w(\rho)$ and $J_{2}^b(\rho)$ satisfy the conditions of Theorem 2.1. Therefore, PPC can be regarded as the general model by using the loss function \eqref{LossPPC}.
Since the global solution to problem \eqref{PPC1} can be obtained by solving an eigenvalue problem, we can immediately conclude that PPC finitely terminates at a local optimal point by Theorem 2.4, which was not provided previously. By the loss function \eqref{LossPPC}, it can be seen that PPC uses $L_2$ norm to measure the discriminative information from between-cluster, which may be sensitive with noise or outliers.

\subsection{TWSVC}
To reduce the influence of the noise and outliers, TWSVC \cite{TWSVC} makes the samples from different clusters far away from the cluster center to a certain distance. The $j$th ($j=1,\ldots,k$) cluster center is considered from following problem
\begin{equation}\label{TWSVC1}
\begin{array}{l}
\underset{w_j,b_j,\xi_i}{\min}~~\sum\limits_{\begin{subarray}{c}
i=1\\
    y(x_i)=j
  \end{subarray}}^m(w_j^\top x_i+b_j)^2
+c\sum\limits_{\begin{subarray}{c}
i=1\\
    y(x_i)\neq j
  \end{subarray}}^m\xi_i\\
\text{s.t.}~~~~|w_j^\top x_i+b_j|\geq1-\xi_i,\xi_i\geq0,y(x_i)\neq j,\\
~~~~~~~~i=1,\ldots,m,
\end{array}
\end{equation}
where $\xi_i\in R$ is a slack variable.

By selecting $f(x;w_j,b_j)=w_j^\top x+b_j$ and hiring the functions $J_{3}^w(\rho)=\rho^2$ and $J_{3}^b(\rho)=(1-|\rho|)_+$ with $c_w=1$, $c_b=c$ (see Fig. \ref{LossFig}(iii)), the loss of the sample $x_i$ ($i=1,\ldots,m$) is
\begin{equation}\label{LossTWSVC}
\begin{array}{l}
L_3(y(x_i);F(x_i;W,\mathbf{b}))=(w_{y(x_i)}^\top x_i+b_{y(x_i)})^2+\\c\sum\limits_{\begin{subarray}{c}
j=1\\
    j\neq y(x_i)
  \end{subarray}}^k(1-|w_j^\top x_i+b_j|)_+,
\end{array}
\end{equation}
where $(\cdot)_+$ replaces the negative value with zero.
Obviously, $J_{3}^w(\rho)$ and $J_{3}^b(\rho)$ satisfy the conditions of Theorem 2.1. Therefore, TWSVC can be regarded as our general model by using the loss function \eqref{LossTWSVC} except a slight difference in the solution to problem \eqref{TWSVC1}, which is obtained independently. It is worth to mention that if TWSVC is implemented by the general model strictly, it would terminate in a finite number of steps at a weak local optimal point by Theorem 2.5.

\begin{figure*}[htb]
\begin{center}{
\subfigure{
\resizebox*{8cm}{!}
{\includegraphics{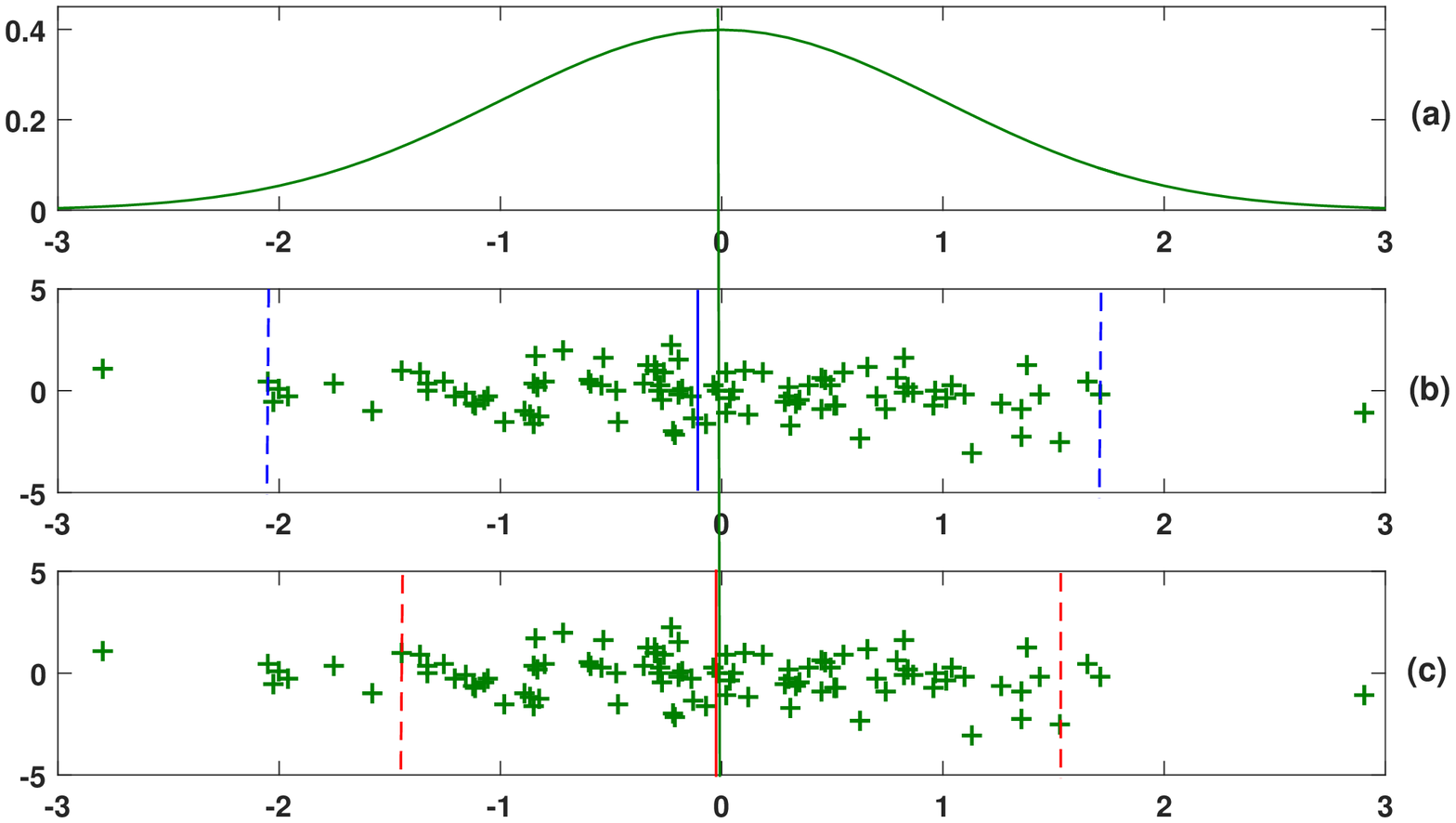}}}\hspace{5pt}
\subfigure{
\resizebox*{8cm}{!}
{\includegraphics{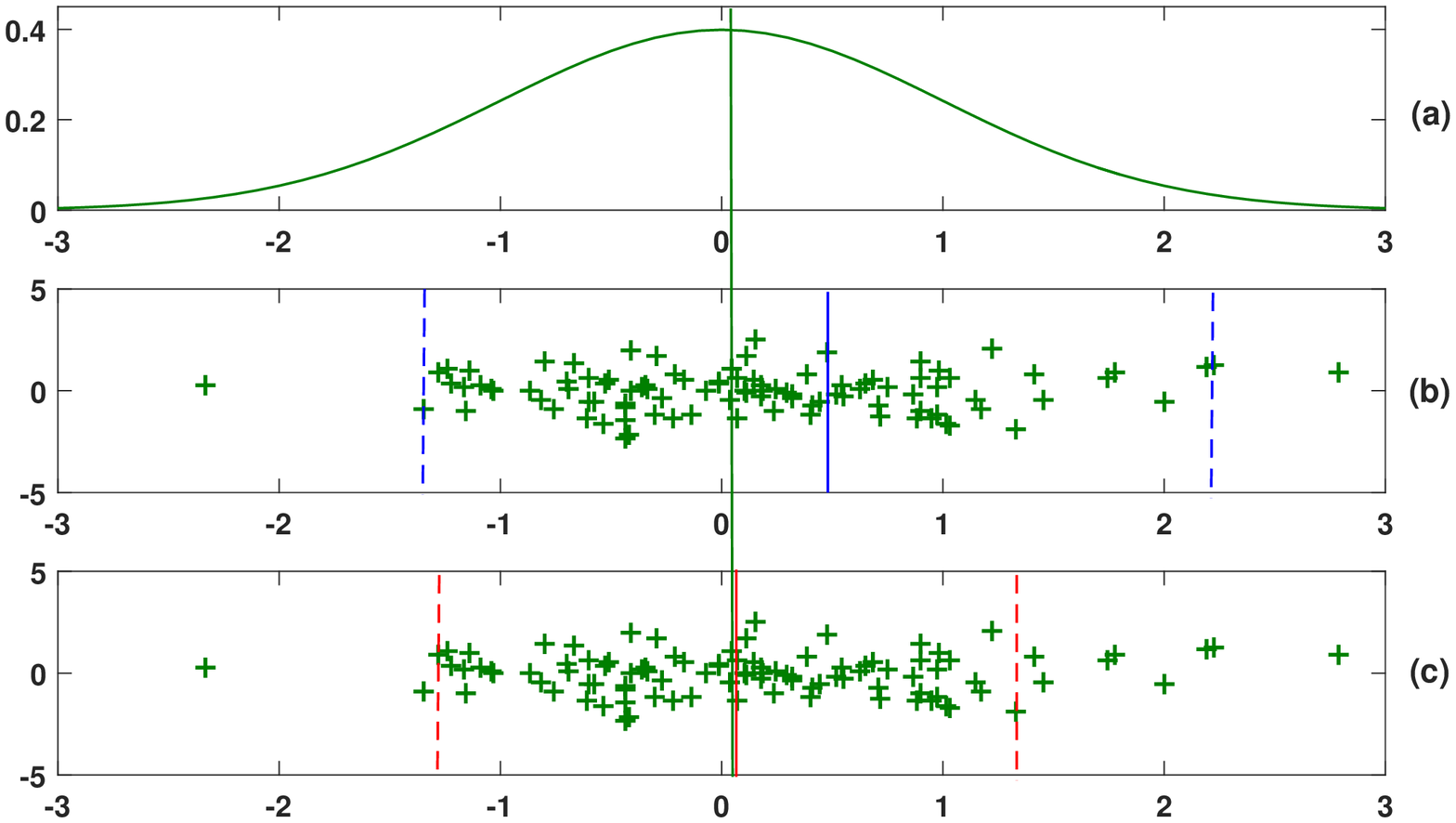}}}\hspace{5pt}
\caption{Illustration of RampTWSVC and RFDPC on two groups of 100 samples from the same data distribution. The vertical green line is the mean of data samples. (a) is the image of normal distribution $N(0,1)$, (b) is 100 samples from the distribution (where the solid blue lines are the centers by $J_6^w(\rho)$ used in RampTWSVC), and (c) is 100 samples the same as (b) (where the solid red lines are the centers by $J_7^w(\rho)$ used in our RFDPC).}\label{RFDPCfig}
}
\end{center}
\end{figure*}

\subsection{Extensions on TWSVC}
There are several extensions on TWSVC due to its stable performance. For instance, RTWSVC \cite{RTWSVC} replaces $L_2$ norm with $L_1$ norm in the within-cluster function, resulting in decreasing the influence of the noise and outliers further. Another extension FRTWSVC \cite{RTWSVC} uses a least squares formation to accelerate the learning speed. The third extension RampTWSVC \cite{RampTWSVC} introduces the ramp loss function into TWSVC to further decrease the influence of the noise and outliers from both within-cluster and between-cluster. They construct the cluster centers by different optimization problems. By selecting $f(x;w_j,b_j)=w_j^\top x+b_j$, we summarize their within-cluster, between-cluster and loss functions (see Fig. \ref{LossFig}(iv)-(vi)) as follows.
\begin{equation}\label{loss3}
\begin{array}{l}
\text{RTWSVC:}\\
J_{4}^w(\rho)=|\rho|,~~~~J_{4}^b(\rho)=(1-|\rho|)_+,\\
L_4(y(x_i);F(x_i;W,\mathbf{b}))=|w_{y(x_i)}^\top x_i+b_{y(x_i)}|+\\c\sum\limits_{\begin{subarray}{c}
j=1\\
    j\neq y(x_i)
  \end{subarray}}^k(1-|w_j^\top x_i+b_j|)_+;\\
\text{FRTWSVC:}\\
J_{5}^w(\rho)=|\rho|,~~~~J_{5}^b(\rho)=|1-|\rho||,\\
L_5(y(x_i);F(x_i;W,\mathbf{b}))=|w_{y(x_i)}^\top x_i+b_{y(x_i)}|+\\c\sum\limits_{\begin{subarray}{c}
j=1\\
    j\neq y(x_i)
  \end{subarray}}^k|1-|w_j^\top x_i+b_j||;\\
\text{and RampTWSVC:}\\
J_6^w(\rho)=\left\{
\begin{array}{ll}
0&\text{if~} |\rho|\leq 1-\Delta\\
|\rho|-1+\Delta& \text{if~}1-\Delta< |\rho|<2-\Delta-s\\
1-s&\text{if~} |\rho|\geq2-\Delta-s
\end{array}\right.,\\
J_6^{b}(\rho)=\left\{
\begin{array}{ll}
2+2\Delta&\text{if~}|\rho|\leq-s\\
-|\rho|+2+2\Delta-s&\text{if~} -s< |\rho|<1+\Delta\\
1+\Delta-s&\text{if~} |\rho|\geq1+\Delta
\end{array}\right.,\\
L_6(y(x_i);F(x_i;W,\mathbf{b}))=J_6^w(f_{y(x_i)}(x_i))+c\sum\limits_{\begin{subarray}{c}
j=1\\
    j\neq y(x_i)
  \end{subarray}}^kJ_6^b(f_j(x_i)),
\end{array}
\end{equation}
where $\Delta\in[0,1]$, $s\in(-1,0]$ are the user defined constants.

By substituting these loss functions \eqref{loss3} into our general model, it is easy to get
the optimization problems of RTWSVC, FRTWSVC and RampTWSVC. In theory, RTWSVC, FRTWSVC and RampTWSVC would terminate in a finite number of steps at the weak local optimal points if they are implemented by the general model strictly. The details are omitted.

\section{RFDPC}
In this section, we introduce a new loss function fluctuated with the dataset, and then propose our robust fitting distribution planes for clustering (RFDPC) based on the general model.

Let us start from the efficient RampTWSVC \cite{RampTWSVC}. Its ability to reduce the influence of the noise and outliers is manifested in Fig. 1. However,
for the case of the samples from the same distribution, RampTWSVC may obtain very different cluster centers, leading bias from the data distribution. For instance, in Fig. \ref{RFDPCfig}, there are two groups of samples from $N(0,1)$ (i.e., left and right three columns). RampTWSVC obtains two centers, depicted by solid blue lines in Fig. \ref{RFDPCfig}(b), are very different from each other.

To capture the data distribution, we introduce the 1-order and 2-order statistics \cite{STATbook} of the cluster into the within-cluster function and propose a new within-cluster function as
\begin{equation}\label{RwRFDPC}
\begin{array}{ll}
J_{7}^w(f(x;w_j,b_j))=&J_6^w(f(x;w_j,b_j))+\frac{\gamma_1}{c_w}\bar{f}(x;w_j,b_j)^2\\
&+\frac{\gamma_2}{c_w}\tilde{f}(x;w_j,b_j),
\end{array}
\end{equation}
where $\gamma_1,\gamma_2$ are positive parameters. $\bar{f}(x;w_j,b_j)=\frac{1}{|N|}f(x;w_j,b_j)$ and $\tilde{f}(x;w_j,b_j)=\frac{1}{|N|-1}(f(x;w_j,b_j)-\frac{1}{|N|}\sum\limits_{y(x_i)\in N}f(x_i;w_j,b_j))^2$, where $N$ is the index set of the $j$th cluster that $x$ belongs to and $|N|$ denotes the sample number of this cluster. In other words, $\bar{f}(x;w_j,b_j)$ and $\tilde{f}(x;w_j,b_j)$ are the corresponding parts in the mean and variance of the $j$th cluster with $j=1,\ldots,k$. The additional statistics in \eqref{RwRFDPC} mean that a sample $x$ assigned to a cluster would lead additional losses: (i) loss derived from the mean deviation, i.e., the deviation of the sample from the statistical center; (ii) loss derived from the variance of deviation, i.e., the deviation proportionality. Minimizing these statistics would make the cluster center close to the highest density region and the samples be uniformly distributed along with the cluster center. Fig. \ref{RFDPCfig}(c) shows the result by new function \eqref{RwRFDPC}.

Then, by setting the between-cluster function $J_7^b(\rho)=J_6^b(\rho)$, the loss function of RFDPC becomes
\begin{equation}\label{lossRFDPC}
\begin{array}{ll}
L_7(y(x_i),F(x_i))=&c_wJ_7^w(f(x_i;w_{y(x_i)},b_{y(x_i)}))\\
&+c_b\sum\limits_{\begin{subarray}{c}
j=1\\
    j\neq y(x_i)
  \end{subarray}}^kJ_7^b(f(x_i;w_j,b_j)),
\end{array}
\end{equation}
where $f(x;w_j,b_j)=w_j^\top x+b_j$.

By introducing a $L_2$ regularization term, the subproblem in step 2(a) is considered as
\begin{equation}\label{RFDPC1}
\begin{array}{l}
\underset{w_j,b_j}{\min}~~c_w\sum\limits_{\begin{subarray}{c}
i=1\\
    y^{(t)}(x_i)=j
  \end{subarray}}^m J_7^w(f(x_i;w_j,b_j))+\\c_b\sum\limits_{\begin{subarray}{c}
i=1\\
    y^{(t)}(x_i)\neq j
  \end{subarray}}^m J_7^b(f(x_i;w_j,b_j))+\frac{1}{2}(||w_j||^2+b_j^2),
\end{array}
\end{equation}
and its local solution can be obtained by the concave-convex procedure (CCCP) \cite{CCCP}.

It should be pointed out that the cluster assignment \eqref{Sub2} can be replaced by the simplified assignment \eqref{Predict}, though the function $J_7^w(\rho)$ does not satisfy properties (i)-(iii).
\begin{theorem}
In RFDPC, the sample assignment \eqref{Sub2} can be simplified as \eqref{Predict}.
\end{theorem}
\begin{proof}
Suppose $l^*$ is the label of an arbitrary sample $x$ obtained by \eqref{Predict}, and $l$ is an arbitrary label of $x$. From the proof of Theorem 2.1, we just need to prove $J_7^w(f(x;w_{l^*},b_{l^*}))\leq J_7^w(f(x;w_l,b_l))$ and $J_7^b(f(x;w_l,b_l))\leq J_7^b(f(x;w_{l^*},b_{l^*}))$.

Note that $\gamma_1$ and $\gamma_2$ are positive parameters. Since smaller $|f|$ leads smaller $\bar{f}^2$ and smaller $\tilde{f}$, and since $J_6^w(\rho)$ is non-decreasing in $[0,\infty)$, the inequality $J_7^w(f(x;w_{l^*},b_{l^*}))\leq J_7^w(f(x;w_l,b_l))$ holds.
Noticing that $J_7^b(\rho)$ satisfies properties (i)-(iii), the inequality $J_7^b(f(x;w_l,b_l))\leq J_7^b(f(x;w_{l^*},b_{l^*}))$ holds. Therefore, the conclusion is obtained.
\end{proof}

In addition, our RFDPC hires termination condition (iii), and thus it terminates in a finite number of steps at a weak local optimal point by Theorem 2.5.

\section{Experimental results}
In this section, we analyze the performance of our RFDPC compared with some state-of-the-art partition-based clustering methods on several artificial and benchmark datasets. All the methods were implemented by MATLAB2017 on a PC with an Intel Core Duo Processor (double 4.2 GHz) with 16GB RAM.
In the experiments, we used the metrics accuracy (AC) \cite{TWSVC} and mutual information (MI) \cite{MI} to measure the performance of these methods.

\begin{table}
\caption{Details of the synthetic datasets} \centering
\begin{tabular}{l}
\hline
\begin{tabular}{lllll}
\multirow{2}*{Dataset}&\multicolumn{4}{c}{Group}\\ \cline{2-5}
&G1&G2&G3&G4\\
\hline
No. of samples&120&100&80&60\\
No. of dimensions&3&3&3&3
\end{tabular}\\
\hline
\hline
\begin{tabular}{llll}
\multirow{2}*{Distribution}&\multicolumn{3}{c}{Class}\\ \cline{2-4}
&1&2&3\\
\hline
Coordinate x&$\mathcal{N}$(1,1)&$\mathcal{N}$(3,1)&$\mathcal{N}$(2,1)\\
Coordinate y&1&1&$\mathcal{N}$(1,1)\\
Coordinate z&-x+1&x-1&0
\end{tabular}\\
\hline
\end{tabular} \label{Art}\\
\end{table}

\begin{figure*}
\centering
    \subfigure[kPC (Group G1)]{\includegraphics[width=0.168\textheight]{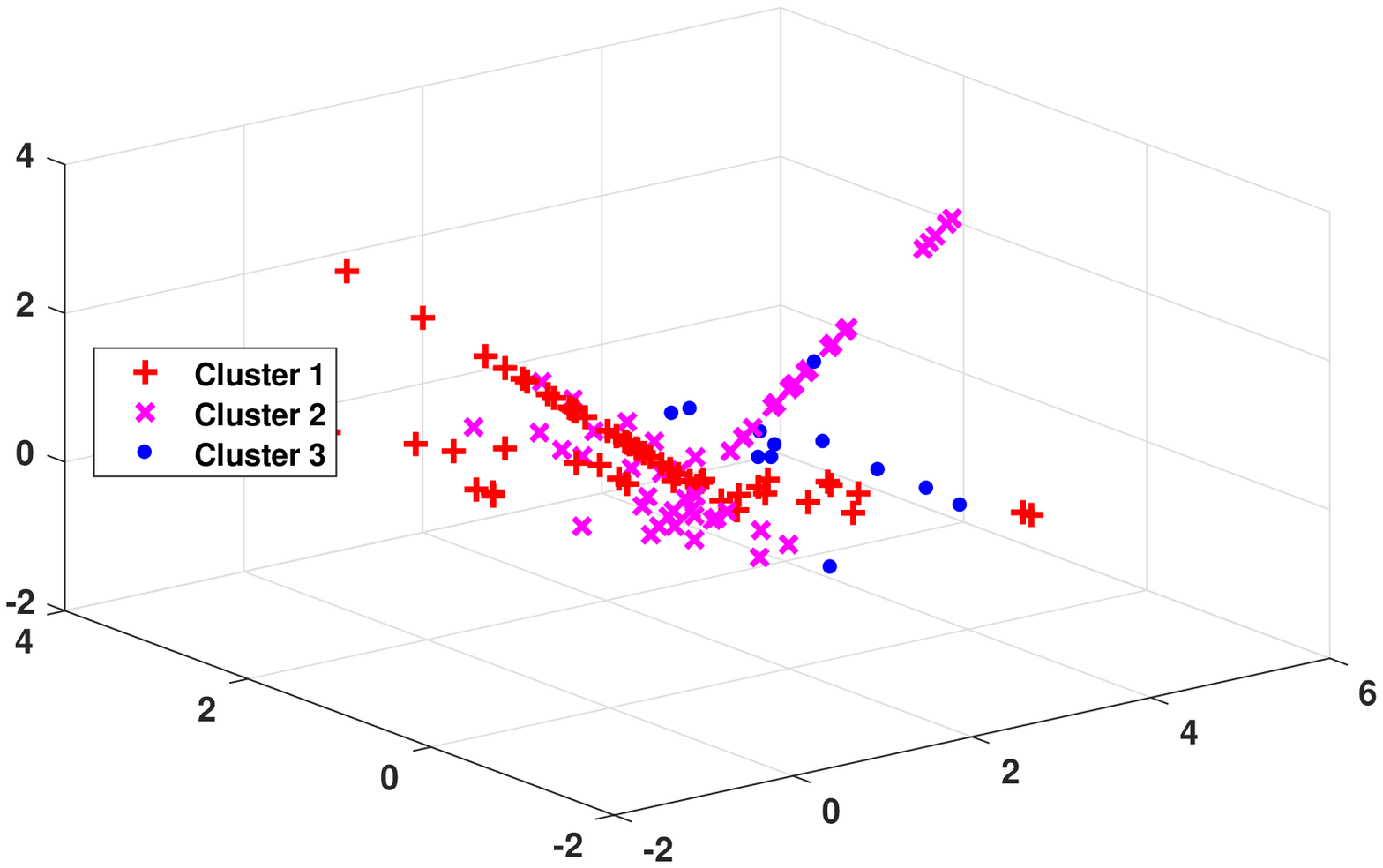}}
    \subfigure[kPC (Group G2)]{\includegraphics[width=0.168\textheight]{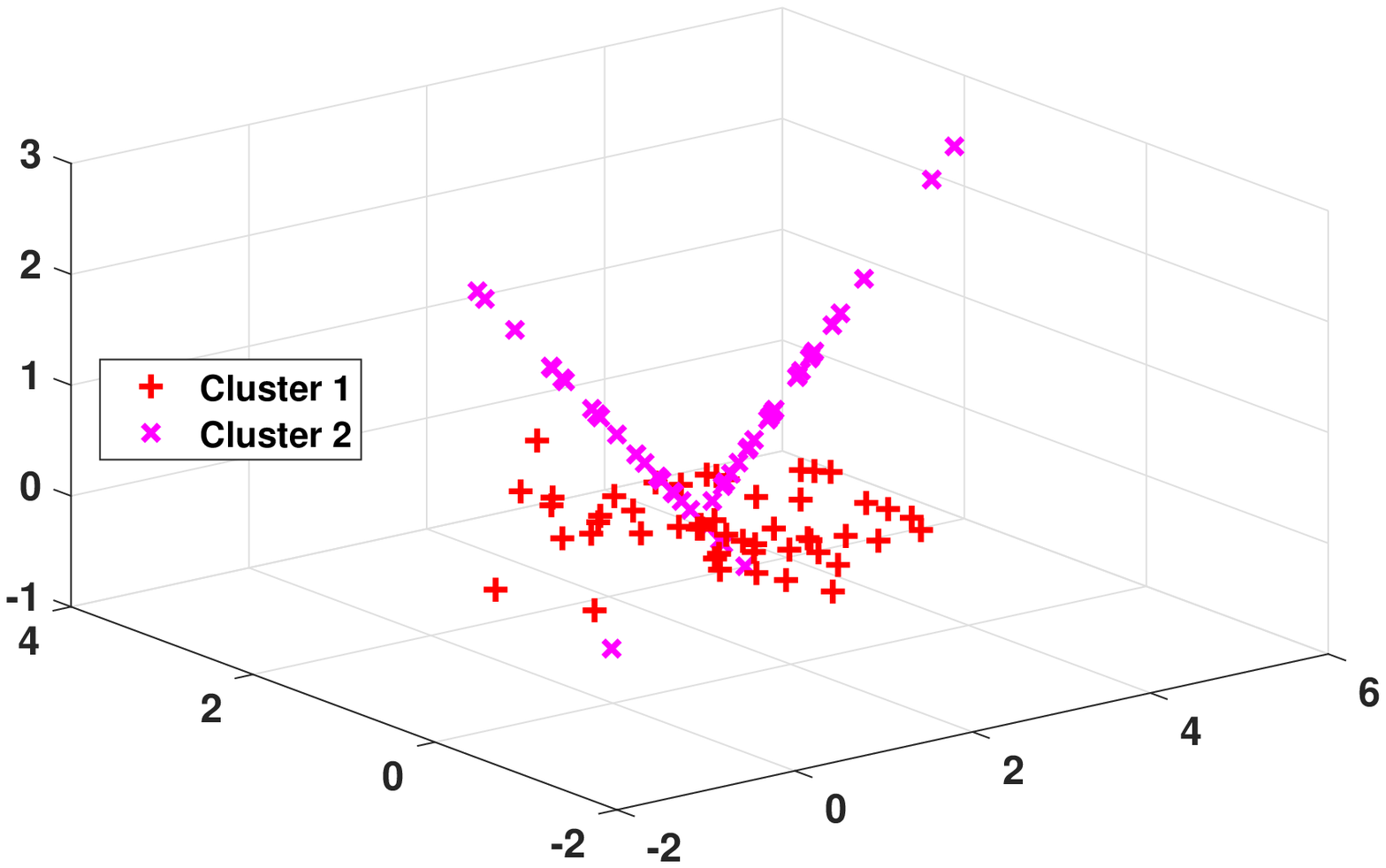}}
    \subfigure[kPC (Group G3)]{\includegraphics[width=0.168\textheight]{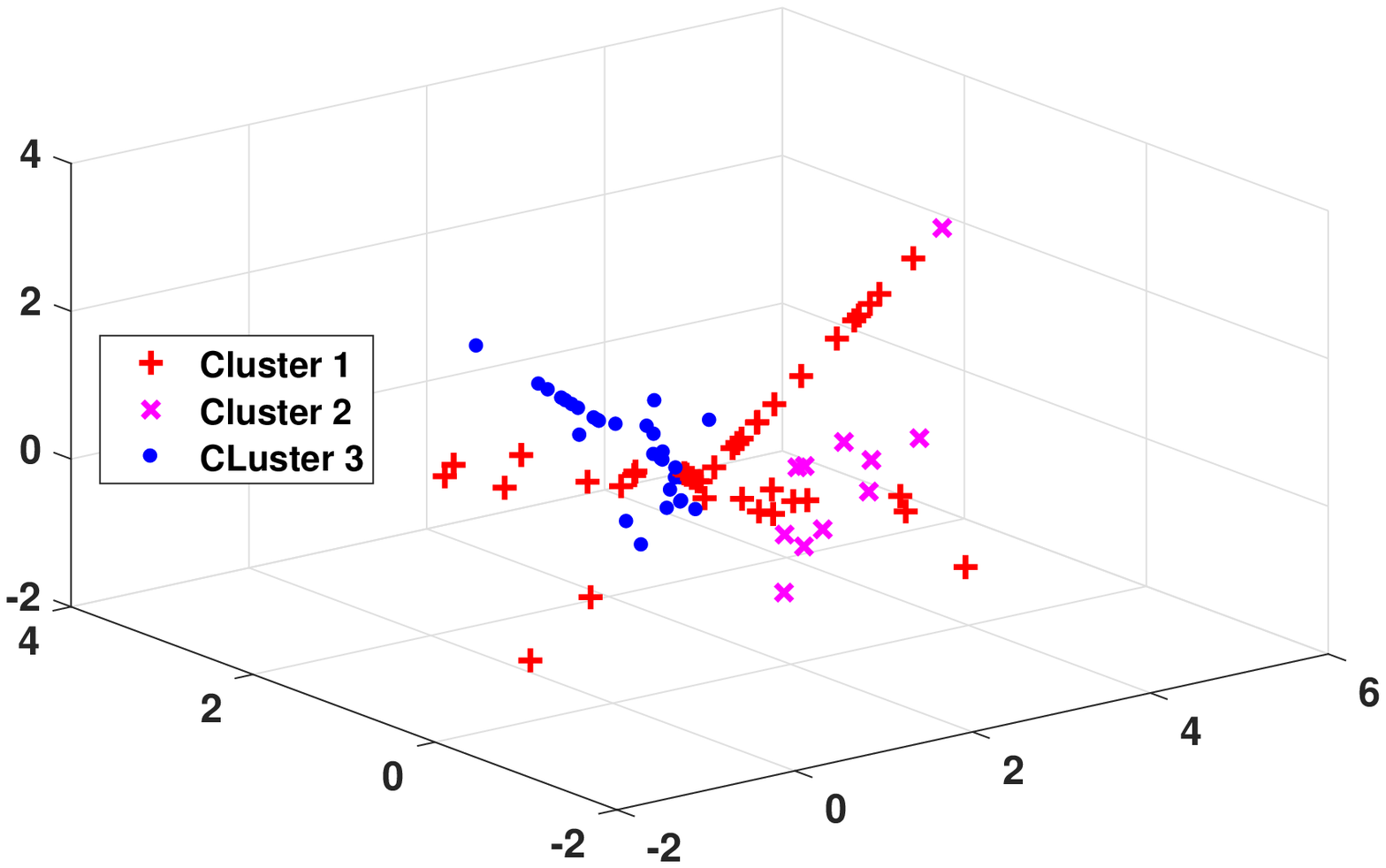}}
    \subfigure[kPC (Group G4)]{\includegraphics[width=0.168\textheight]{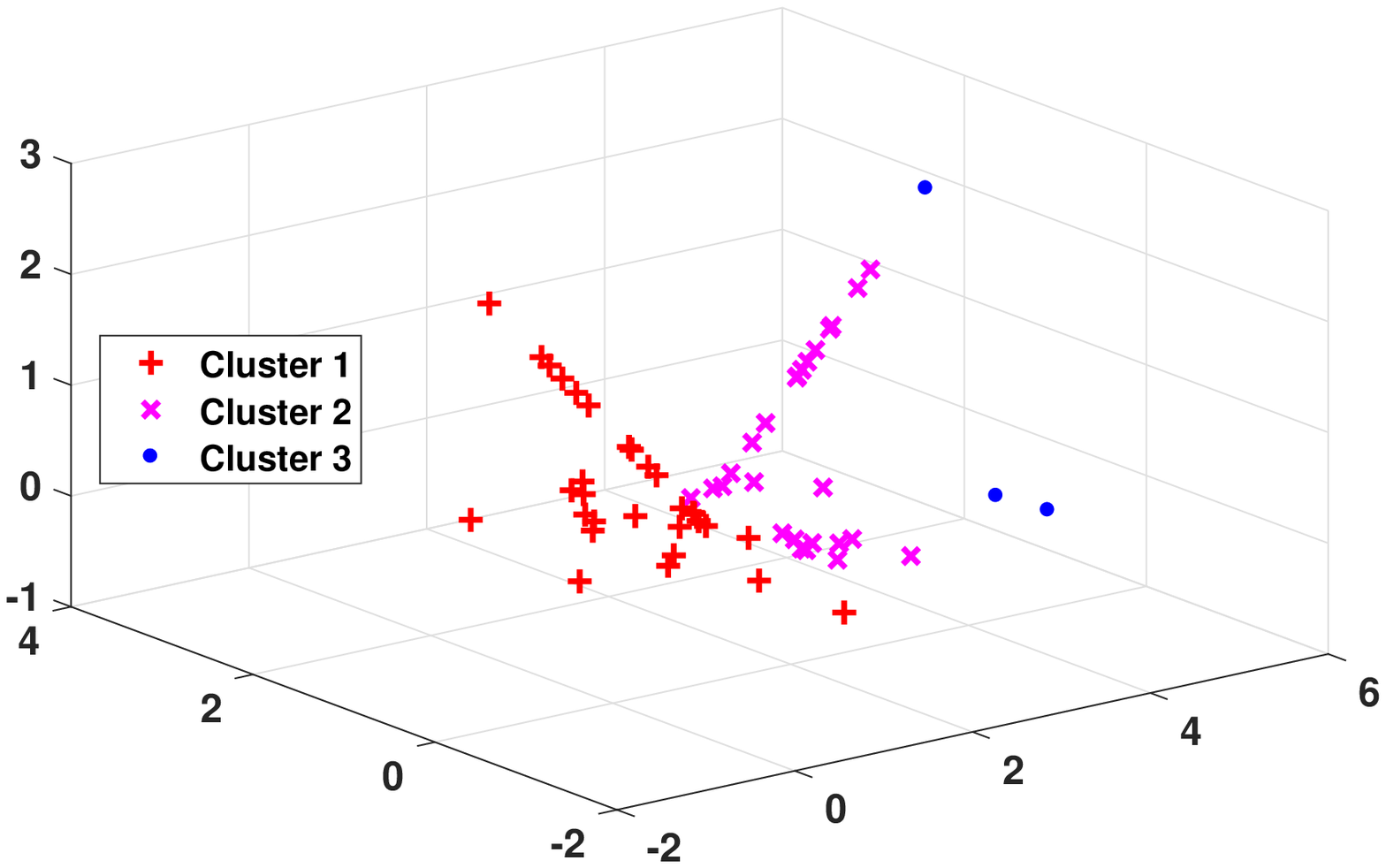}}
    \subfigure[PPC (Group G1)]{\includegraphics[width=0.168\textheight]{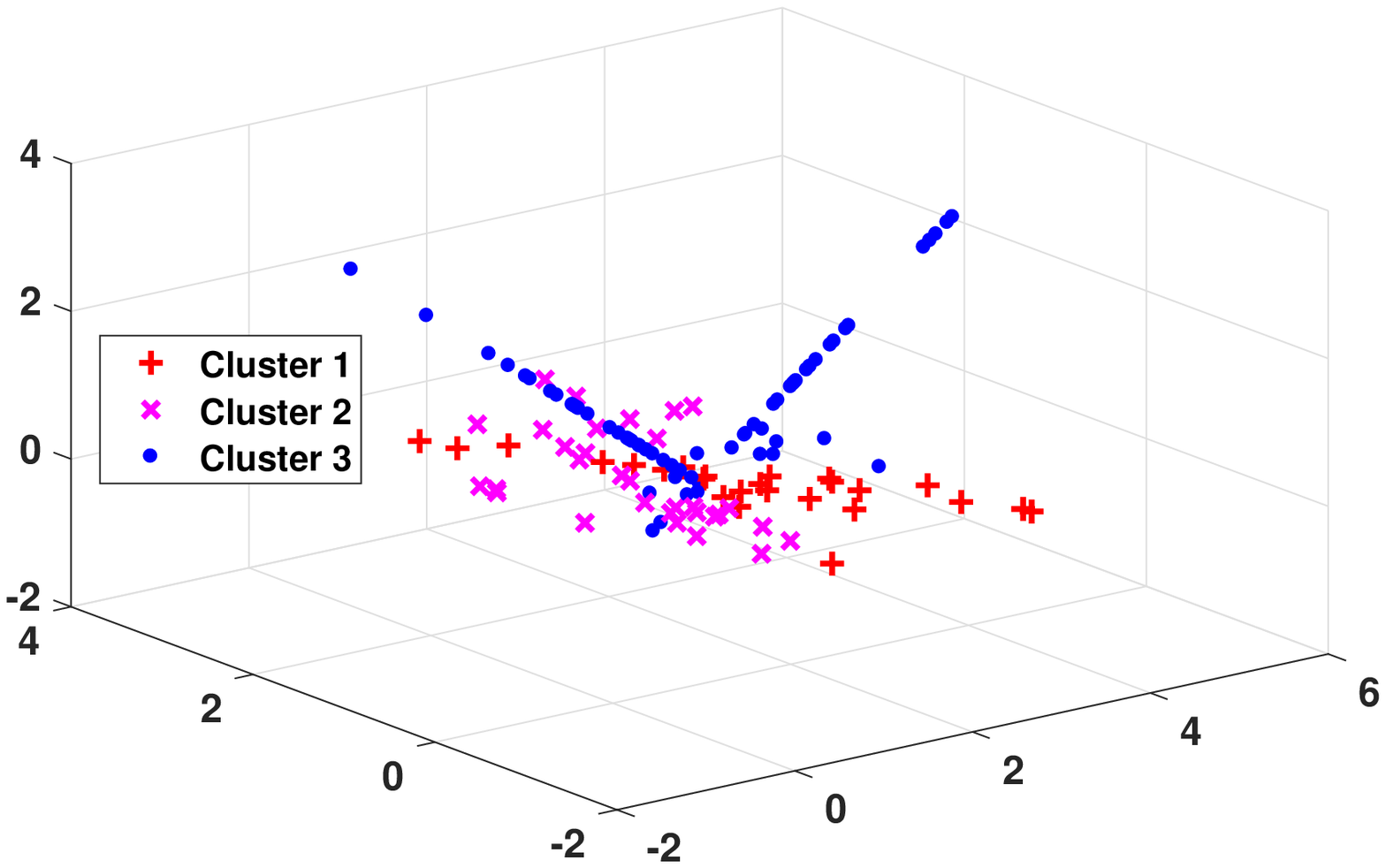}}
    \subfigure[PPC (Group G2)]{\includegraphics[width=0.168\textheight]{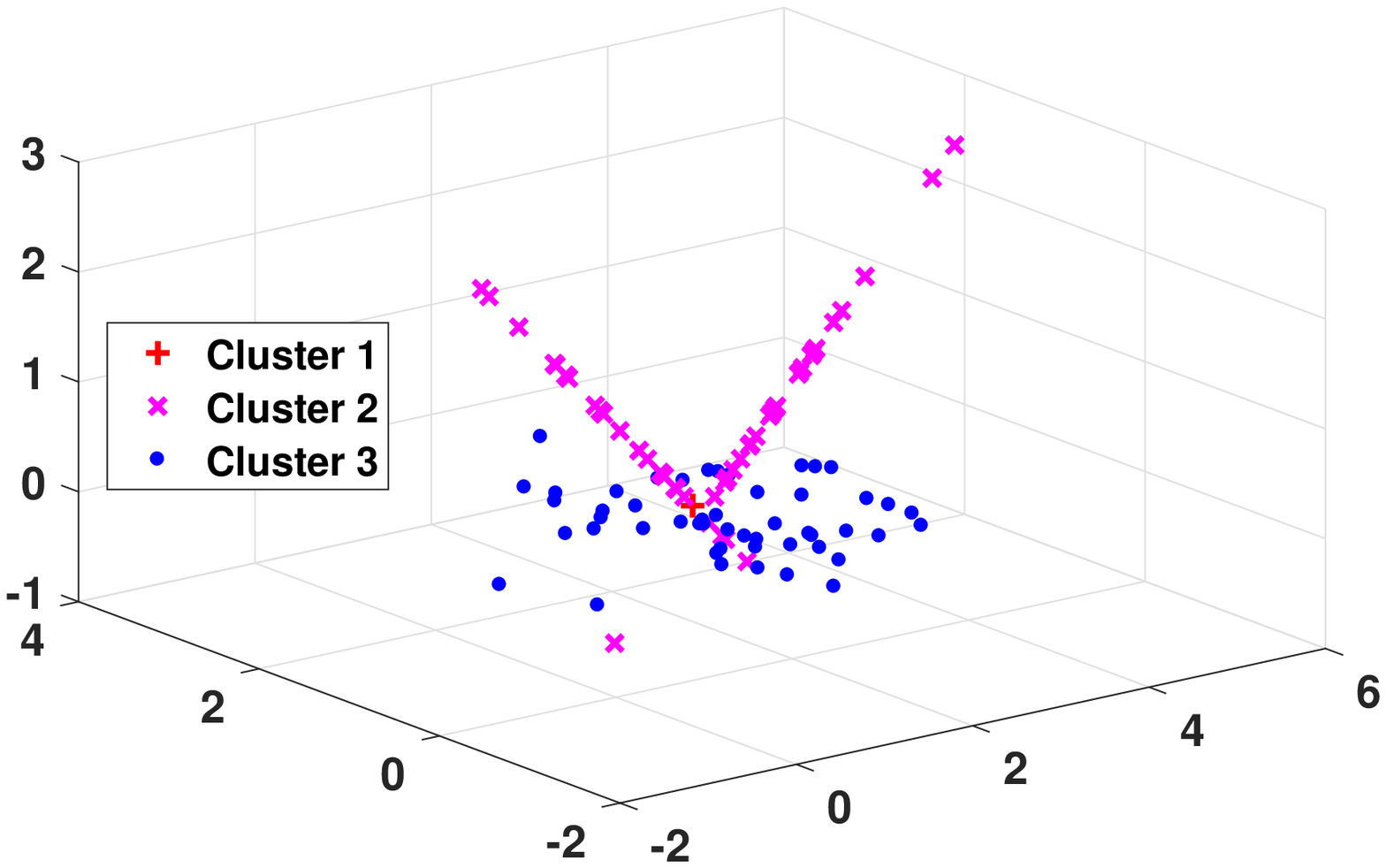}}
    \subfigure[PPC (Group G3)]{\includegraphics[width=0.168\textheight]{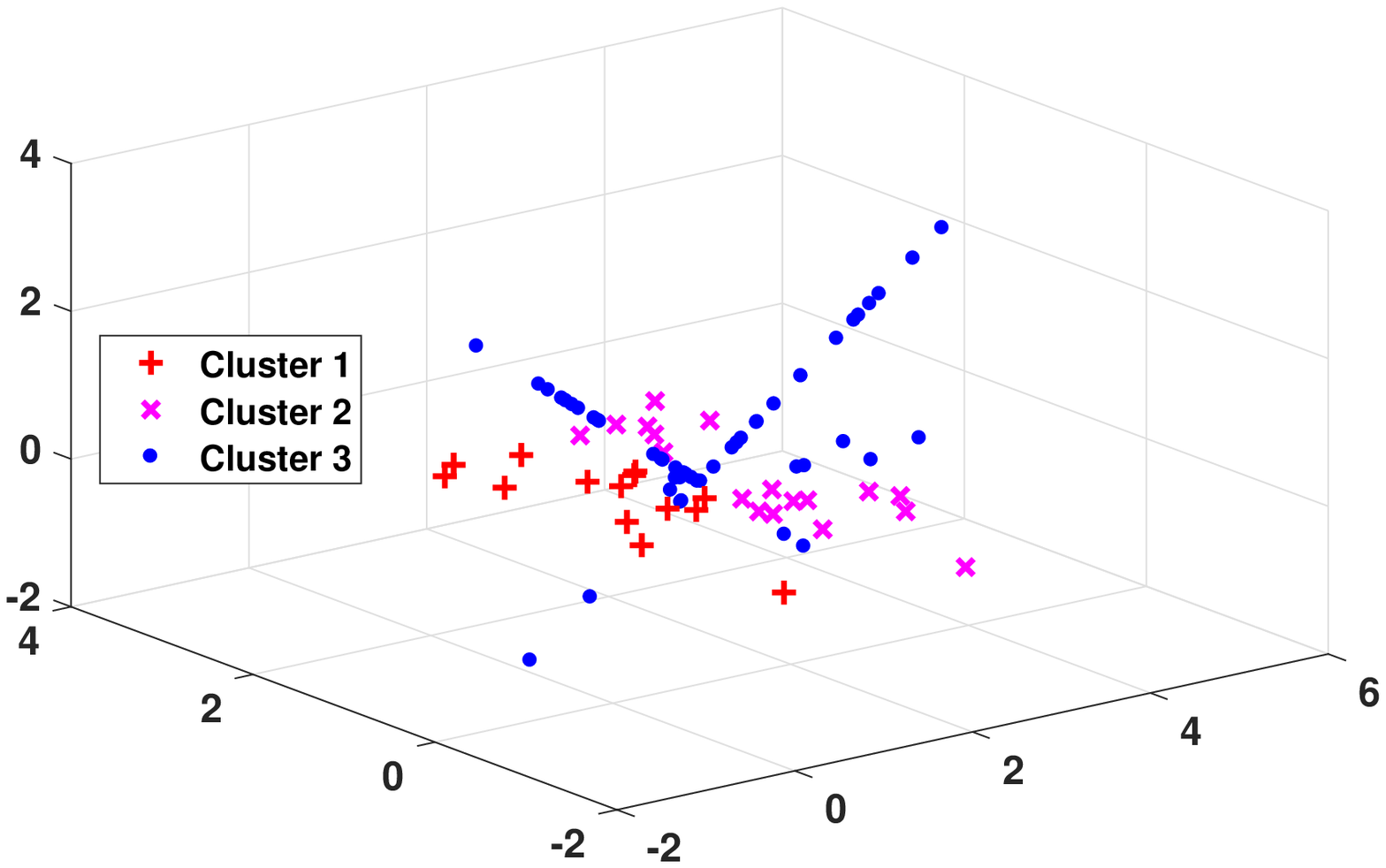}}
    \subfigure[PPC (Group G4)]{\includegraphics[width=0.168\textheight]{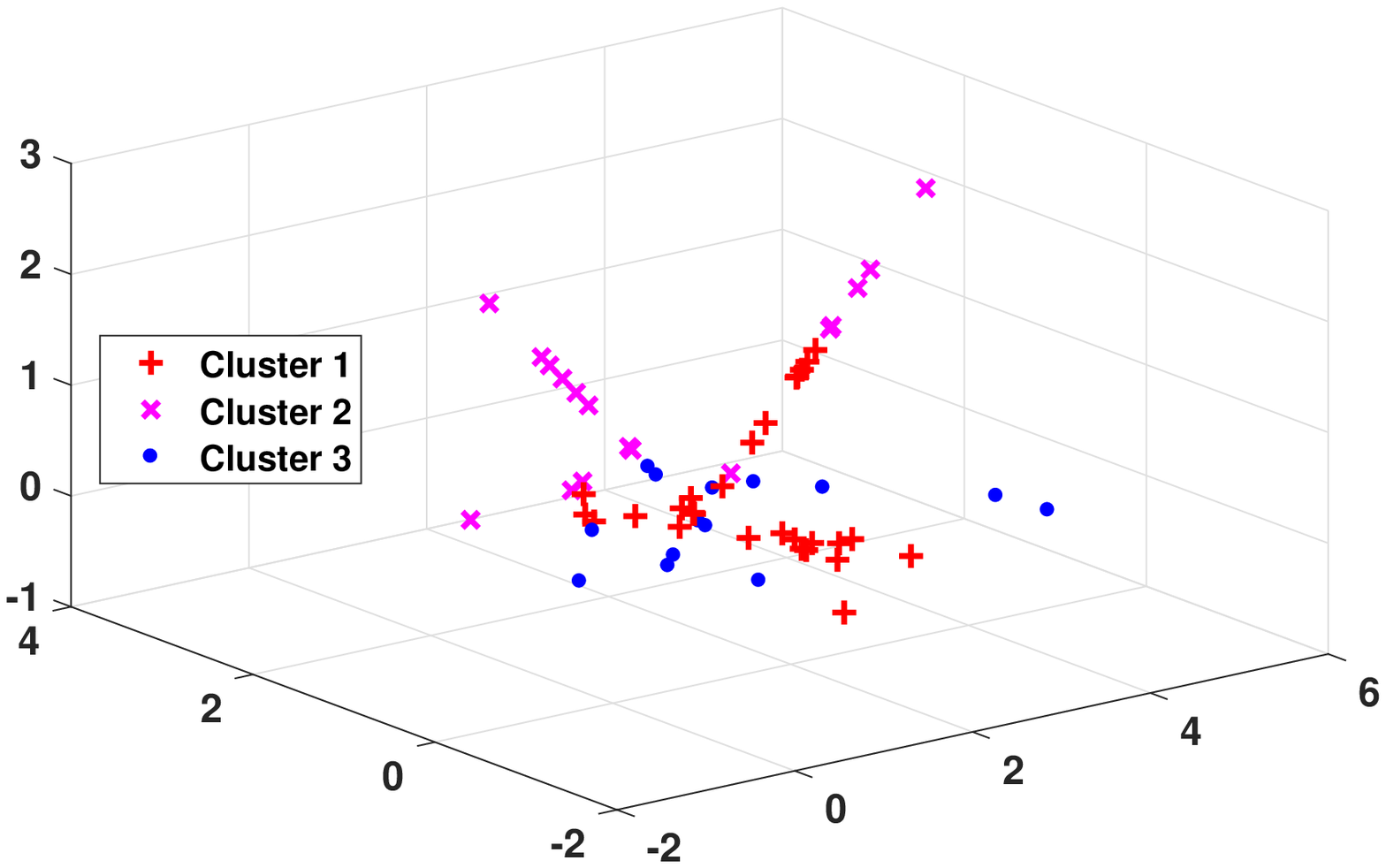}}
    \subfigure[TWSVC (Group G1)]{\includegraphics[width=0.168\textheight]{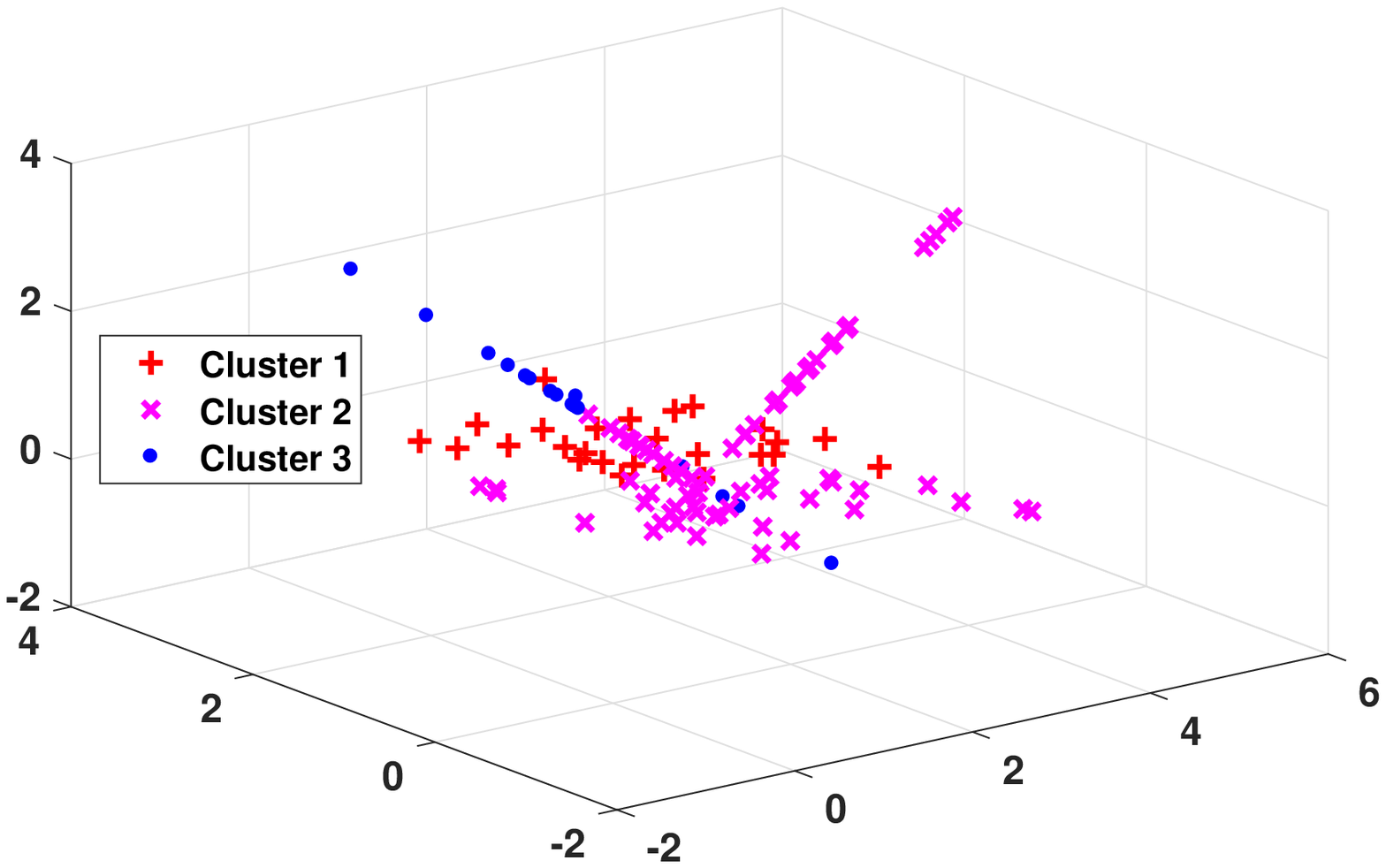}}
    \subfigure[TWSVC (Group G2)]{\includegraphics[width=0.168\textheight]{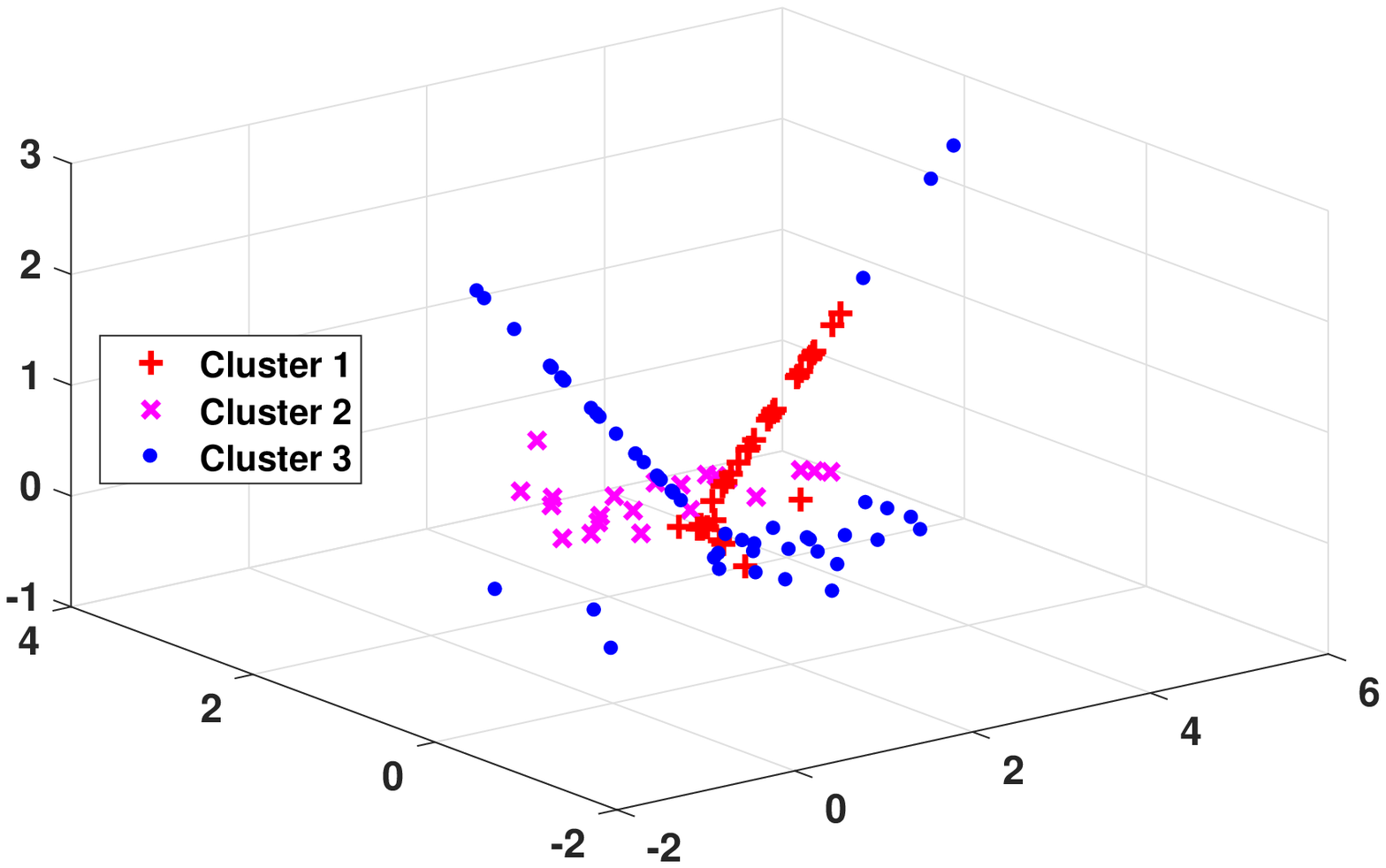}}
    \subfigure[TWSVC (Group G3)]{\includegraphics[width=0.168\textheight]{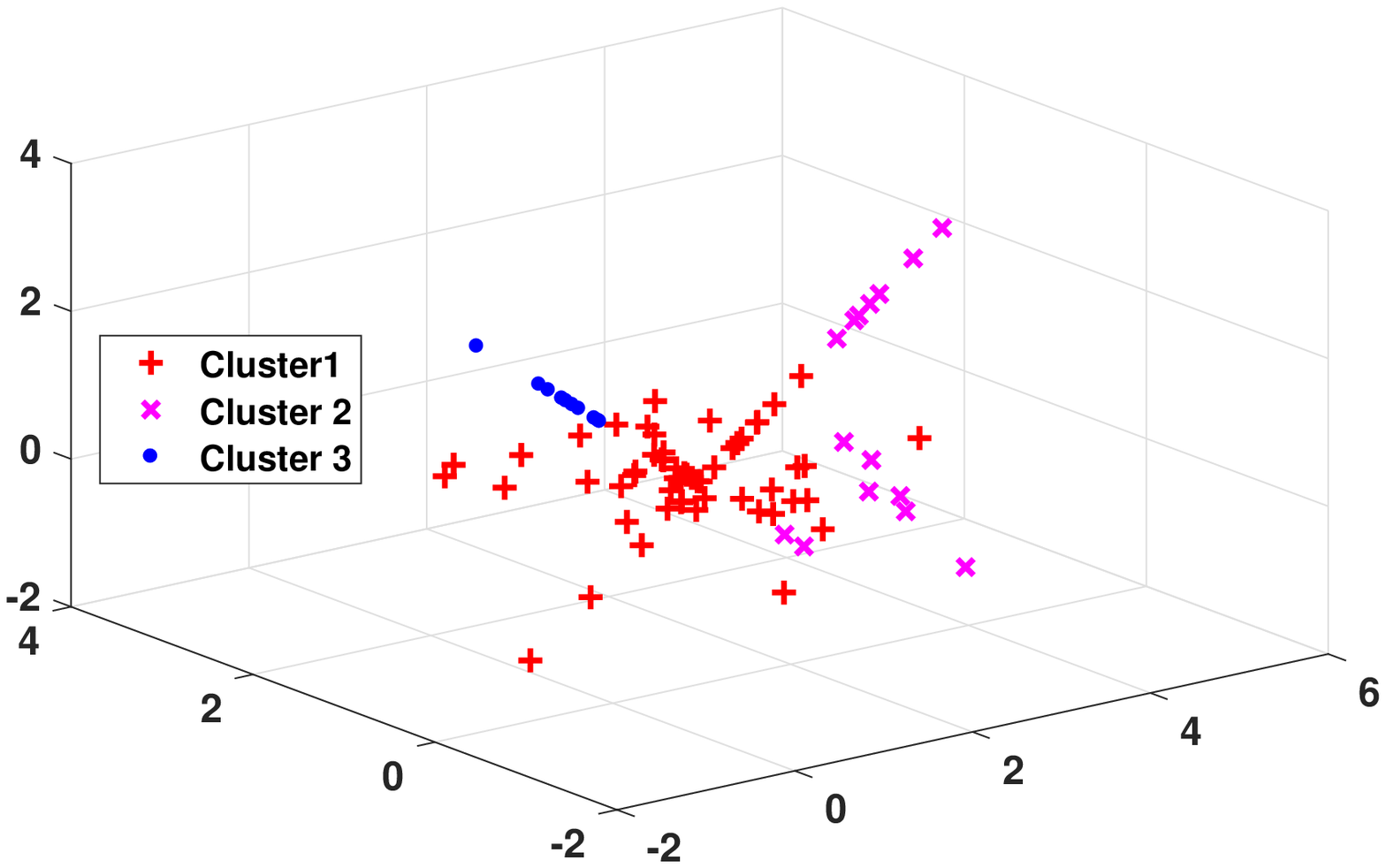}}
    \subfigure[TWSVC (Group G4)]{\includegraphics[width=0.168\textheight]{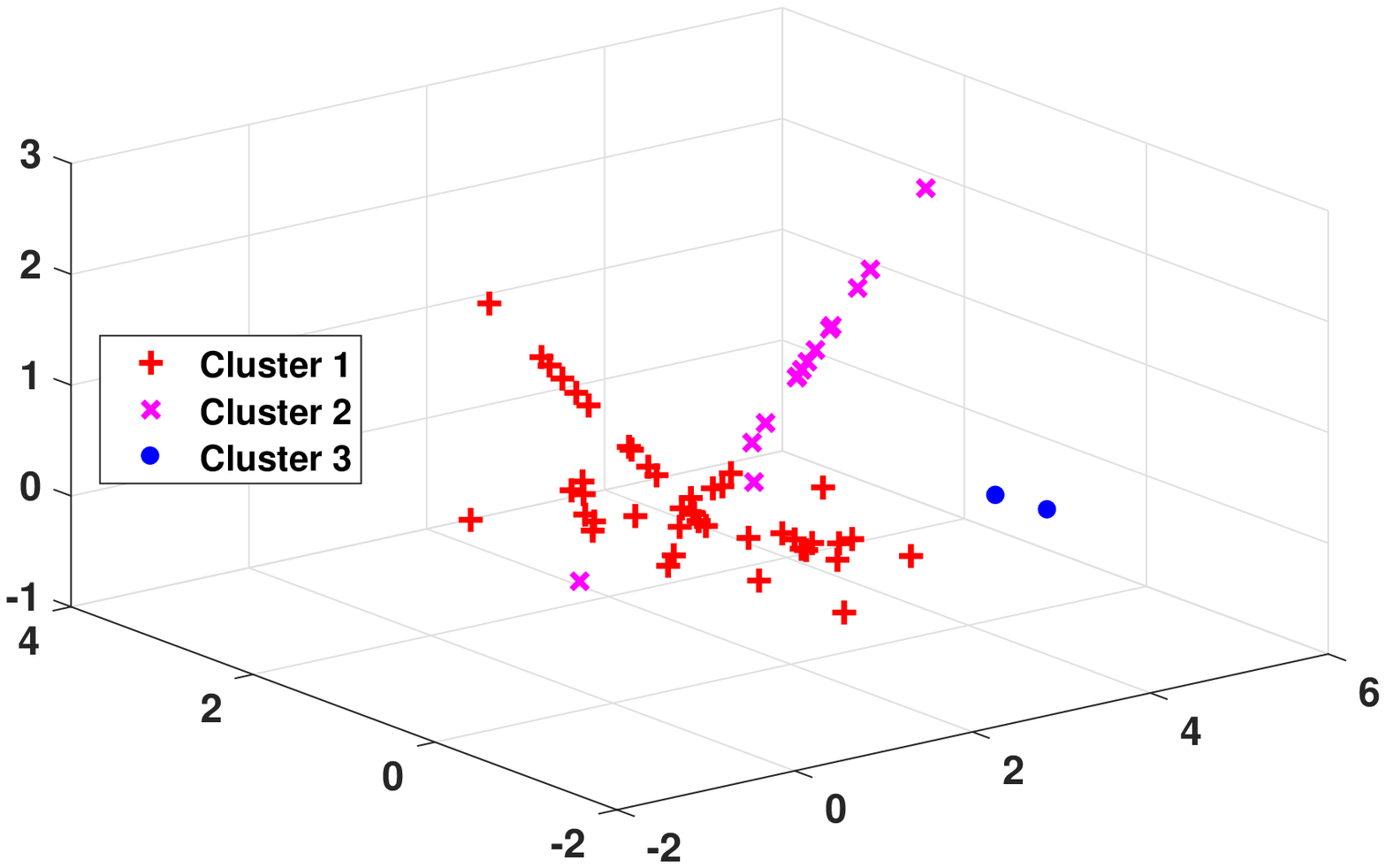}}
    \subfigure[RTWSVC (Group G1)]{\includegraphics[width=0.168\textheight]{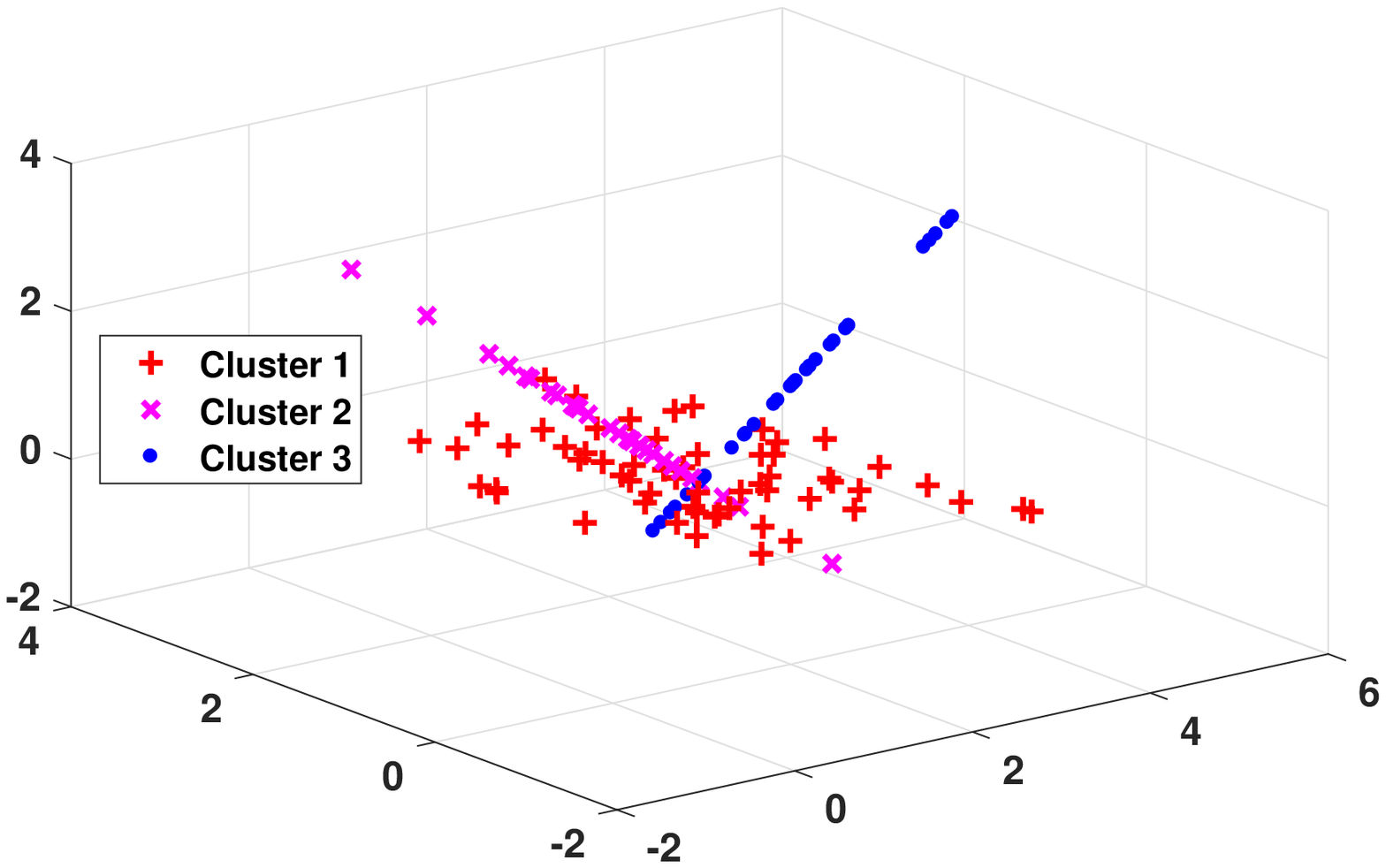}}
    \subfigure[RTWSVC (Group G2)]{\includegraphics[width=0.168\textheight]{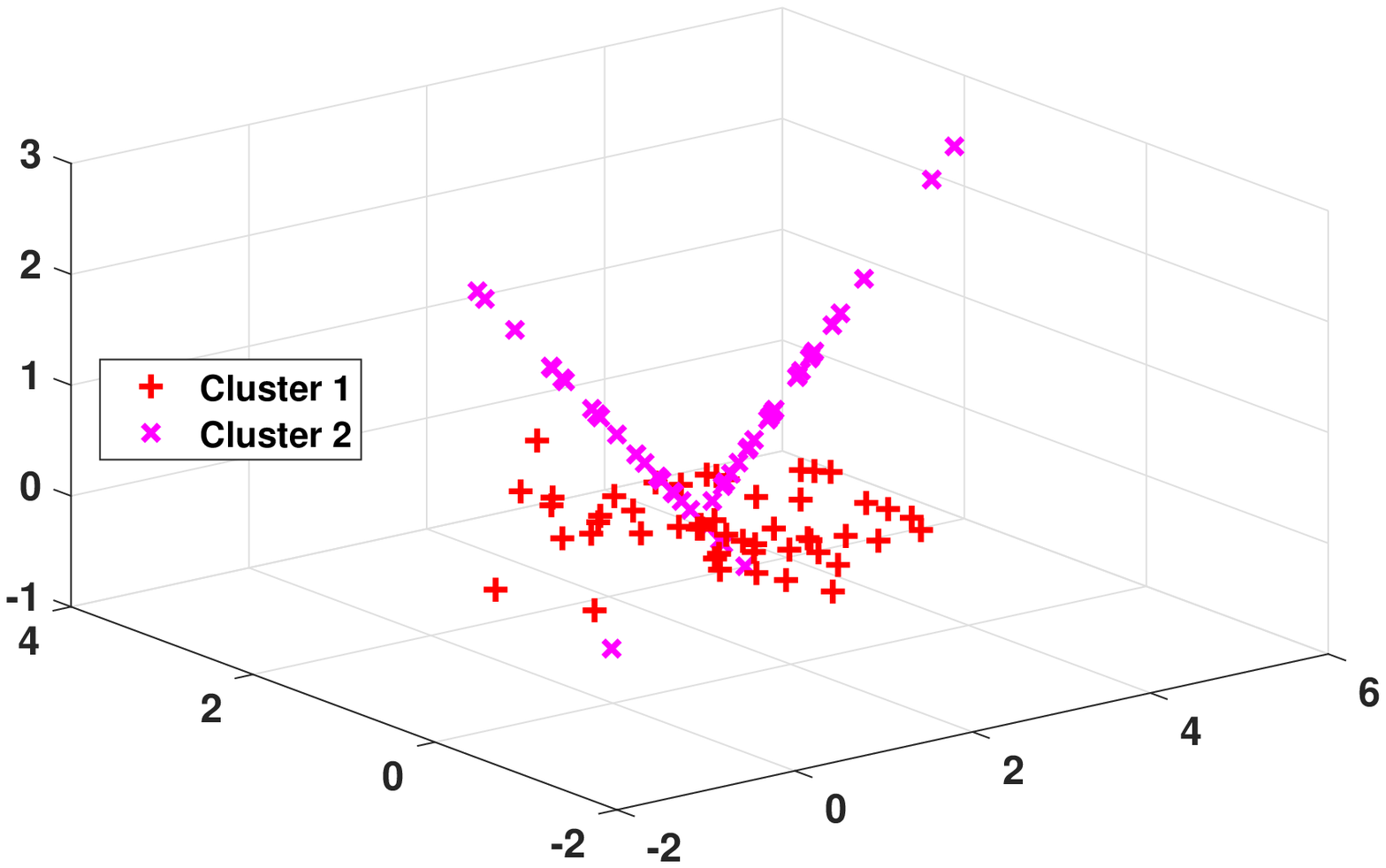}}
    \subfigure[RTWSVC (Group G3)]{\includegraphics[width=0.168\textheight]{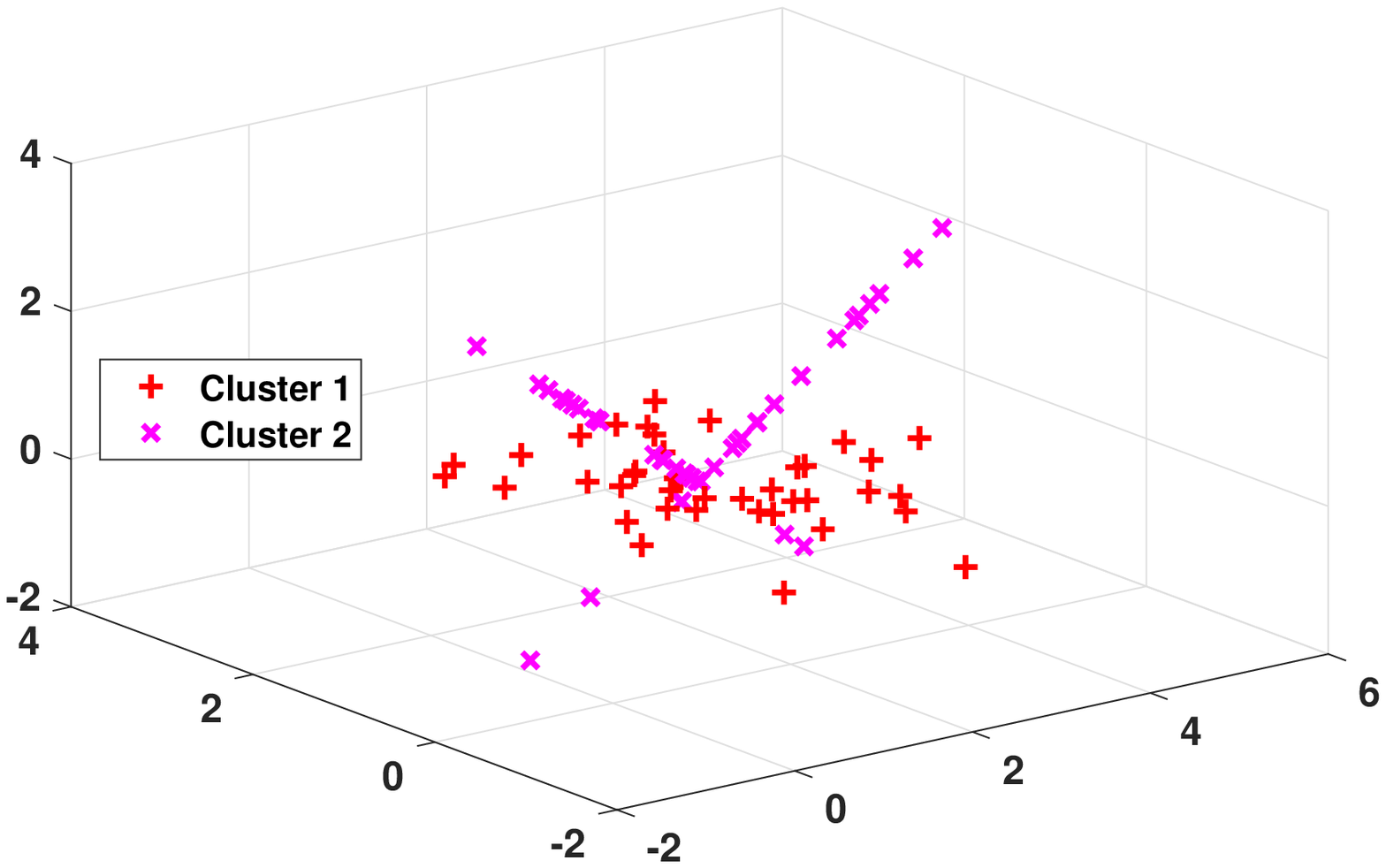}}
    \subfigure[RTWSVC (Group G4)]{\includegraphics[width=0.168\textheight]{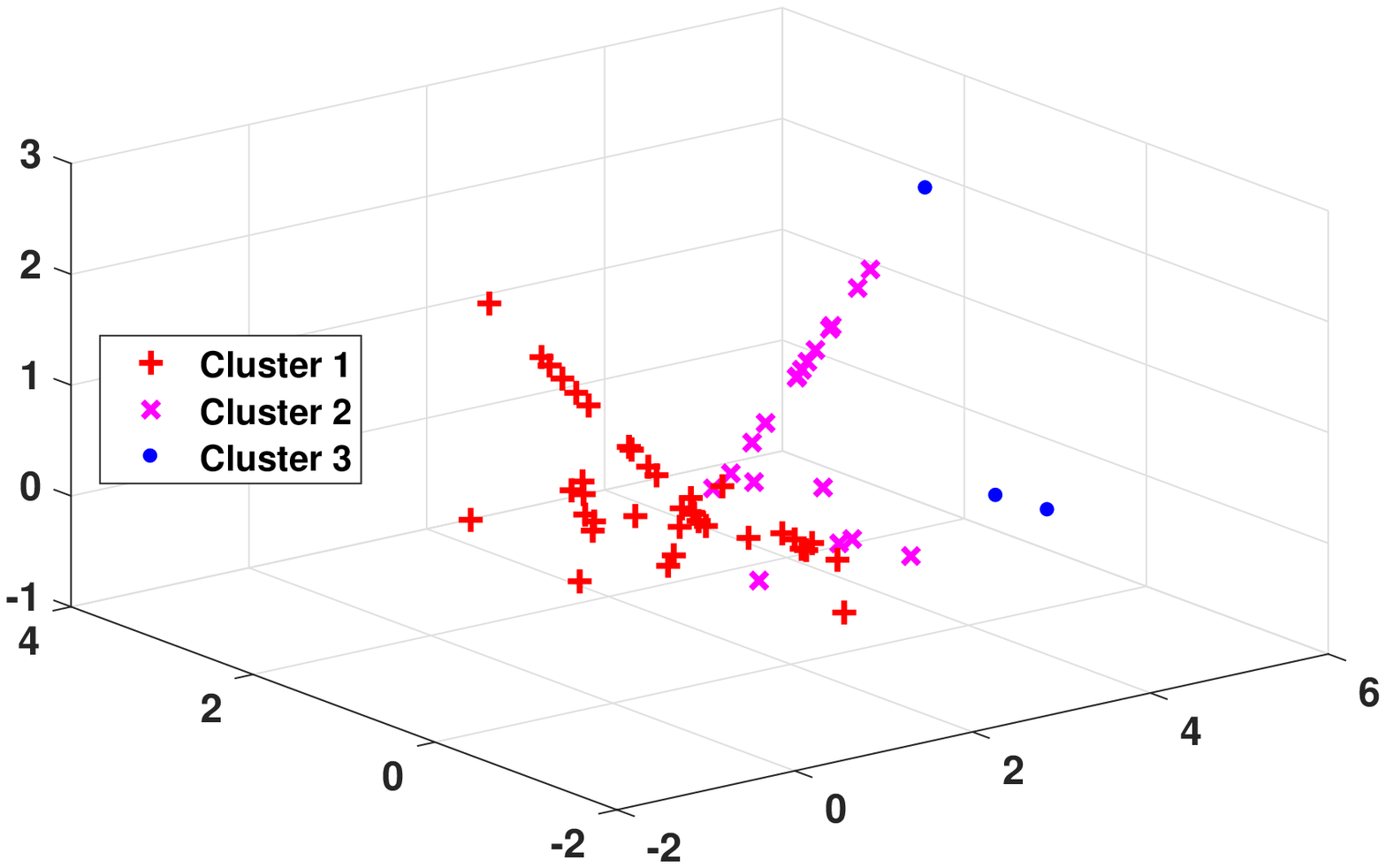}}
    \subfigure[FRTWSVC (Group G1)]{\includegraphics[width=0.168\textheight]{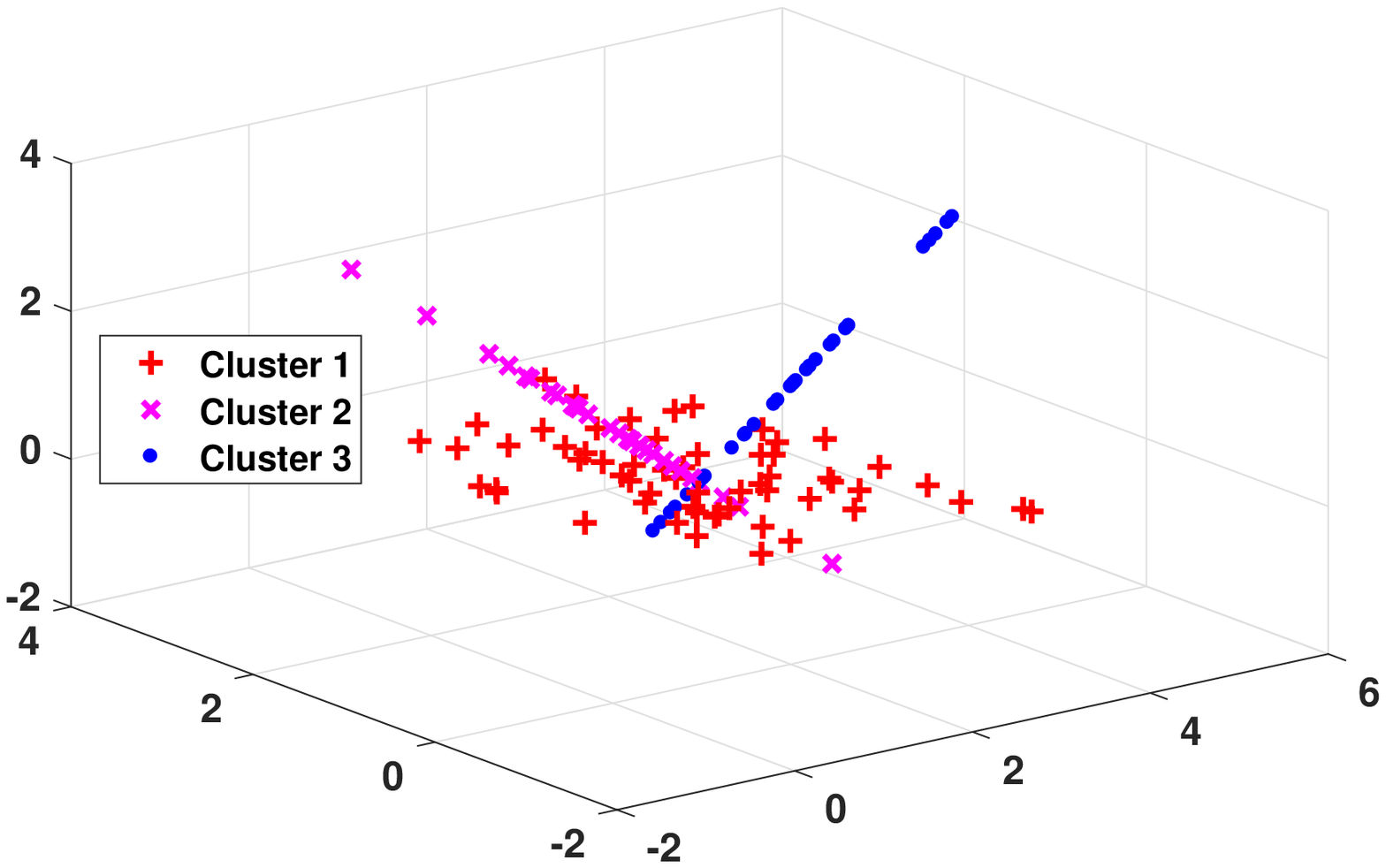}}
    \subfigure[FRTWSVC (Group G2)]{\includegraphics[width=0.168\textheight]{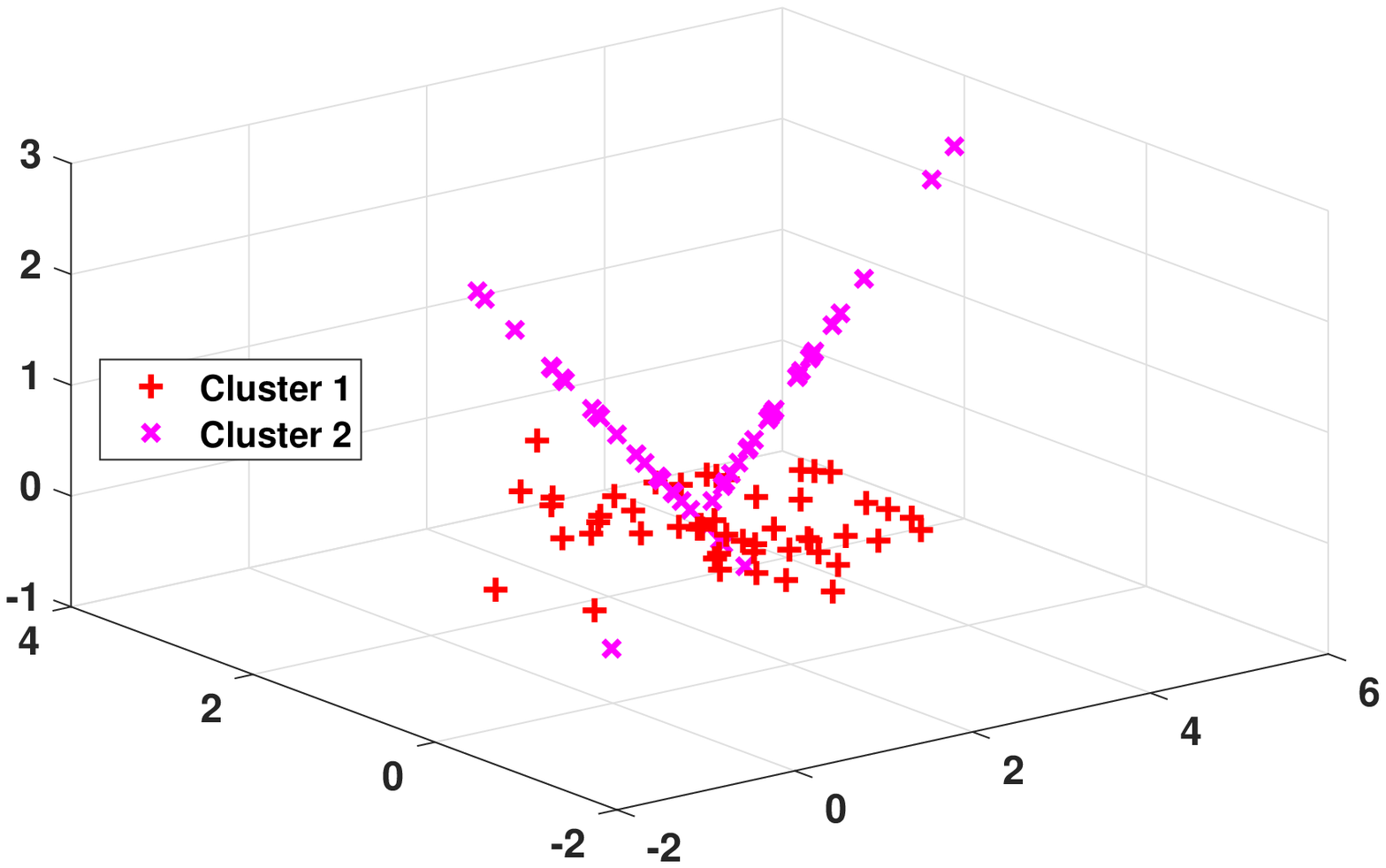}}
    \subfigure[FRTWSVC (Group G3)]{\includegraphics[width=0.168\textheight]{A80_RTWSVC.eps}}
    \subfigure[FRTWSVC (Group G4)]{\includegraphics[width=0.168\textheight]{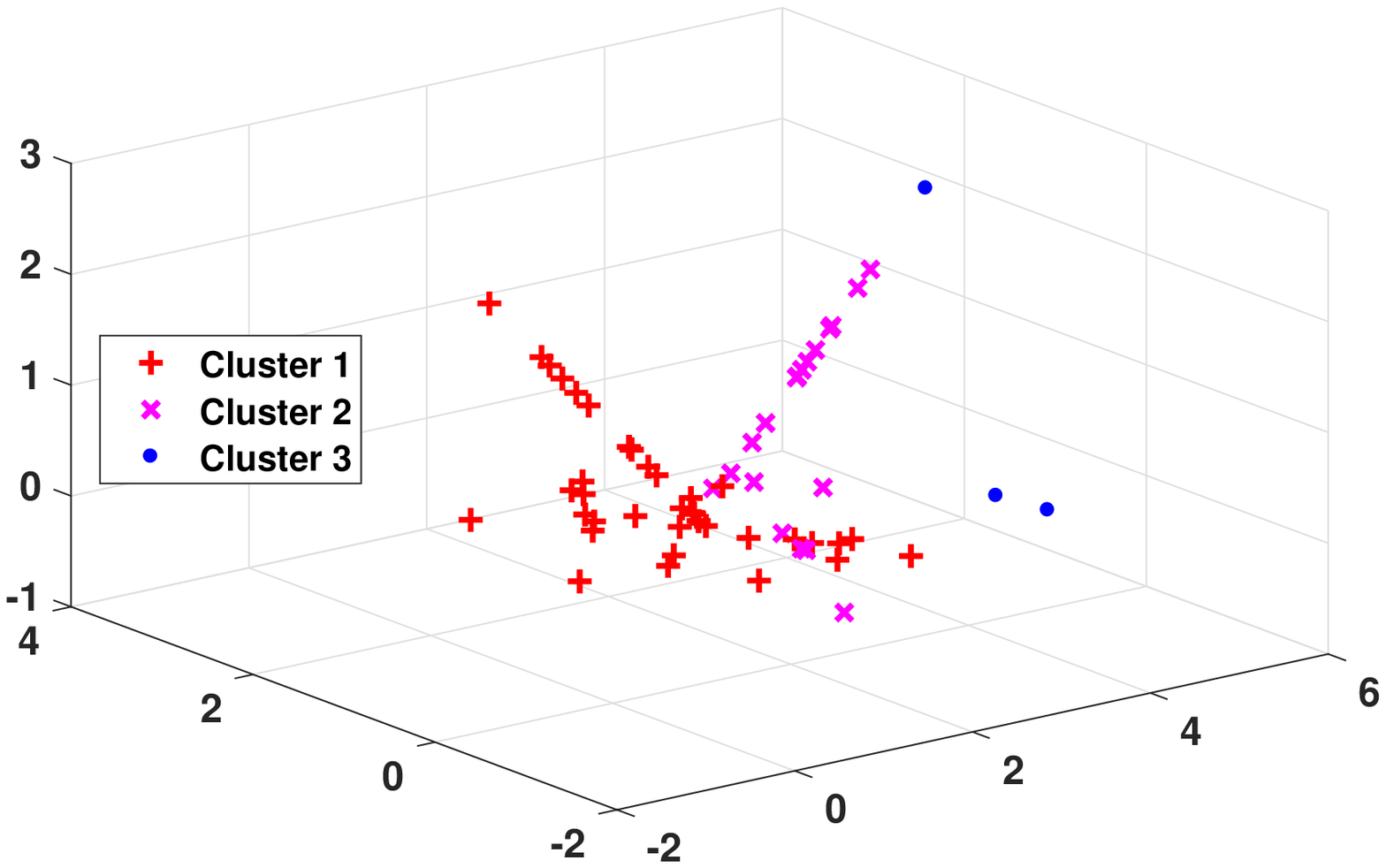}}
    \subfigure[RampTWSVC (Group G1)]{\includegraphics[width=0.168\textheight]{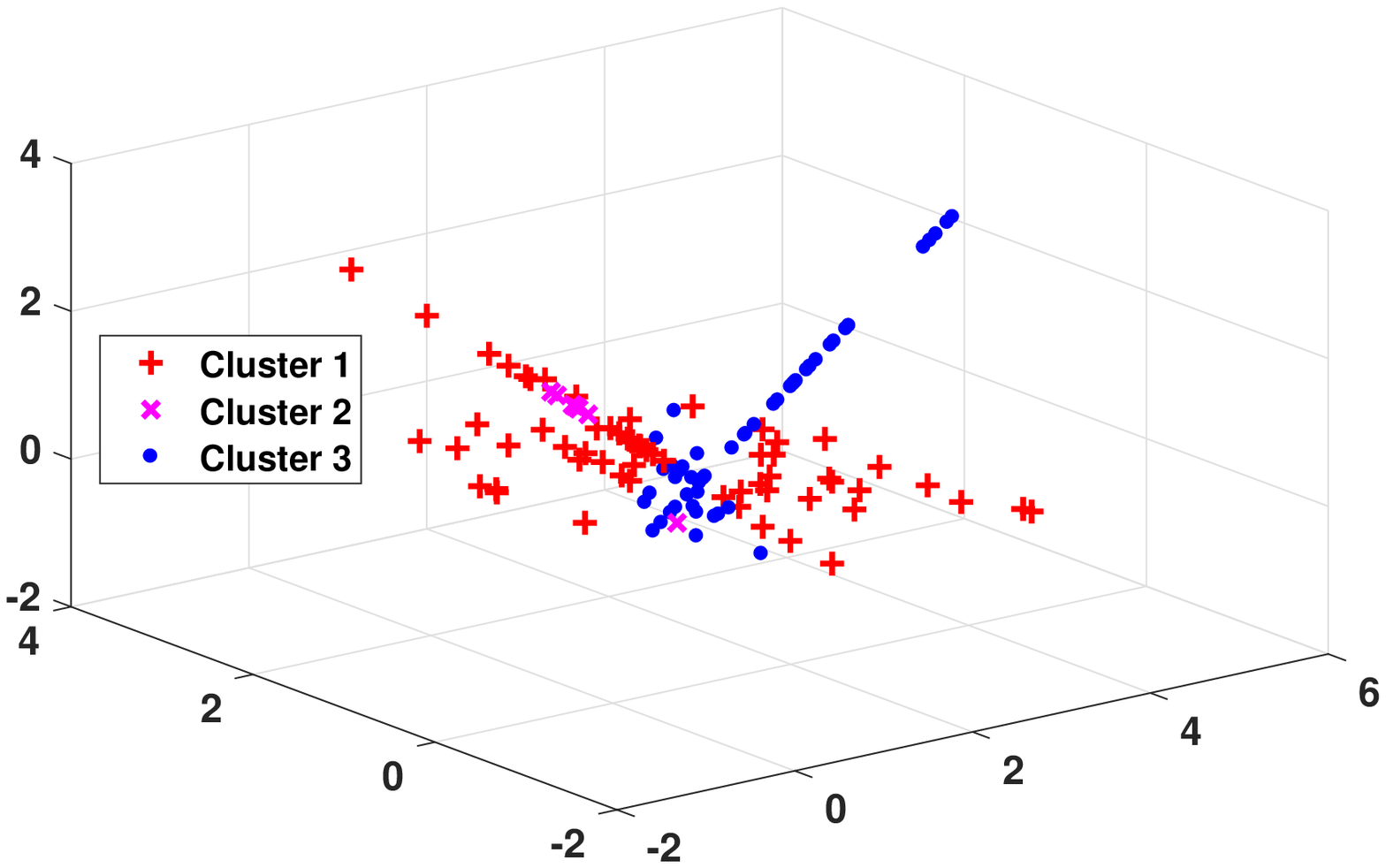}}
    \subfigure[RampTWSVC (Group G2)]{\includegraphics[width=0.168\textheight]{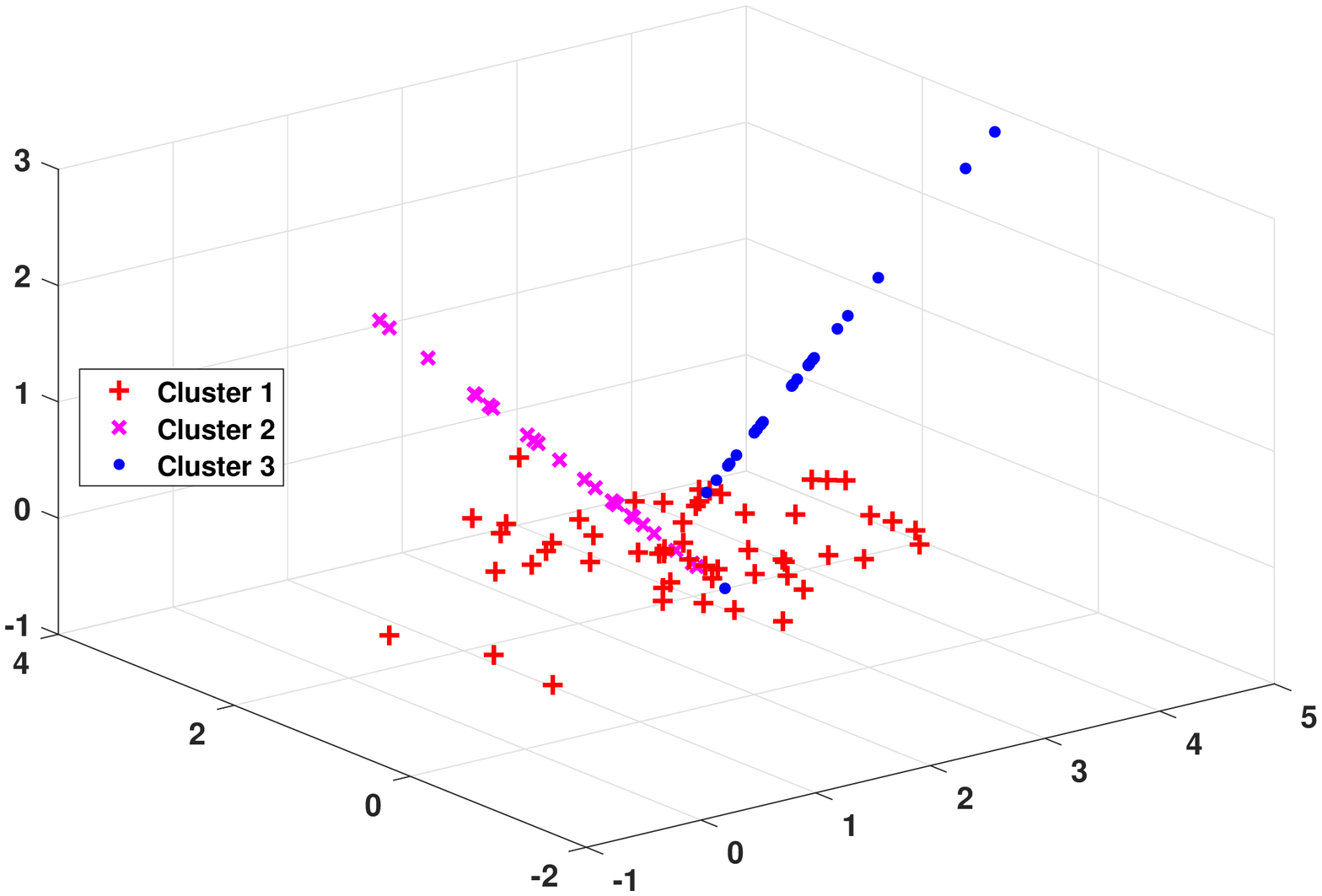}}
    \subfigure[RampTWSVC (Group G3)]{\includegraphics[width=0.168\textheight]{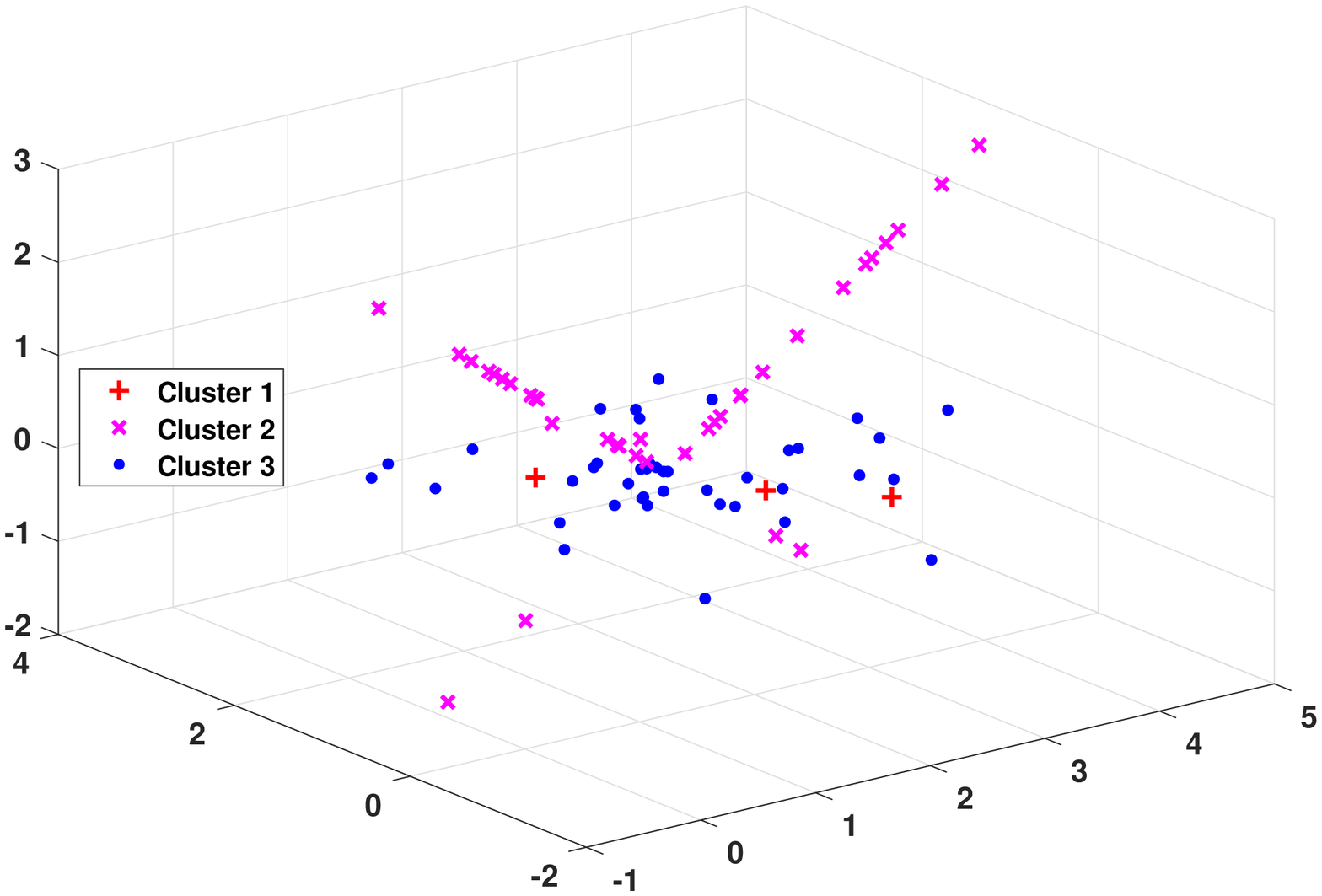}}
    \subfigure[RampTWSVC (Group G4)]{\includegraphics[width=0.168\textheight]{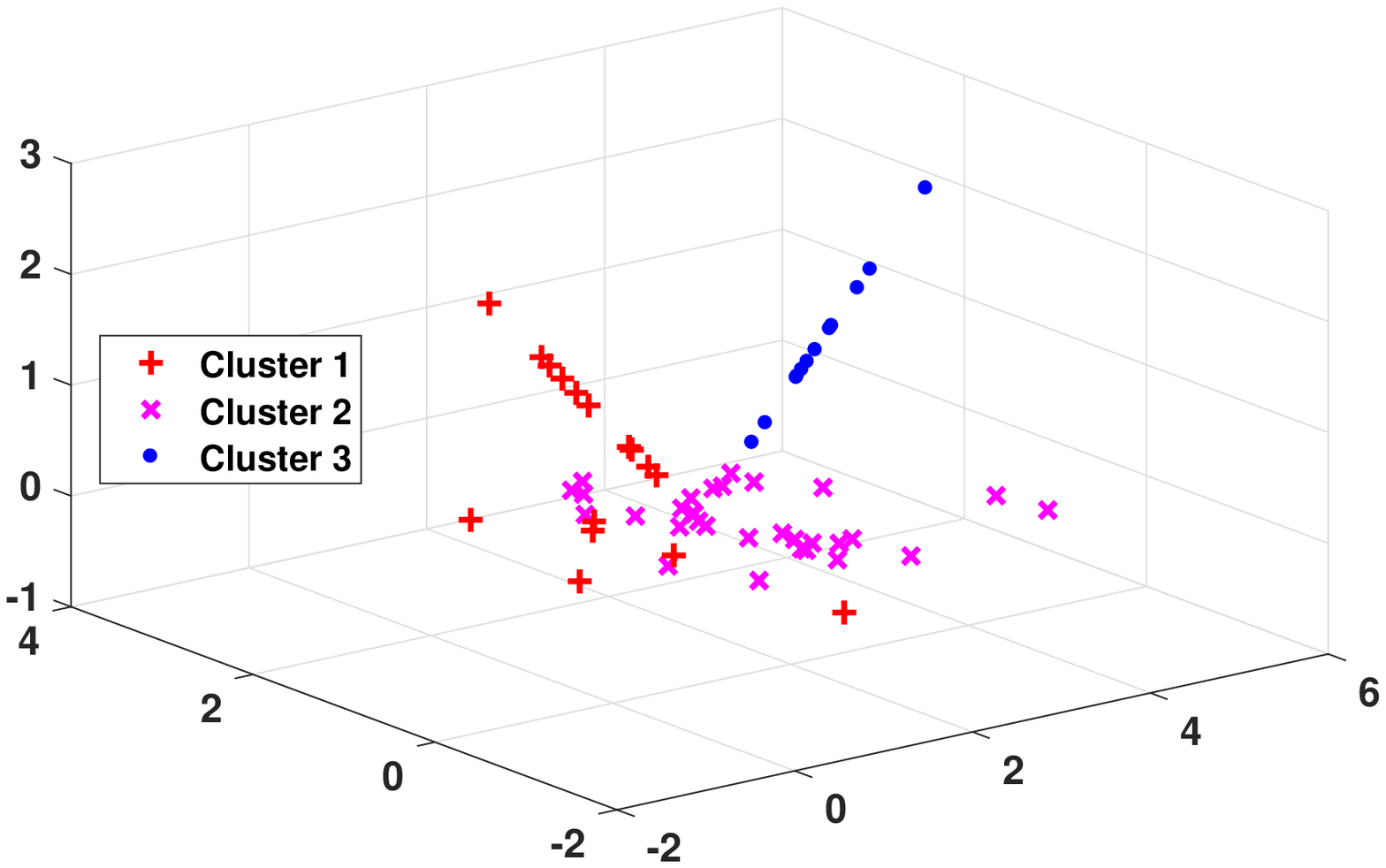}}
    \subfigure[RFDPC (Group G1)]{\includegraphics[width=0.168\textheight]{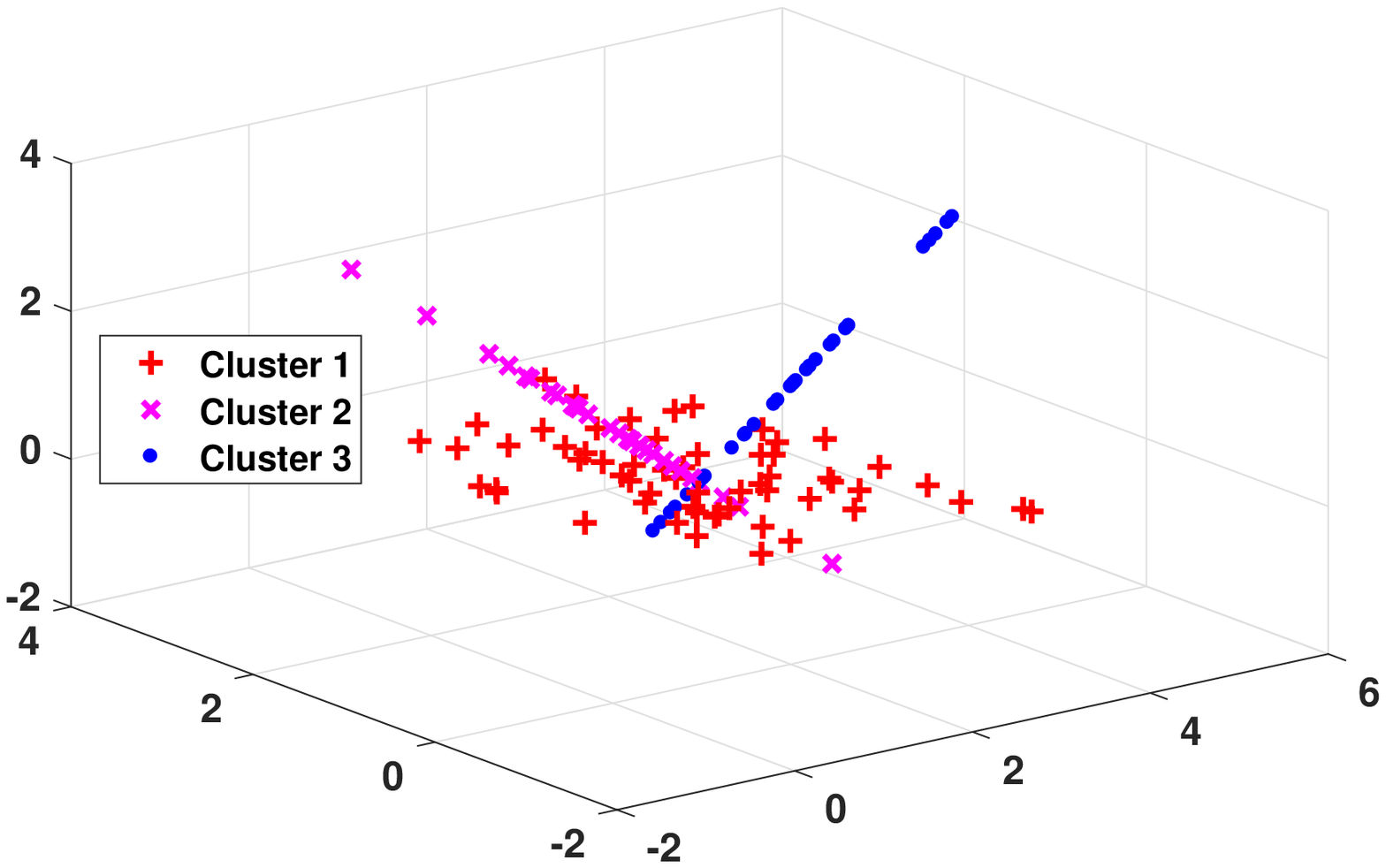}}
    \subfigure[RFDPC (Group G2)]{\includegraphics[width=0.168\textheight]{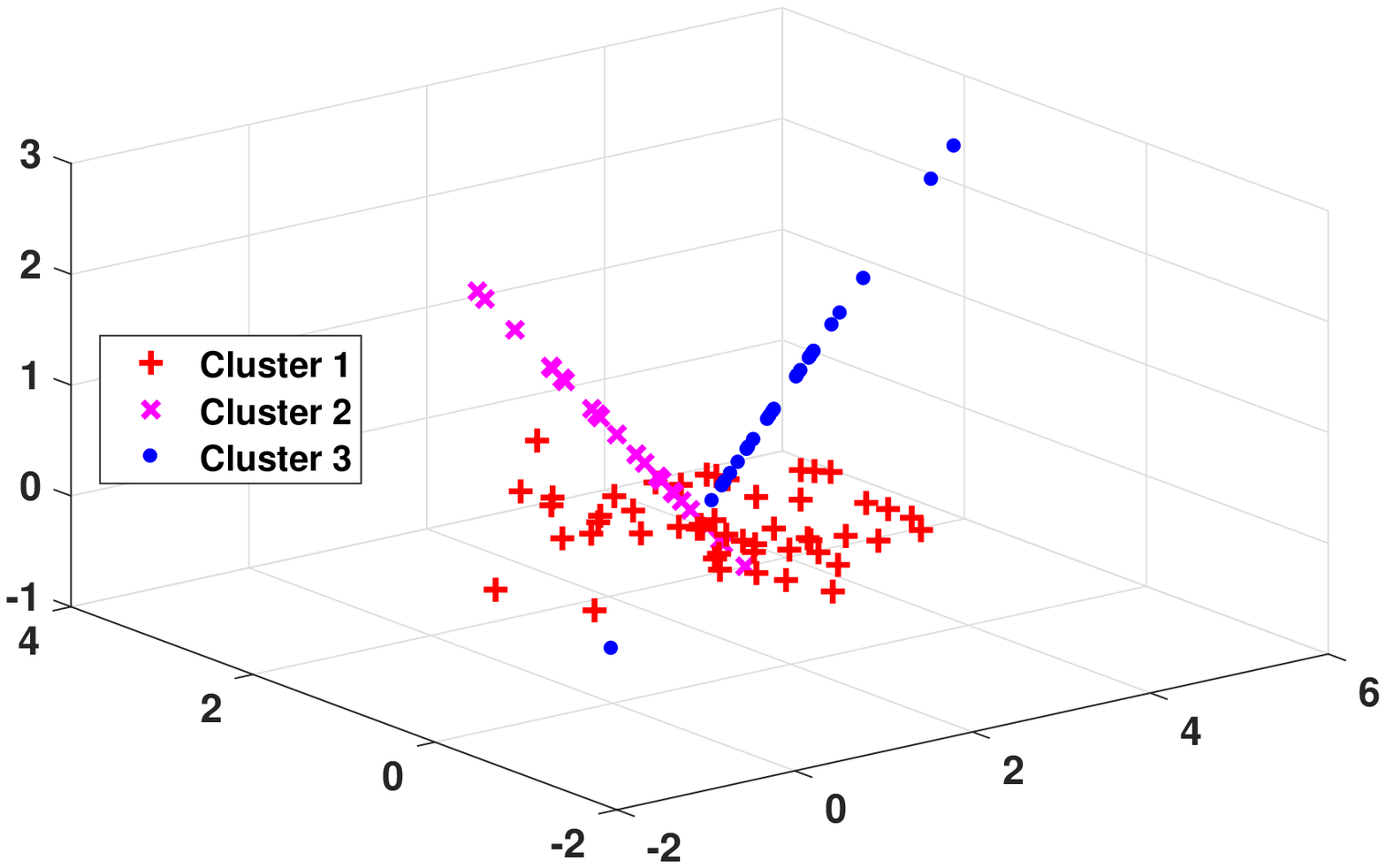}}
    \subfigure[RFDPC (Group G3)]{\includegraphics[width=0.168\textheight]{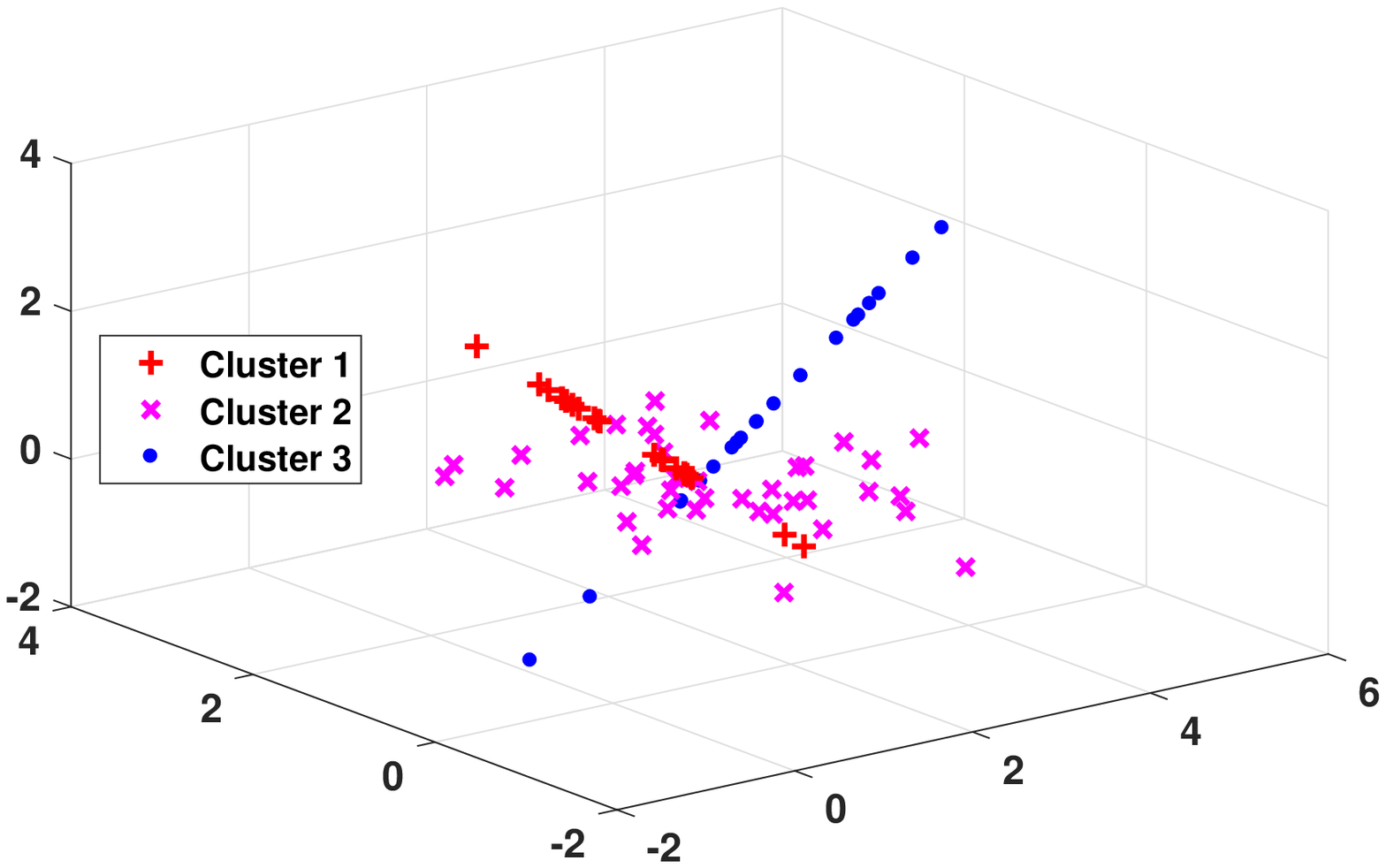}}
    \subfigure[RFDPC (Group G4)]{\includegraphics[width=0.168\textheight]{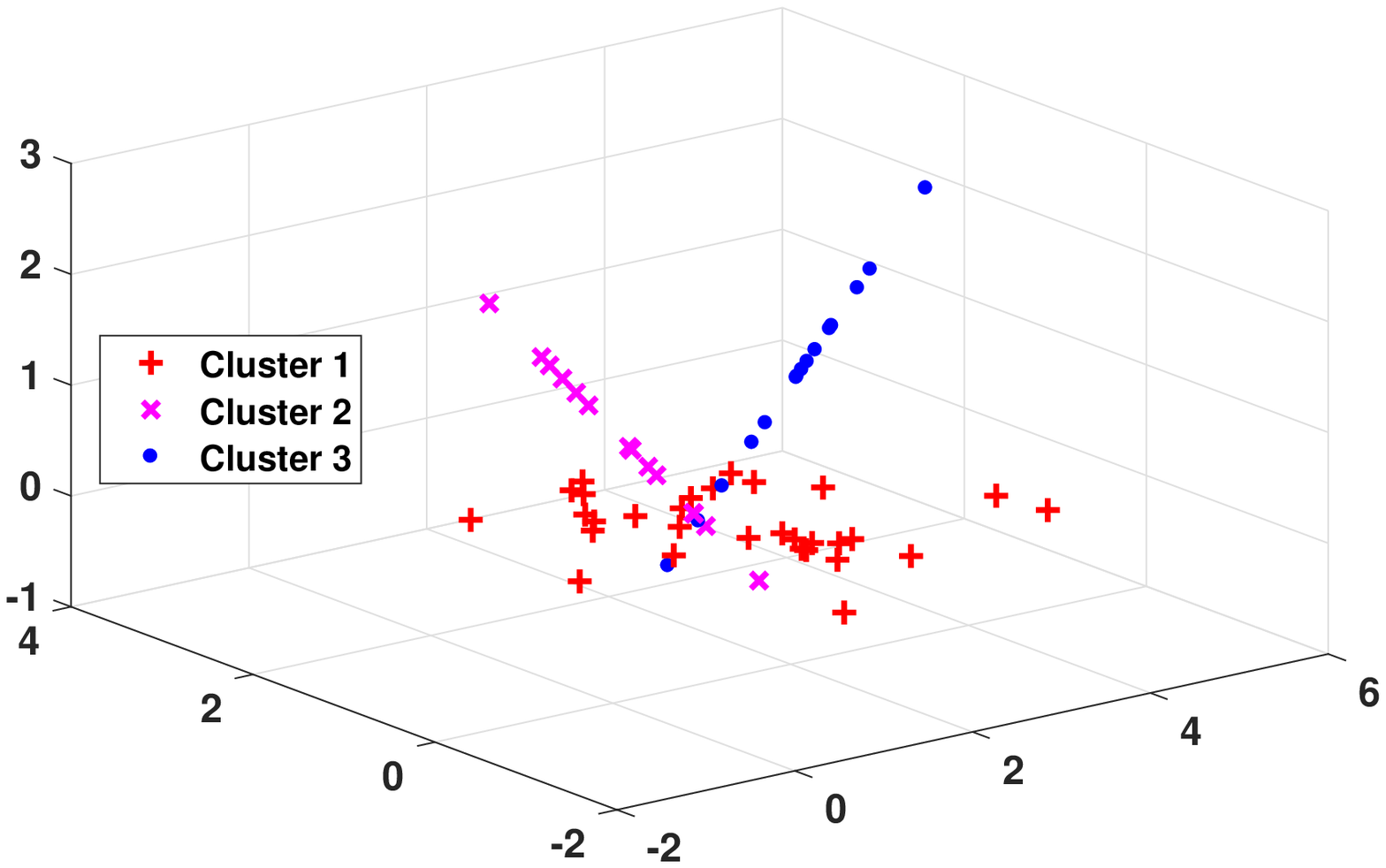}}
\caption{Plane-based clustering methods applied to four groups of the synthetic datasets, where groups G1, G2, G3 and G4 include 120, 100, 80 and 60 samples, respectively.}\label{Artshow}
\end{figure*}

\begin{figure*}
\centering
    \subfigure[Group G1]{\includegraphics[width=0.33\textheight]{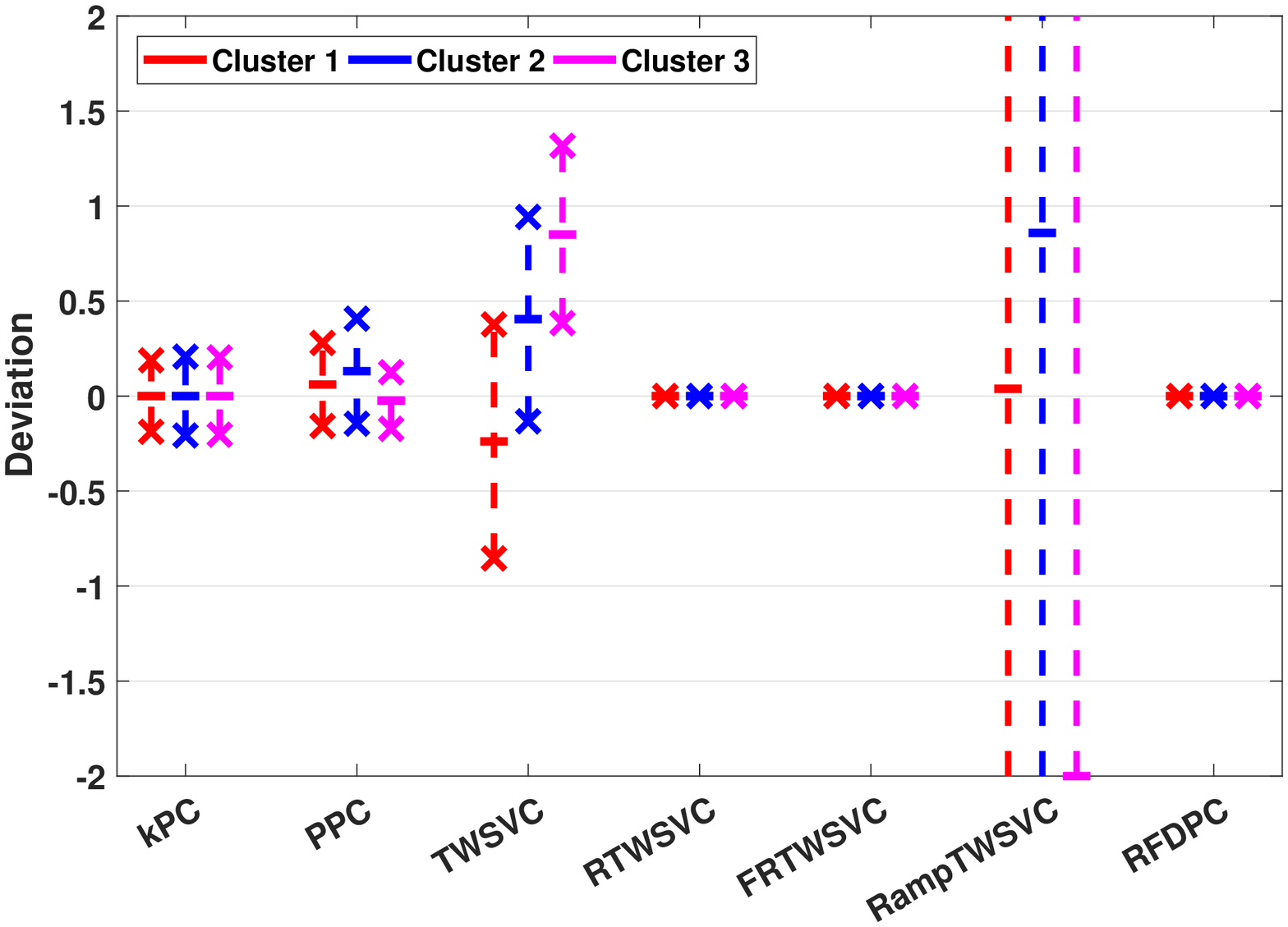}}
    \subfigure[Group G2]{\includegraphics[width=0.33\textheight]{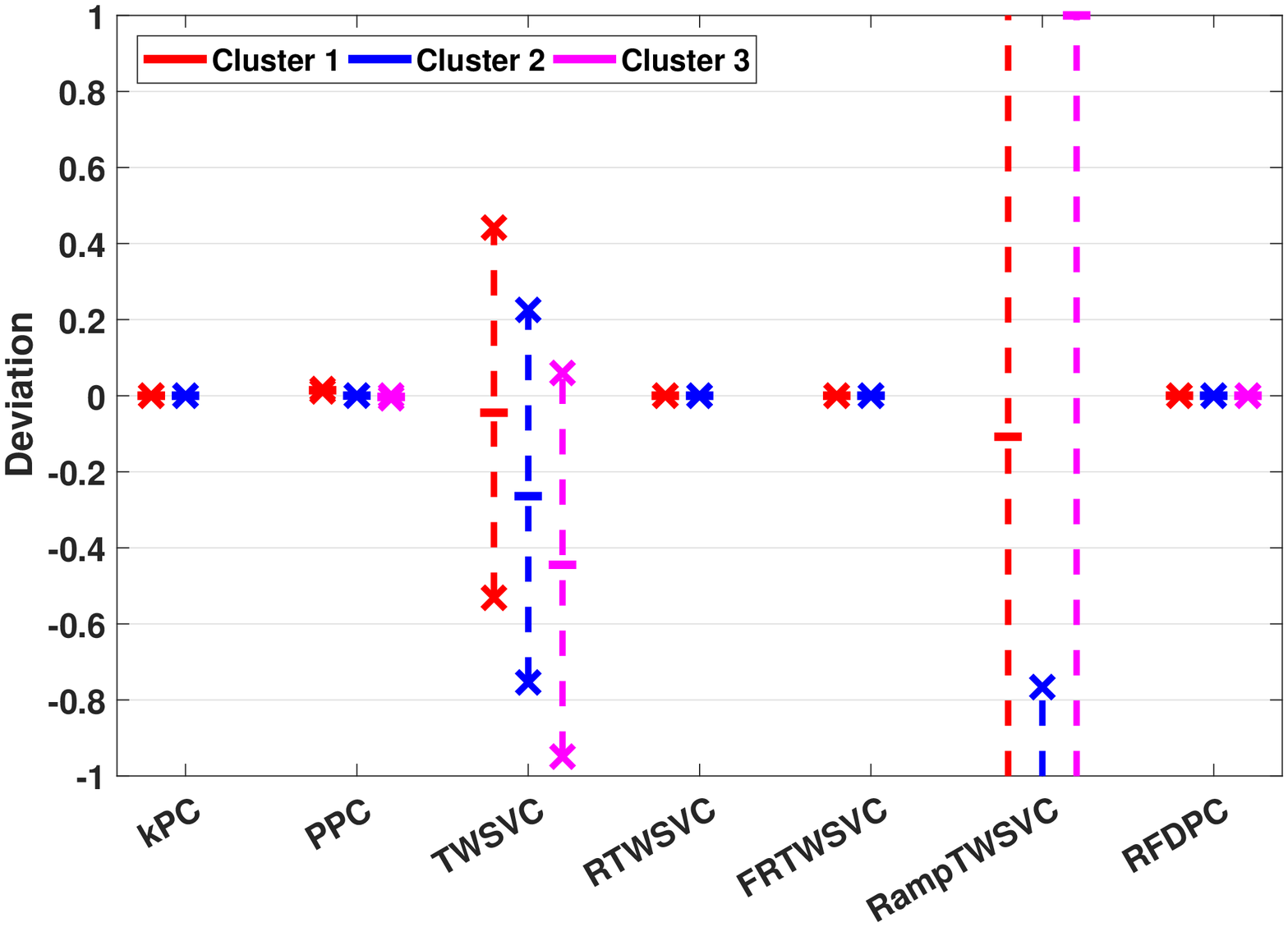}}
    \subfigure[Group G3]{\includegraphics[width=0.33\textheight]{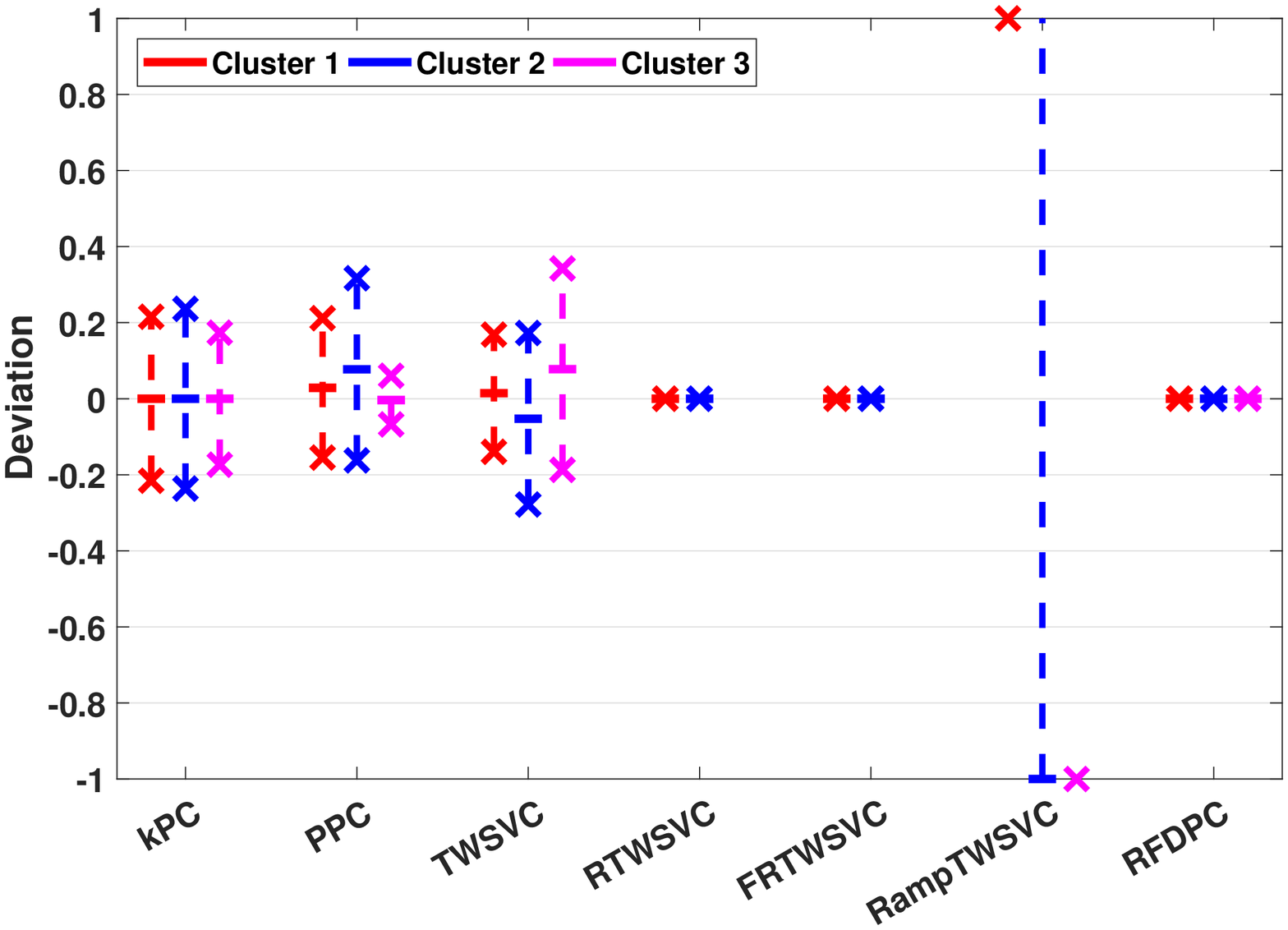}}
    \subfigure[Group G4]{\includegraphics[width=0.33\textheight]{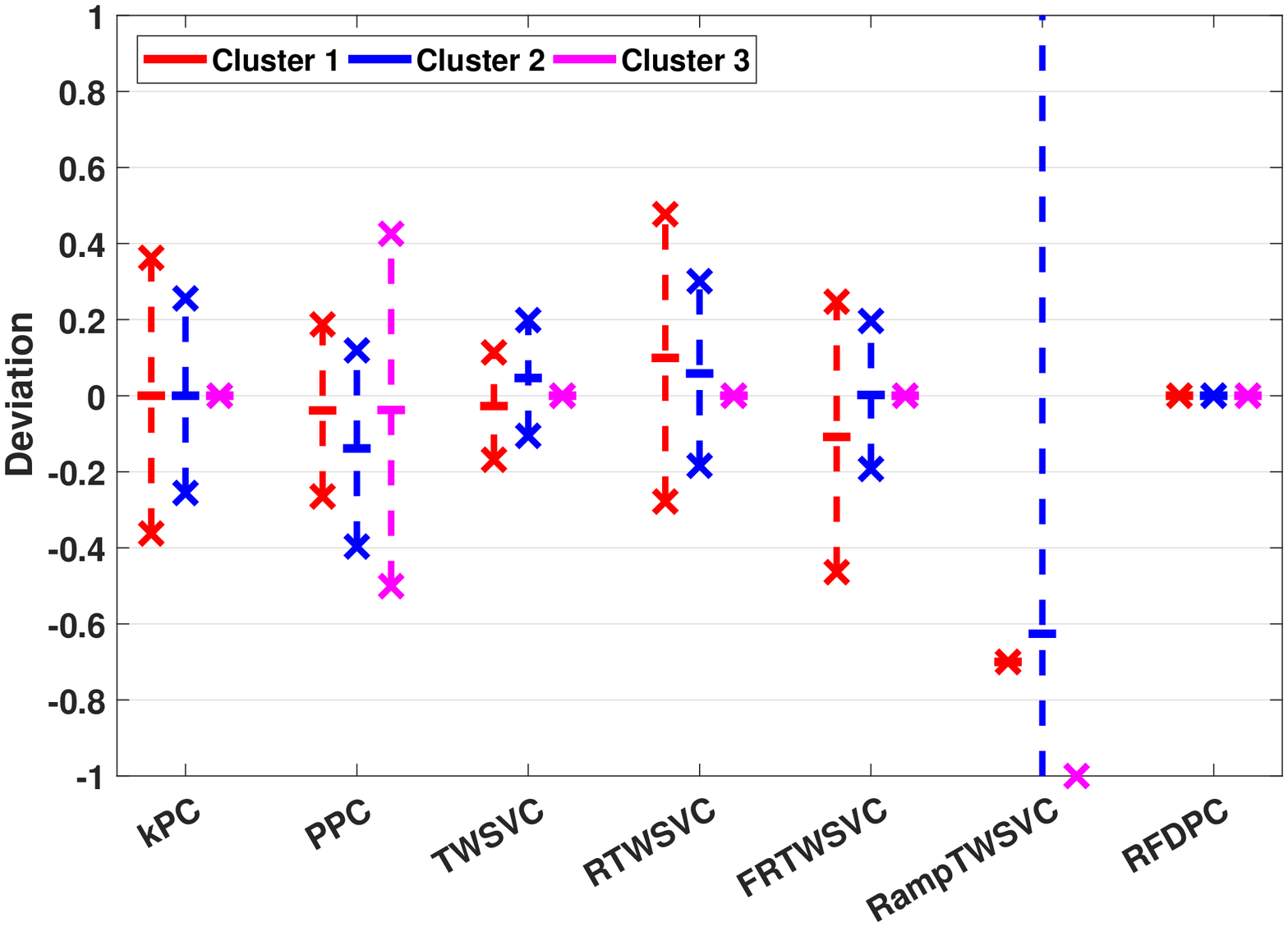}}
\caption{Deviation statistics of the samples from their cluster center planes on the synthetic datasets, where the clustering methods include kPC, PPC, TWSVC, RTWSVC, FRTWSVC, RampTWSVC and RFDPC, `-' denotes the 1-order statistics and `$\times$' denotes the 2-order statistics. If the symbol `$\times$' for a method is on the lower bound or upper bound, it means the deviations of this cluster are out of the region. If there is not any symbol for a cluster, it means there is not such a cluster for this method.}\label{ArtDev}
\end{figure*}

\begin{table*}
\caption{AC and MI of the plane-based clustering methods on the synthetic datasets} \centering
\begin{tabular}{llllllll}
\hline\noalign{\smallskip}
 Group&\scriptsize{kPC}  &\scriptsize{PPC}  &\scriptsize{TWSVC}&\scriptsize{RTWSVC}&\scriptsize{FRTWSVC}&\scriptsize{RampTWSVC} &\scriptsize{RFDPC} \\
&AC(\%)/MI(\%)&AC(\%)/MI(\%)&AC(\%)/MI(\%)&AC(\%)/MI(\%)&AC(\%)/MI(\%)&AC(\%)/MI(\%)&AC(\%)/MI(\%)\\
\noalign{\smallskip}\hline\noalign{\smallskip}
G1&60.74/27.18&62.20/29.63&56.57/27.55&$\mathbf{100.0/100.0}$&$\mathbf{100.0/100.0}$&72.82/33.97&$\mathbf{100.0/100.0}$\\
G2&87.37/67.10&87.39/67.34&62.79/28.42&87.37/67.10&87.37/67.10&84.38/56.71&$\mathbf{100.0/100.0}$\\
G3&56.42/13.13&61.14/40.67&58.96/21.39&87.34/67.21&87.34/67.21&65.09/23.83&$\mathbf{98.13/94.35}$\\
G4&58.87/24.09&58.81/11.19&61.53/30.08&57.06/16.79&57.18/21.07&73.45/48.10&$\mathbf{100.0/100.0}$\\
\hline \noalign{\smallskip}
\end{tabular} \label{ArtAC}\\
\end{table*}

\begin{table*}
\caption{AC and MI of the clustering methods on benchmark datasets for linear case} \centering
\begin{tabular}{lllllllll}
\hline\noalign{\smallskip}
Data &kmeans&kPC  &PPC  &TWSVC&RTWSVC&FRTWSVC&RampTWSVC &RFDPC \\
k:m$\times$n& AC(\%)& AC(\%)& AC(\%)& AC(\%)& AC(\%)& AC(\%)& AC(\%)& AC(\%)\\
&MI(\%)&MI(\%)&MI(\%)&MI(\%)&MI(\%)&MI(\%)&MI(\%)&MI(\%)\\
\noalign{\smallskip}\hline\noalign{\smallskip}
Compound& 86.29$\pm$4.05& 72.04& 73.90& 75.97& 80.44& 82.52& 77.07&$\mathbf{ 89.01}$\\
6:399$\times$2&$\mathbf{70.44}$$\pm$7.14& 34.71& 38.85& 50.44& 56.63& 48.18& 48.37& 67.38\\\hline
Dermatology& 69.76$\pm$0.77& 60.50& 70.36& 71.93& 60.50& 60.50& 72.67&$\mathbf{ 93.47}$\\
6:366$\times$34& 11.47$\pm$2.15& 29.65&  3.48& 10.17& 28.95& 28.59& 24.42&$\mathbf{ 76.48}$\\\hline
Ecoli& 82.19$\pm$2.68& 33.11& 66.46& 85.74& 34.33& 34.33& 79.42&$\mathbf{ 91.03}$\\
8:336$\times$7& 56.84$\pm$4.42&  8.61&  9.65& 58.45& 10.42& 10.42& 43.35&$\mathbf{ 70.37}$\\\hline
Glass& 65.58$\pm$3.22& 55.73&$\mathbf{ 66.75}$& 66.62& 57.59& 57.40& 62.77& 62.37\\
6:214$\times$9&$\mathbf{ 35.76}$$\pm$2.23& 22.55&  8.54& 17.83& 17.69& 18.20& 20.95& 12.61\\\hline
Iris& 84.57$\pm$6.86& 67.54& 60.95& 91.24& 92.67& 94.95& 86.79&$\mathbf{ 98.25}$\\
3:150$\times$4& 70.47$\pm$9.10& 25.41& 12.04& 82.53& 82.31& 86.97& 71.71&$\mathbf{ 94.86}$\\\hline
Pathbased& 74.85$\pm$0.09& 66.49& 74.57& 73.94& 76.30& 76.30& 65.73&$\mathbf{79.14}$\\
3:300$\times$2& 51.46$\pm$0.16& 30.17& 50.92& 47.90& 54.63& 54.63& 28.21&$\mathbf{ 59.21}$\\\hline
Zoo& 87.49$\pm$1.96& 54.12& 84.06& 88.83& 54.12& 54.12& 90.22&$\mathbf{ 95.47}$\\
7:101$\times$16& 71.93$\pm$3.15& 34.23& 55.56& 72.93& 32.15& 32.15&$\mathbf{ 76.98}$& 71.79\\\hline
Aggregation& 91.91$\pm$0.69& 79.19& 79.00& 88.49& 82.82& 84.10& 80.71&$\mathbf{ 95.97}$\\
7:788$\times$2& 81.24$\pm$0.75& 48.84& 48.43& 63.52& 64.23& 60.70& 52.36&$\mathbf{ 84.82}$\\\hline
R15&$\mathbf{98.21}$$\pm$0.67& 92.15& 92.00& 93.76& 93.07& 92.91& 81.76& 96.92\\
15:600$\times$2&$\mathbf{ 93.35}$$\pm$2.40& 64.86& 59.28& 73.64& 73.49& 67.33& 47.37& 86.80\\\hline
Vehicle& 63.24$\pm$2.21& 62.03& 62.77& 51.00& 65.23& 65.00& 58.59&$\mathbf{ 68.42}$\\
4:846$\times$18& 17.73$\pm$0.82&  3.25&  1.28&  9.02& 12.63& 12.07& 14.84&$\mathbf{ 22.19}$\\\hline
Vowel&$\mathbf{ 86.62}$$\pm$0.61& 82.93& 84.10& 83.28& 84.23& 83.60& 80.94& 84.18\\
11:528$\times$10&$\mathbf{ 45.64}$$\pm$1.94& 11.39& 10.60& 11.57& 24.28&  7.73& 25.93& 34.20\\\hline
Echocardiogram& 66.41$\pm$7.92& 52.81& 56.66& 56.10& 75.01& 75.01& 71.84&$\mathbf{ 85.79}$\\
2:131$\times$10& 24.79$\pm$17.27&  0.54&  2.99&  1.35& 39.64& 39.64& 35.46&$\mathbf{ 58.50}$\\\hline
Haberman& 49.91$\pm$0.02& 49.84& 60.95& 61.89& 61.57& 62.21& 60.95&$\mathbf{ 64.96}$\\
2:306$\times$3&  0.04$\pm$0.04&  0.07&  0.74&  2.28&$\mathbf{  9.00}$&  4.44&  0.74&  8.70\\\hline
Heartc& 51.04$\pm$0.00& 50.12& 50.23& 50.67& 59.21& 59.21& 50.75&$\mathbf{ 66.57}$\\
2:303$\times$14&  1.39$\pm$0.00&  0.05&  0.14&  0.90& 13.98& 13.98&  1.14&$\mathbf{ 25.41}$\\\hline
Heartstatlog& 51.45$\pm$0.07& 50.04& 50.35& 50.81& 51.40& 51.40& 51.82&$\mathbf{ 59.40}$\\
2:270$\times$13&  1.87$\pm$0.07&  0.20&  0.15&  0.63&  1.63&  1.67&  2.40&$\mathbf{ 17.08}$\\\hline
Hepatitis& 62.77$\pm$3.03& 55.56& 71.90& 66.27& 67.02&$\mathbf{ 73.66}$& 67.02& 69.38\\
2:155$\times$19&  0.29$\pm$0.13&  0.96& 14.93&  0.17&  7.18&$\mathbf{ 17.36}$&  1.95&  6.09\\\hline
Hourse& 50.15$\pm$0.00& 51.34&$\mathbf{ 54.15}$& 50.15& 51.34& 51.34& 52.12& 51.98\\
2:300$\times$26&  1.24$\pm$0.00&  0.55&$\mathbf{  5.39}$&  1.24&  0.55&  0.55&  0.46&  0.25\\\hline
Housevotes& 78.83$\pm$0.15& 63.77& 68.77& 75.83& 71.40& 71.40& 79.61&$\mathbf{ 83.64}$\\
2:435$\times$16& 48.07$\pm$0.38& 34.16& 27.27& 44.66& 39.36& 39.36& 50.15&$\mathbf{ 56.38}$\\\hline
Sonar& 50.22$\pm$0.18& 49.80& 49.99& 50.43& 51.26& 50.06&$\mathbf{ 51.62}$& $\mathbf{51.62}$\\
2:208$\times$60&  0.74$\pm$0.28&  0.01&  0.23&  0.64&  2.06&  0.67&$\mathbf{  4.05}$&  2.72\\\hline
Spect& 52.97$\pm$0.00& 65.86& 50.67& 65.86& 50.88& 50.58& 67.17& $\mathbf{67.61}$\\
2:267$\times$44&$\mathbf{  8.48}$$\pm$0.00&  0.51&  0.51&  0.51&  0.35&  0.34&  1.15&  1.17\\\hline
Spectf& 53.95$\pm$2.31& 49.49& 50.51& 51.93& 41.93& 51.39& 53.20& $\mathbf{59.62}$\\
2:88$\times$44&$\mathbf{ 16.09}$$\pm$6.43&  0.19&  1.67&  6.34&  4.00&  3.03&  6.51& 15.15\\\hline
Pimaindian& 55.07$\pm$0.00& 51.74& 54.50& 57.99& 53.97& 54.82& 55.07& $\mathbf{60.33}$\\
2:768$\times$8&  2.67$\pm$0.00&  0.23&  0.09&  5.64&  0.21&  0.58&  0.95&$\mathbf{  9.51}$\\\hline
Tictactoe& 50.90$\pm$0.47& 50.08& 96.71& 55.26& 59.84& 59.84& 55.82& $\mathbf{96.91}$\\
2:958$\times$27&  0.69$\pm$0.14&  0.69& 87.89&  0.97& 13.48& 13.48&  1.95&$\mathbf{ 88.49}$\\\hline
AC-win & 2  &  0 & 2  & 0 & 0&1&1&18 \\
MI-win&6&0&1&0&1&1&2&12\\
Both-win&2&0&1&0&0&1&1&12\\
\hline \noalign{\smallskip}
\end{tabular} \label{Accuracy1}\\
\end{table*}

\begin{table*}
\caption{AC and MI of the clustering methods on benchmark datasets for nonlinear case} \centering
\begin{tabular}{lllllllll}
\hline\noalign{\smallskip}
Data &kmeans&kPC  &PPC  &TWSVC&RTWSVC&FRTWSVC&RampTWSVC &RFDPC \\
k:m$\times$n& AC(\%)& AC(\%)& AC(\%)& AC(\%)& AC(\%)& AC(\%)& AC(\%)& AC(\%)\\
&MI(\%)&MI(\%)&MI(\%)&MI(\%)&MI(\%)&MI(\%)&MI(\%)&MI(\%)\\
\noalign{\smallskip}\hline\noalign{\smallskip}
Compound& 84.84$\pm$4.11& 90.16& 70.32& 90.25& 90.16& 90.16& 89.73&$\mathbf{ 91.65}$\\
6:399$\times$2& 70.65$\pm$6.84&$\mathbf{ 72.24}$& 16.97& 71.60& 72.24& 72.24& 64.07& 61.57\\\hline
Dermatology& 71.66$\pm$1.26& 72.66& 70.62& 72.60& 72.60& 72.60& 72.90&$\mathbf{ 74.44}$\\
6:366$\times$34& 17.84$\pm$3.67& 18.00&  3.65& 18.00& 18.00& 18.00&$\mathbf{ 26.79}$& 20.16\\\hline
Ecoli& 79.93$\pm$1.24& 82.49& 69.13& 88.29& 82.49& 82.68& 83.01& $\mathbf{90.34}$\\
8:336$\times$7& 49.31$\pm$2.28& 57.79& 16.46&$\mathbf{ 62.21}$& 57.79& 57.57& 49.97& 50.97\\\hline
Glass& 69.27$\pm$1.45& 69.04& 66.82& 70.10& 69.04& 69.04&70.77& $\mathbf{ 71.41}$\\
6:214$\times$9& 37.50$\pm$2.09&$\mathbf{ 41.42}$&  7.35& 23.42& 41.42& 41.42& 29.18& 30.81\\\hline
Iris& 87.63$\pm$8.09& 91.24& 59.47& 91.24& 91.24& 91.24& 94.95& $\mathbf{97.40}$\\
3:150$\times$4& 76.26$\pm$9.85& 79.15& 13.93& 79.15& 79.15& 79.15&$\mathbf{ 86.23}$& 85.59\\\hline
Pathbased&$\mathbf{96.11}$$\pm$0.18& 76.29& 59.94& 80.57& 76.29& 76.29& 91.92& 93.92\\
3:300$\times$2&$\mathbf{ 88.23}$$\pm$0.40& 57.51& 11.60& 65.87& 57.51& 57.51& 79.83& 82.28\\\hline
Zoo& 87.14$\pm$3.39& 90.63& 89.52& 90.63& 90.63& 90.63&$\mathbf{ 91.25}$& 91.15\\
7:101$\times$16& 70.79$\pm$5.39& 77.99& 72.90& 77.99& 77.99& 77.99&$\mathbf{ 79.70}$& 70.27\\\hline
Echocardiogram& 71.14$\pm$0.82& 55.04& 56.66& 56.66& 55.04& 55.04& 71.84& $\mathbf{83.23}$\\
2:131$\times$10& 32.41$\pm$0.53&  0.85&  2.73&  2.73&  0.85&  0.85& 28.53&$\mathbf{ 51.66}$\\\hline
Haberman& 60.61$\pm$0.30& 63.21& 61.26&$\mathbf{ 63.55}$& 63.21& 63.21& 63.21& $\mathbf{63.55}$\\
2:306$\times$3&  0.23$\pm$0.13&  4.97&  0.75&$\mathbf{  5.55}$&  4.97&  4.97&  4.97&  5.41\\\hline
Heartc& 50.76$\pm$0.09& 51.37& 51.26& 50.50& 51.37& 51.37& 52.44& $\mathbf{53.24}$\\
2:303$\times$14&  1.74$\pm$0.31&  2.19&  1.68&  0.62&  2.19&  2.19&  3.52&$\mathbf{  5.31}$\\\hline
Heartstatlog& 50.83$\pm$0.41& 53.00& 51.54& 50.92& 53.00& 53.00&$\mathbf{ 54.91}$& 54.22\\
2:270$\times$13&  1.88$\pm$0.54&  3.79&  1.64&  0.81&  3.79&  3.79&$\mathbf{  6.98}$&  6.25\\\hline
Hepatitis& 65.35$\pm$1.58& 66.27& 67.79& 67.79& 66.27& 66.27& 67.02&$\mathbf{ 69.38}$\\
2:155$\times$19&  1.04$\pm$0.68&  0.29&  2.01&  2.01&  0.29&  0.29&  0.29&$\mathbf{  5.35}$\\\hline
Hourse& 52.27$\pm$0.25& 52.12& 53.05& 51.71& 52.12& 52.12& 52.12& $\mathbf{54.55}$\\
2:300$\times$26&  0.68$\pm$0.57&  0.46&  2.30&  0.50&  0.46&  0.46&  0.46&$\mathbf{  3.92}$\\\hline
Housevotes& 79.79$\pm$0.94& 75.50& 75.83&$\mathbf{ 91.21}$& 75.50& 75.50& 80.68& 79.96\\
2:435$\times$16& 46.91$\pm$1.87& 42.09& 46.38&$\mathbf{ 72.31}$& 42.09& 42.09& 48.86& 48.74\\\hline
Sonar& 50.16$\pm$0.28& 51.62& 52.66& 52.22& 51.62& 51.62&$\mathbf{ 54.52}$& $\mathbf{54.52}$\\
2:208$\times$60&  0.39$\pm$0.39&  4.24&  4.08&  5.43&  4.24&  4.24&  6.64&$\mathbf{  6.77}$\\\hline
Spect& 60.68$\pm$4.79& 66.73& 68.06& 68.06& 66.73& 66.73& 68.98&$\mathbf{ 71.87}$\\
2:267$\times$44&  3.38$\pm$3.72&  0.17&  2.35&  2.35&  0.17&  0.17&$\mathbf{ 17.69}$& 10.96\\\hline
Spectf& 63.87$\pm$0.94& 50.16&$\mathbf{ 74.18}$& 50.16& 50.16& 50.16& 62.03& 70.76\\
2:88$\times$44& 21.84$\pm$1.52&  3.88&$\mathbf{ 49.34}$&  3.88&  3.88&  3.88& 20.54& 34.36\\\hline
AC-win & 1  &0   &1   &2  &0&0&3&12  \\
MI-win&1&2&1&3&0&0&5&5\\
Both-win&1&0&1&1&0&0&2&5\\
\hline \noalign{\smallskip}
\end{tabular} \label{Accuracy2}\\
\end{table*}

On the synthetic data, we tested the ability of the plane-based clustering methods to capture the plane-based data distribution. The synthetic data in $R^3$ consists of three classes, where one class is on a plane and the other two classes are on two lines, respectively. The details of the synthetic data are shown in Table \ref{Art}. We sampled four groups from the synthetic data which include 120, 100, 80 and 60 samples, respectively. Then, the plane-based clustering methods, including kPC \cite{Kplane}, PPC \cite{PPC,kPPC}, TWSVC \cite{TWSVC}, RTWSVC \cite{RTWSVC}, FRTWSVC \cite{RTWSVC}, RampTWSVC \cite{RampTWSVC} and our RFDPC, were implemented on these four groups, where the parameters $c$, $c_1$ and $c_2$ were set to $0.1$, $\Delta$ was set to $0.3$, and $s$ was set to $-0.2$.
The clustering results were depicted in Fig. \ref{Artshow}. It can be seen from Fig. \ref{Artshow} that (i) kPC and TWSVC cannot capture these plane-based clusters; (ii) PPC obtains a plane constructed by the two lines frequently; (iii) RTWSVC and FRTWSVC capture the three clusters on group G1, but both of them lose a cluster when the number of samples decreases; (iv) RampTWSVC always finds three clusters inaccurately; (iv) our RFDPC finds the three clusters exactly. Thus, our RFDPC captures the plane-based clusters more precisely than other methods on the synthetic datasets.

To exhibit the relationship between sample and its cluster center, the deviation statistics of the samples from their cluster center planes were depicted in Fig. \ref{ArtDev}, where `-' denotes the 1-order statistics $\bar{f}$ of each cluster and `$\times$' denotes the 2-order statistics $\pm\tilde{f}$ of each cluster. A cluster that only has a `$\times$' in Fig. \ref{ArtDev} means its 1-order and 2-order statistics are out of the figure window. It is obvious that the 2-order statistics of deviation of kPC, PPC, and TWSVC are far from their 1-order statistics, and hence they cannot find the three plane-based clusters exactly. The cluster samples lie on their cluster center planes by RTWSVC and FRTWSVC on groups G1, G2 and G3, but they fail to find the 3rd cluster on groups G2 and G3. RampTWSVC has great fluctuation, and thus it cannot capture the three plane-based clusters exactly. Accordingly, our RFDPC captures the plane-based clusters well by adding additional statistics. The quantitative measurements were reported in Table \ref{ArtAC}, and the highest ones were bold. Apparently, our RFDPC owns the highest performance on the four groups than other plane-based clustering methods, which is consistent with the previous observations.

\begin{figure*}
\centering
    \subfigure[Haberman]{\includegraphics[width=0.33\textheight]{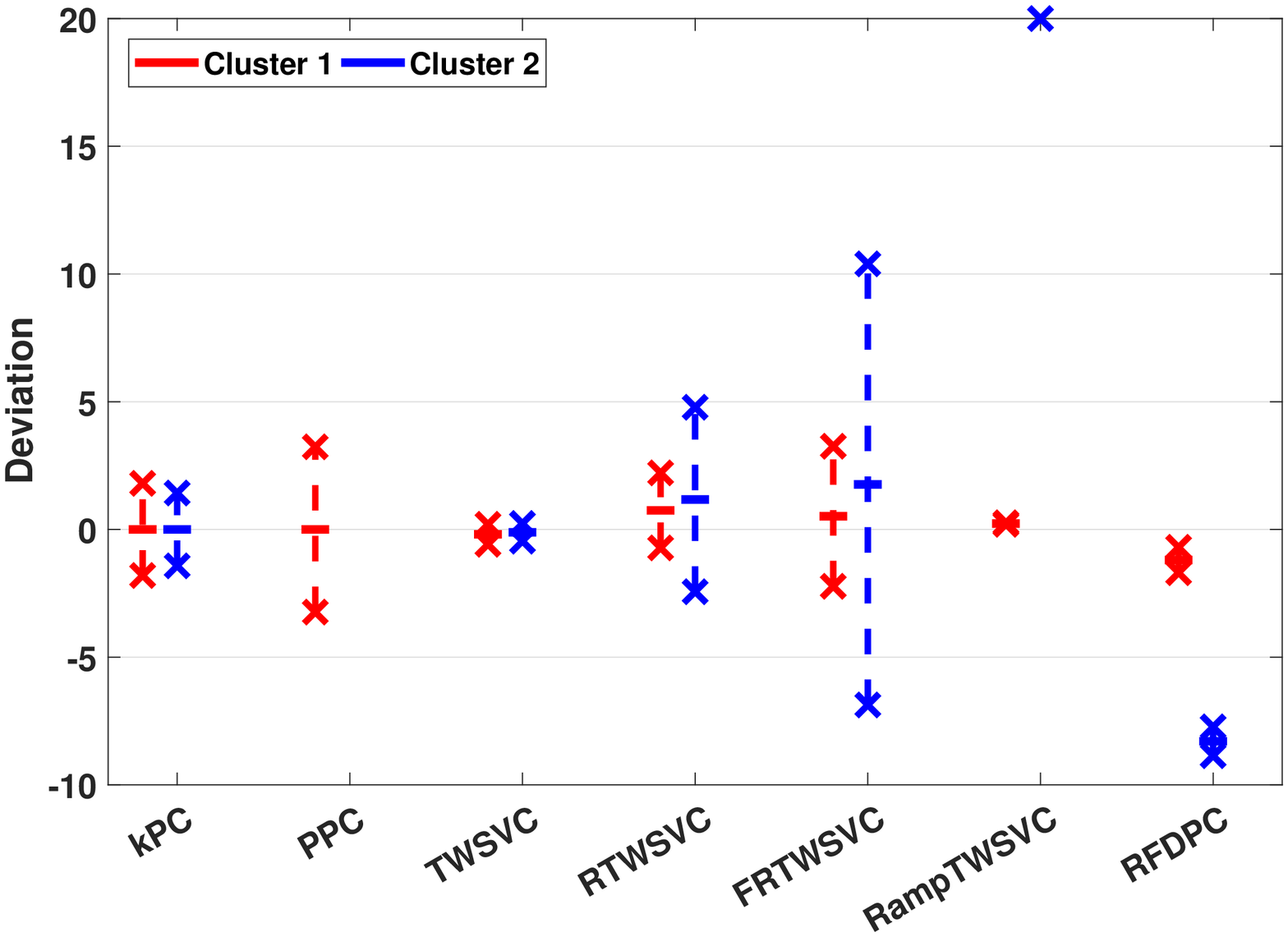}}
    \subfigure[Iris]{\includegraphics[width=0.33\textheight]{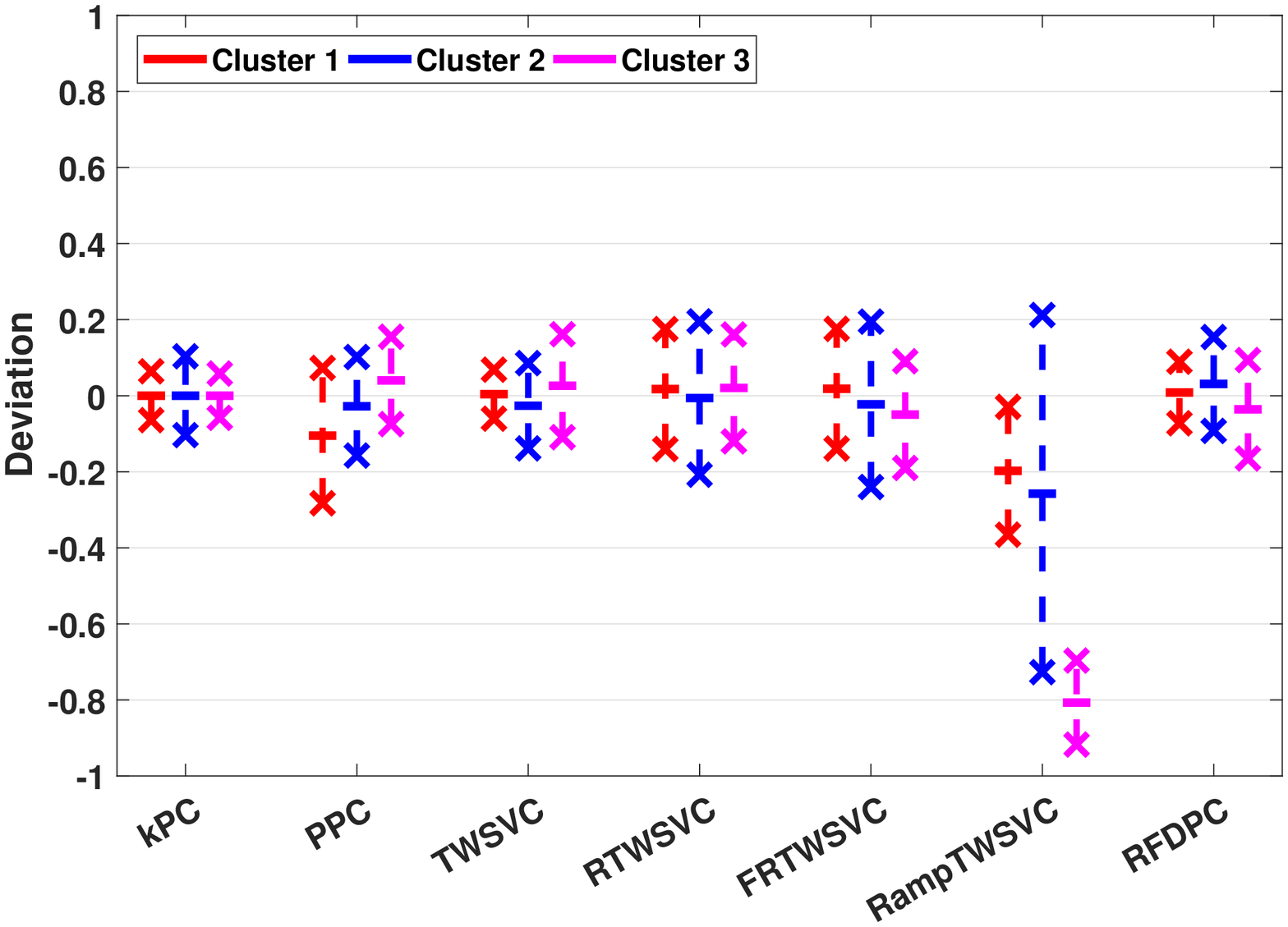}}
    \subfigure[Pathbased]{\includegraphics[width=0.33\textheight]{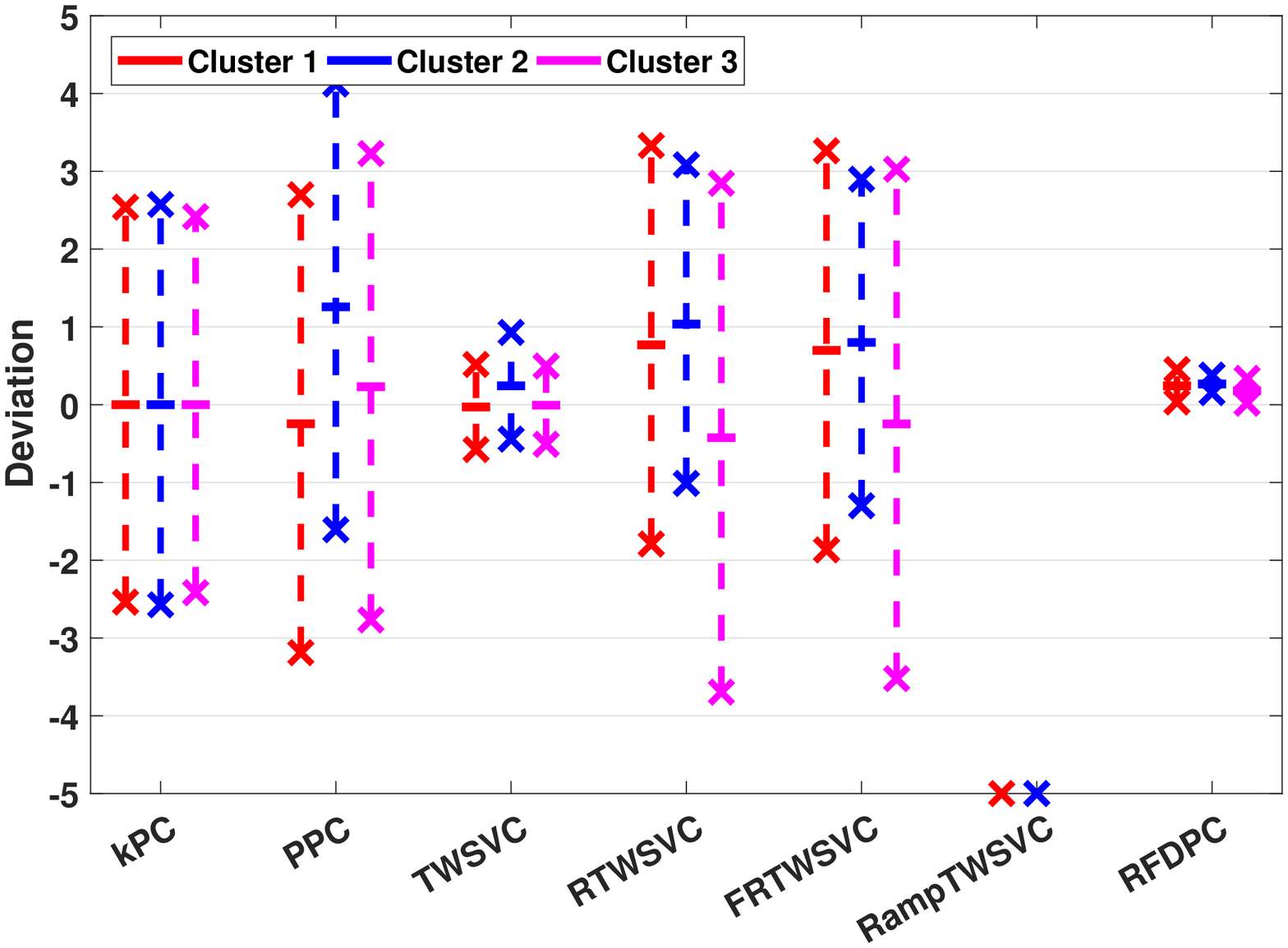}}
    \subfigure[Vehicle]{\includegraphics[width=0.33\textheight]{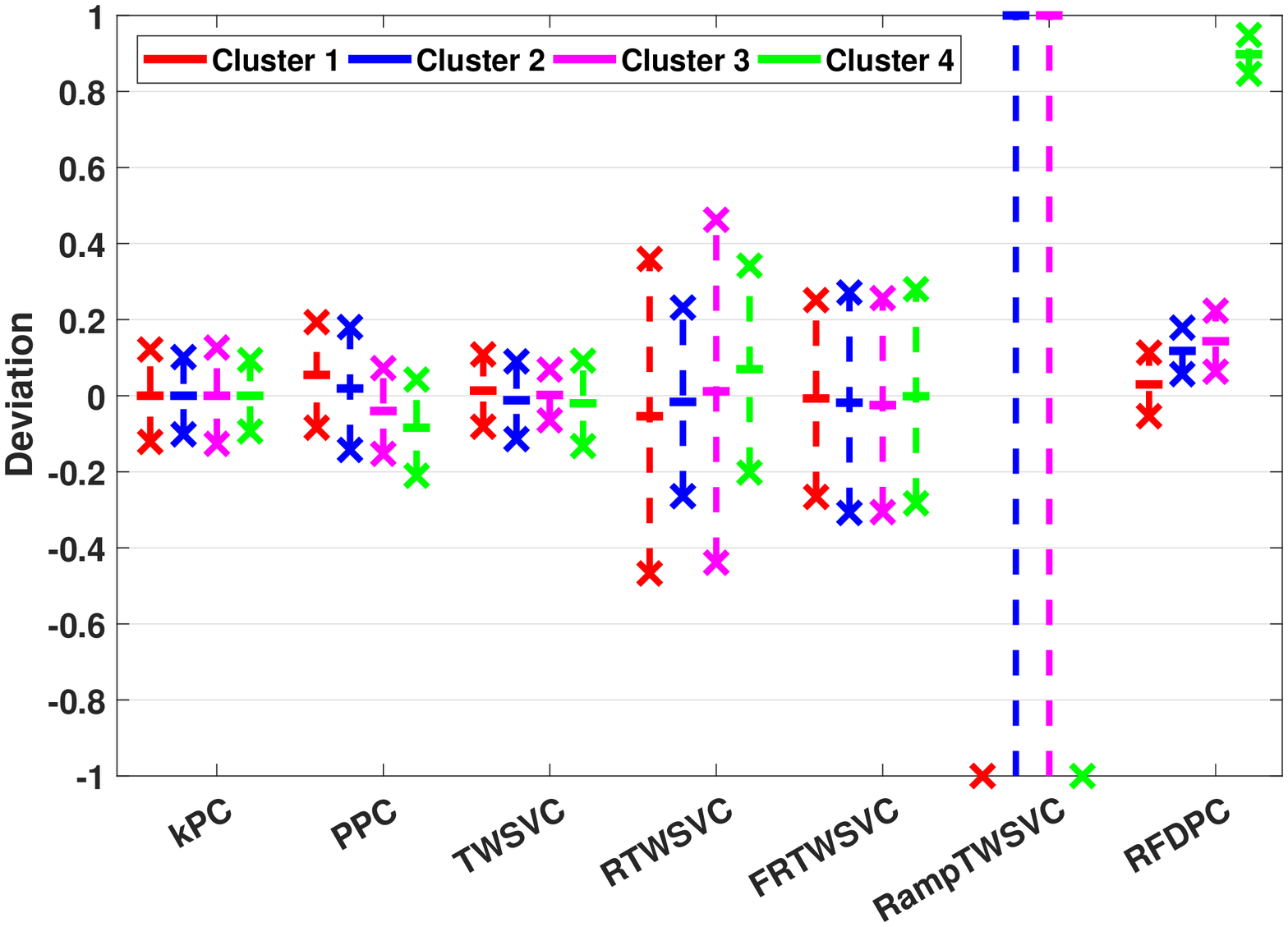}}
\caption{Deviation statistics of the samples from their cluster center planes on four benchmark datasets, where the symbolic descriptions can be found in Fig. \ref{ArtDev}.}\label{UCIDev}
\end{figure*}

In the following experiment, we implemented the above methods and kmeans \cite{Kmeans} on several benchmark datasets \cite{UCI} for linear and nonlinear cases.
Typically, $\Delta$ was set to $0.3$, $s$ was set to $-0.2$, and $\gamma_1=\gamma_2$. Other parameters in these methods were selected from $\{2^i|i=-8,-7,\ldots,7\}$. For nonlinear case, Gaussian kernel \cite{Kernel1,kPPC,Kernel4} $K(x_1,x_2)=\exp\{-\mu||x_1-x_2||^2\}$ was used and its parameter $\mu$ was selected from $\{2^i |i = -10,-9,\ldots,5\}$. The random initialization was fused for kmeans, and the NNG initialization \cite{TWSVC} was fused for the rest plane-based clustering methods to obtain stable performances. We reported AC and MI of these methods in Tables \ref{Accuracy1} and \ref{Accuracy2} for linear and nonlinear cases, respectively. Thereinto, kmeans was implemented ten times, and then the mean value and standard deviation were computed and reported. The highest ACs or MIs are bold, and the numbers of datasets with highest AC, MI and both are also shown in these tables.
From Table \ref{Accuracy1}, it can be seen that our RFDPC outperforms other methods on most of the datasets. Our RFDPC has the highest AC on 18/23 datasets, the highest MI on 12/23 datasets, and both of them on 12/23 datasets. Moreover, our RFDPC is comparable with the methods that own the highest AC or MI on most of the rest datasets.
Table \ref{Accuracy2} has similar results to that of Table \ref{Accuracy1} and confirms the observation from Table \ref{Accuracy1}. To exhibit the cluster center planes obtained by these plane-based clustering methods, we depicted the deviation statistics on the datasets ``Haberman'', ``Iris'', ``Pathbased'' and ``Vehicle'' in Fig. \ref{UCIDev} as instances.
Obviously, our RFDPC has a small and tight 2-order deviation statistics around the 1-order statistics, which improves the performance of plane-based clustering significantly.

\section{Conclusions}
A general model for plane-based clustering has been
proposed by introducing loss function and regularization. It has been shown that the general model terminates in a finite number of steps at the local or weak local optimal points theoretically. The existing plane-based clustering methods, including kPC, PPC, TWSVC, RTWSVC, FRTWSVC and RampTWSVC, are consistent with this general model. Furthermore, a new plane-based clustering method (RFDPC) based on the general model has been proposed. Experimental results on the synthetic and public available datasets
have indicated that our RFDPC can capture the data distribution more precisely. For practical convenience,
the corresponding RFDPC Matlab code has been uploaded upon
http://www.optimal-group.org/Resources/Code/RFDPC.html. In the future work, it is interesting to find more efficient loss functions and generalization terms in the general model to suit for specific clustering purpose.

\bibliographystyle{IEEEtran}
\bibliography{FBib}

%

\end{document}